\documentclass[preprint,3p,11pt]{elsarticle}

\usepackage{amssymb}
\usepackage{amsthm}

\usepackage{hyperref}
\hypersetup{colorlinks=true,
linkcolor=blue,
anchorcolor=blue,
citecolor=blue}

\usepackage{algorithmic}
\usepackage{algorithm}
\usepackage{amsmath}

\usepackage{caption}
\usepackage{subcaption}
\usepackage{float}
\usepackage{breqn}

\usepackage[section]{placeins} %用于控制图片显示在参考文献之前

\allowdisplaybreaks

\newtheorem{myTheo}{Theorem}[section]
\newtheorem{myLem}{Lemma}
\newtheorem{myCor}{Corollary}[section]
\newtheorem{myAss}{A}[section]

\journal{arXiv}

\begin{document}
\begin{sloppypar}

\begin{frontmatter}



%% use the ead command for the email address, and the form \ead[url] for the home page:



\title{Conjugate-Gradient-like Based Adaptive Moment Estimation \\
        Optimization Algorithm for Deep Learning}

\author[INS1,CEQ]{Jiawu Tian}
\ead{jiawu2015@gmail.com}

\author[INS1,COR]{Liwei Xu}
\ead{xul@uestc.edu.cn}

\author[INS1,COR,CEQ]{Xiaowei Zhang}
\ead{x.w.zhang@126.com}

\author[INS1]{Yongqi Li}
\ead{yongqi2023@126.com}

% \affiliation[INS1]{organization={School of Mathematical Sciences, 
%             University of Electronic Science and Technology of China},
%             addressline={\\ No.2006, Xiyuan Ave., West Hi-Tech Zone},
%             city={Chengdu},
%             postcode={611731},
%             state={Sichuan},
%             country={China}}

\affiliation[INS1]{{School of Mathematical Sciences, 
University of Electronic Science and Technology of China.}}

\affiliation[COR]{{Co-corresponding author.}}

\affiliation[CEQ]{{These authors contributed equally to this work.}}

\begin{abstract}
Training deep neural networks is a challenging task. In order to speed up training and 
enhance the performance of deep neural networks, we rectify the vanilla conjugate 
gradient as conjugate-gradient-like and incorporate it into the generic Adam, and thus
propose a new optimization algorithm named CG-like-Adam for deep learning. 
Specifically, both the first-order and the second-order moment estimation of generic Adam
are replaced by the conjugate-gradient-like. Convergence analysis handles the cases 
where the exponential moving average coefficient of the 
first-order moment estimation is constant and the first-order moment estimation is unbiased. 
Numerical experiments show the superiority of the proposed algorithm based on 
the CIFAR10/100 dataset.
\end{abstract}


\end{frontmatter}


\section{Introduction}
Deep learning has been used in many aspects, such as recommendation systems \cite{ref1}, 
natural language processing \cite{ref2}, image recognition \cite{ref3}, 
reinforcement learning \cite{ref4}, etc.
Neural network model is the main research object of deep learning, 
which includes input layer, hidden layer and output layer. 
Each layer includes a certain number of neurons, 
and each neuron is connected with each other in a certain way. 
The parameters and connection parameters of each neuron determine 
the performance of the deep learning model. 
How to optimize these huge number of parameters affects 
the performance of the deep learning model, attracting researchers 
to devote their energy to exploration \cite{ref5}.

Stochastic Gradient Descent(for short, SGD) has dominated training of deep neural networks 
despite it was proposed in the last century. It updates parameters 
of deep neural network toward the negative gradient direction 
which would be scaled by a constant called learning rate.
Simple as it is but may encounter non-convergence. More than one kind of 
improvement has been put forward, including momentum \cite{{ref6,ref7}}
and Nesterov’s acceleration \cite{ref8}, to accelerate the
training process and upgrade optimization. Large number of parameters 
sharing the same learning rate may be inappropriate since some parameters 
can be very close to the optimal, which needs to adjust learning rate at 
that situation. AdaGrad \cite{ref9} 
was the practicer of the idea of adaptive learning rate. 
It scales every coordinate of the gradient by the square roots of 
sum of the squared past gradients. Since then various of improved 
optimization algorithm combining the strength of momentum and adaptive learning rate
emerge in an endless stream. Some of the most popular ones are AdaDelta \cite{ref10}, 
RMSProp \cite{ref11}, Adam \cite{ref12}, NAdam \cite{ref13}, etc. 

Adam is the most popular one among them as a result of its fast convergence,
good performance as well as generality in most of practical applications. 
Although it has achieved wonderful results in many deep learning tasks, 
non-convergence issues trouble Adam in some other cases. Reddi etc.
constructed several examples to disprove convergence of Adam \cite{ref14}, and proposed AMSGrad fixing
the problem in proof of convergence of Adam algorithm given in reference  \cite{ref12}. 
However, AMSGrad's theoretical proof only handles the case where 
the objective function is convex, and can not ensure convergence when the 
exponential moving average coefficient of first-order moment estimation is a constant.
Zhou etc. tackled non-convex convergence of adaptive gradient methods \cite{ref15},
but the ``constant'' case is an open problem until the work of Chen etc. \cite{ref16}.
Unfortunately, when the first-order moment estimation is unbiased, further work
needs to be put into the proof. In a word, there is still room for improvement in for Adam-type 
optimization algorithm.

Conjugate Gradient method was proposed by Hestenes in the 1950s \cite{ref17} and was 
generalized to the non-linear optimization by Fletcher and Reeves in 1964 \cite{ref18}. 
It is very suitable for solving large-scale unconstrained non-convex optimization 
problems. The method iteratively moves the current parameters toward negative conjugate 
gradient direction which can be computed through the current gradient and the previous 
conjugate gradient multiplied by a conjugate coefficient. Specifically, the conjugate coefficient
could be calculated by efficient formulae such as Fletcher-Reeves(FR) \cite{ref18},
Polak-Ribiere-Polyak(PRP) \cite{{ref24,ref25}}, Hestenes-Stiefel(HS) \cite{ref17}, 
Dai-Yuan(DY) \cite{ref20} as well as Hager-Zhang(HZ) \cite{ref21}.

Directly replacing the gradient of Adam-type with vanilla conjugate gradient will not 
bring the satisfactory results, i.e., non-convergence, which was demonstrated 
in our experiments and the reference  \cite{ref22} whose author incorporated conjugate 
gradient into optimization algorithm for deep learning. 
Our work absorbs the idea of conjugate-gradient-like(CG-like)
in reference  \cite{{ref22,ref23}}. To be more specific, the conjugate coefficient is scaled by a positive 
real monotonic decreasing sequence depending on the number of iterations.(see 
Alg.\ref{alg:CG-like-Adam} for more details) 
Such rectification leads to conjugate-gradient-like 
direction, and we prove the convergence (Th.\ref{th3.2}) of the proposed algorithm for the non-convex. 
Unlike the work of the reference \cite{ref23}, the conjugate-gradient-like is also
used as second-order moment estimation. Experiments show that
the proposed algorithm(Alg.\ref{alg:CG-like-Adam}) performs better than CoBA \cite{ref23}.

There are two main perspectives to prove convergence: Regret Bound and Stationary Point. Convergence analysis 
for convex case usually adopts Regret Bound and Stationary Point for non-convex.
Our convergence analysis follows the train of thought in reference \cite{ref16} and extends 
the theorem to the case of first-order moment unbiased estimation, obtaining
more general convergence theorem. Several experiments were implemented to 
validate the effectiveness and commendable performance of our optimization 
algorithm for deep learning.

To sum up, the main contributions can be summarized as follows:

(i) The vanilla conjugate gradient direction is rectified by means 
of conjugate coefficient multiplied by a positive real monotonic decreasing 
sequence, which leads to conjugate-gradient-like direction that is incorporated
into Adam-type optimization algorithm for deep learning in order to speed up 
training process, boost performance and enhance generalization ability 
of deep neural networks. The algorithm, named CG-like-Adam now, combines the advantages of conjugate 
gradient for solving large-scaled unconstrained non-convex optimization 
problems and adaptive moment estimation.

(ii) The convergence for non-convex case is analyzed. The convergence analysis
tackles two hard situations: a constant coefficient of exponential moving average of 
first-order moment estimation, and unbiased first-order moment estimation. A great deal 
of work is based on the objective function convex and yet they cannot deal with the 
constant coefficient. Our proofs generalize the convergence theorem, making it more 
practical.

(iii) The numerical experiments both on CIFAR-10 and CIFAR-100 dataset using 
popular ResNet-34 and VGG-19 network for image classification task are done.
Experiment results not only testify the effectiveness of CG-like-Adam
but also provide satisfactory performance especially VGG-19 network 
on CIFAR-10/100 dataset.

\section{Preliminaries}
Here some necessary knowledge are prepared for better understanding.

\subsection{Notation}
$\left\| \cdot \right\|_{2}$ which is defined as $\ell_{2}$-norm
is replaced by $\left\| \cdot \right\|$ for convenience. 
$\langle \cdot , \cdot \rangle$ denotes inner product.
$[d]$, $\mathcal{T}$, $S^{d}_{+}$ are both set and denotes $\{1,2,\ldots,d\}$,
$\{1,2,\ldots,T\}$,  $\{V | V \in \mathbb{R}^{d \times d}, V \succ 0\}$ that 
is the set of all positive definite $d \times d$ matrices, respectively.
For any vector $x_t \in \mathbb{R}^{d}$, $x_{t,i}$ denotes its $i^{th}$ 
coordinate where $i \in [d]$ and $\hat{V}^{-\frac{1}{2}}_{t}x_{t}$ represents
$(\hat{V}^{\frac{1}{2}}_{t})^{-1}x_{t}$. Besides, $\sqrt{x_t}$ or $x^{\frac{1}{2}}_{t}$
represents for element-wise square root, $x^{2}_{t}$ for element-wise square,
$x_{t}/y_{t}$ or $\frac{x_{t}}{y_{t}}$ for element-wise division, 
$\max(x_{t},y_{t})$ for element-wise maximum, where $y_{t} \in \mathbb{R}^{d}$.
$diag(v) \in \mathbb{R}^{d \times d}$ is used to denote a diagonal matrix,
in which the diagonal elements $d_{ii}$ are $v_{i}$ and the other elements are 0,
$i \in [d]$. Finally, $O(\cdot), o(\cdot), \Omega(\cdot)$ are used
as standard asymptotic notations.

\subsection{Stochastic Optimization, Generic Adam and Stationary Point}
\textbf{Stochastic Optimization}
For any deep learning or machine learning model, it can be analysed by stochastic 
optimization framework. In the first place, consider the problem in the following form:
\begin{equation} 
{\min _{x \in \mathcal{X}} \mathbb{E}_{\pi}[\mathcal{L}(x, \pi)]+ \sigma(x)} ,
\end{equation}
where $\mathcal{X} \subseteq \mathbb{R}^{d}$ is feasible set. $\pi$ is a 
random variable with an unknown distribution,
representing randomly selected data sample or random noise. $\sigma(x)$ is a regular term.
For any given $x$, $\mathcal{L}(x, \pi)$ usually represents the loss function on sample
$\pi$. But for most practical cases,  the distribution of $\pi$ can not be obtained. 
Hence the expectation $\mathbb{E}_{\pi}$ can not be computed. 
Now there is another one to be considered,
known as Empirical Risk Minimization Problem(ERM):
\begin{equation}\label{eq2}
{\min _{x \in \mathcal{X}} f(x) = \frac{1}{N} \sum_{i=1}^{N} \mathcal{L}_{i}(x)+\sigma(x)} ,
\end{equation}
where $\mathcal{L}_{i}(x)=\mathcal{L}(x, \pi_{i})$, $i=1,2,\ldots,N$ and $\pi_{i}$ are samples.
For the convenience of our discussion below, without loss of generality, 
the regular term $\sigma(x)$ is ignored.

\textbf{Generic Adam}
There are many stochastic optimization algorithms to solve the ERM 
(see Eq.\eqref{eq2}), such as SGD, AdaDelta\cite{ref10}, RMSProp\cite{ref11}, 
Adam\cite{ref12}, NAdam\cite{ref13},etc. All the algorithms listed above are 
first-order optimization algorithm and can be described by the following Generic 
Adam paradigm(Alg.\ref{alg:G-Adam}), where $\varphi_{t}:\mathcal{X} \rightarrow S^{d}_{+}$, 
an unspecified ``averaging'' function, is usually applied to inversely weight the 
first-order moment estimation. The first-order moment estimation 
of Generic Adam is biased. Our study will focus on the Generic Adam, yet we will 
rectify all the moment estimation as unbiased and are going to incorporate conjugate
gradient into the Generic Adam as well as prove the convergence based on that. 

\begin{algorithm}[ht]
    \caption{Generic Adam}
    \label{alg:G-Adam}
        \begin{algorithmic}
            \STATE Require: ${x}_{1} \in \mathcal{X}$, ${m}_{0}:={0}$, 
            $(\beta_{1t})_{t \in \mathcal{T} } \subset [0,1)$.
            
            \FOR{$t=1$ to $T$}
                \STATE 
                $g_{t}$ : \text{noisy gradient}
            
                $m_{t}:=\beta_{1t} {m}_{t-1} + (1-\beta _{1t}) {g}_{t}$
                
                ${V}_{t}:=\varphi_{t} (g_{1},g_{2},\ldots,g_{t})$
            
                ${x}_{t+1}:={x}_{t} - \frac{\alpha_{t}}{\sqrt{V_{t}+\epsilon I}} \cdot m_{t}$
            \ENDFOR
    \end{algorithmic}
\end{algorithm}

\textbf{Stationary Point}
For any differentiable function $f(x)$, $x^{*}$ is called a stationary point 
when $\| \nabla f(x^{*}) \|^{2} = 0$, where $\nabla f$ denotes the gradient of $f$.
If $f$ is convex, $x^{*} \in \mathcal{X}$ will be a global minimizer of $f$ 
over $\mathcal{X}$. However, $f$ is usually non-convex in practice. When the 
solution obtained by an optimization algorithm is a stationary point, the gradient of 
$f$ is almost zero near the solution, resulting in the optimization process 
appearing stagnant. Although long runs can jump out of this local optimal solution, 
the time cost can be unacceptable.

\subsection{Vanilla Conjugate Gradient}
The vanilla conjugate gradient method for solving unconstrained nonlinear 
optimization problems has been studied for ages\cite{{ref18,ref17,ref20,ref21,ref24,ref25}}. 
It generates update direction $d_{t}$ by the following manner:
\begin{equation}\label{eq3}
    {d}_{t}:={g}_{t} - \gamma _{t} \cdot {d}_{t-1} ,
\end{equation}
where $\gamma _{t}$ is called conjugate coefficient and can be calculated 
directly by the following manner:
\begin{equation}\label{eq4}
    \gamma_{t}^{\mathrm{HS}}:=\frac{\left\langle{g}_{t}, {y}_{t}\right\rangle}{\left\langle{d}_{t-1}, {y}_{t}\right\rangle} ,
\end{equation}
\begin{equation}\label{eq5}
    \gamma_{t}^{\mathrm{FR}}:=\frac{\left\|{g}_{t}\right\|^{2}}{\left\|{g}_{t-1}\right\|^{2}} ,
\end{equation}
\begin{equation}\label{eq6}
    \gamma_{t}^{\mathrm{PRP}}:=\frac{\left\langle{g}_{t}, {y}_{t}\right\rangle}{\left\|{g}_{t-1}\right\|^{2}} ,
\end{equation}
\begin{equation}\label{eq7}
    \gamma_{t}^{\mathrm{DY}}:=\frac{\left\|{g}_{t}\right\|^{2}}{\left\langle{d}_{t-1}, {y}_{t}\right\rangle} ,
\end{equation}
\begin{equation}\label{eq8}
    \gamma_{t}^{\mathrm{HZ}}:=\frac{\left\langle{g}_{t}, {y}_{t}\right\rangle}{\left\langle{d}_{t-1}, {y}_{t}\right\rangle} 
    -\lambda \frac{\left\|{y}_{t}\right\|^{2}}{\left\langle{d}_{t-1}, {y}_{t}\right\rangle^{2}}\left\langle{g}_{t}, {d}_{t-1}\right\rangle ,
\end{equation}
where ${y}_{t}:={g}_{t}-{g}_{t-1},  \lambda>\frac{1}{4}$.
All the above was proposed by Hestenes-Stiefel\cite{ref17}, Fletcher-Reeves\cite{ref18},
Polak-Ribiere-Polyak\cite{{ref24,ref25}}, Dai-Yuan\cite{ref20}, Hanger-Zhang\cite{ref21}, respectively.

\section{CG-like-Adam}
\label{}
\subsection{Proposed Algorithm}
\textbf{First-order moment estimation}
In Generic Adam, $m_{t}:=\beta_{1t} {m}_{t-1} + (1-\beta_{1t}) {g}_{t}$, 
the biased estimation of the gradient of $f$, is the update direction.
The conjugate gradient is desired to be update direction, 
so that $m_{t}:=\beta_{1t} {m}_{t-1} + (1-\beta_{1t}) {d}_{t}$.
But this is not over yet. If just use Eq.\eqref{eq3} as the calculation of the conjugate gradient, then the conclusion of    
algorithm divergence has been verified through both our experiment and the reference \citep{ref22}. 
The vanilla conjugate gradient should be modified as conjugate-gradient-like:
\begin{equation}\label{eq9}
    {d}_{t}:={g}_{t} - \frac{\gamma _{t}}{t^{a}} \cdot {d}_{t-1} ,
\end{equation}
where $a \in \left[\frac{1}{2}, +\infty \right)$. What needs to emphasize is that our CG-like is different
from CoBA\citep{ref23} in which a very small constant $M$ was used, and therefore the update direction
of CoBA is almost the noisy gradient $g_{t}$ of $f$.

The update direction $m_{t}$ should be unbiased, so the 
following amendment is made and it is taken as the update direction of our algorithm:
\begin{equation}\label{eq10}
    {\hat{m}}_{t}:=\frac{m_{t}}{1-\beta^{t}_{11}}.
\end{equation}
Such amendment will bring us considerable trouble to the convergence proof, 
but we have managed to solve it.

\textbf{Second-order moment estimation}
Theoretically, there are many ways to instantiate $\varphi_{t}$ in Generic Adam, 
but the commonly used one is the exponential moving average of the square of the past
gradient of $f$ until the current step $t$:
\begin{equation}\label{eq11}
    {v}_{t}:=\beta _{2} {v}_{t-1} + (1-\beta _{2}) g^{2}_{t},
\end{equation}
\begin{equation}\label{eq12}
    {V}_{t}:=diag(v_{t}).
\end{equation}
This momentum adaptively adjusts the learning rate $\alpha_{t}$. In other words,
the update stepsize of each dimension of the solution toward 
the $-\hat{m}_{t}$ direction will be different. Since CG-like is used 
in our algorithm, the second-order moment estimation should be as follows 
and also unbiased:
\begin{equation}\label{eq13}
    {v}_{t}:=\beta _{2} {v}_{t-1} + (1-\beta _{2}) d^{2}_{t},
\end{equation}
\begin{equation}\label{eq14}
    \hat{{v}}_{t} = \frac{{v}_{t}}{1-\beta^{t}_{2}}.
\end{equation}

In order to further ensure the algorithm convergence, the maximum of all 
$\hat{{v}}_{t}$ until the current time step $t$ is also maintained:
\begin{equation}\label{eq15}
    \hat{{v}}_{t}:= max \{ \hat{{v}}_{t-1}, \hat{{v}}_{t} \},
\end{equation}
\begin{equation}\label{eq16}
    \hat{V}_{t}:=diag(\hat{v}_{t}).
\end{equation}

Finally we get the CG-like-Adam(Alg.\ref{alg:CG-like-Adam}).

\begin{algorithm}[ht]
    \caption{CG-like-Adam Algorithm}
    \label{alg:CG-like-Adam}
        \begin{algorithmic}
            \STATE Require: ${x}_{1}$ $\in \mathcal{X}$, 
            % $(f_{t})_{t \in \mathcal{T} }$,
            $(\beta  _{1t})_{t \in \mathcal{T} } \subset [0,1)$, 
            $\beta _{2} \in (0,1)$, 
            $a \in \left[ \frac{1}{2},+\infty \right)$, 
            % $b \in [\frac{1}{2},1)$,
            % $\alpha_{t} = \frac{\alpha}{t^{b}} > 0$, 
            $\gamma _{1}=0$, $\epsilon > 0$.
            
            ${m}_{0}:={0}$, ${v}_{0}:={0}$, ${d}_{0}:={0}$.
        
            \FOR{$t=1$ to $T$}
            \STATE 
            $g_{t}$ : \text{noisy gradient}

            ${d}_{t}:={g}_{t} - \frac{\gamma _{t}}{t^{a}} \cdot {d}_{t-1}$ 
        
            ${m}_{t}:=\beta _{1t} {m}_{t-1} + (1-\beta _{1t}) {d}_{t}$
        
            $\hat{{m}}_{t}:=\frac{{m}_{t}}{1-\beta^{t}_{11}}$
            
            ${v}_{t}:=\beta _{2} {v}_{t-1} + (1-\beta _{2}) d^{2}_{t}$
        
            $\hat{{v}}_{t} = \frac{{v}_{t}}{1-\beta^{t}_{2}}$
        
            $\hat{{v}}_{t}:= max \{ \hat{{v}}_{t-1}, \hat{{v}}_{t} \}$
        
            $\hat{V}_{t}:=diag(\hat{{v}}_{t})$
        
            ${x}_{t+1}:= {x}_{t} - \frac{\alpha_{t}}{\sqrt{\hat{V}_{t}+\epsilon I}} \cdot \hat{{m}}_{t}$
            
            \ENDFOR
        \end{algorithmic}
  \end{algorithm}

\subsection{Assumptions and Convergence Analysis}
Based on the demanding of our convergence analysis, 
the assumptions are directly listed as follows.
\begin{myAss}\label{ass31}
    $f$ is differentiable and has $L$-Lipschitz gradient: 
    $\| \nabla f(x) - \nabla f(y) \| \leqslant L \| x - y \| $ holds for all $x, y$ $\in \mathcal{X}.$
\end{myAss}
\begin{myAss}\label{ass32}
    $f$ is lower bounded: $f(x^{*}) > - \infty$, where $x^{*}$ is an optimal solution.
\end{myAss}
\begin{myAss}\label{ass33}
    The noisy gradient $g_{t}$ is unbiased and the noise is independent:
    $g_{t} = \nabla f(x_{t}) + \zeta _{t}$, $\mathbb{E}[\zeta _{t}] = 0$, and $\zeta _{i}$ 
    is independent of $\zeta _{j}$ if $i \neq j$.
\end{myAss}
\begin{myAss}\label{ass34}
    There exists a constant $H$, for all $t \in \mathcal{T}$,  
    $\| \nabla f(x_{t}) \| \leqslant H$, $\| g_{t} \| \leqslant H$.
\end{myAss}

\begin{myTheo}\label{th3.1}
  Suppose that the assumptions A\ref{ass31}-A\ref{ass34} are satisfied. $\beta_{1t} \in [0,1)$, 
  $\beta_{1t} \leq \beta_{1(t+1)}$, $\beta_{1(t+1)} \leq \beta_{1t} h(t)$
  (or $\beta_{1t} h(t) \leq \beta_{1(t+1)}$) hold for all $t \in \mathcal{T}$, 
  in which $h(t)=\frac{(1-\beta_{11}^{t-1})(1-\beta_{11}^{t+1})}{(1-\beta_{11}^{t})^{2}}$.
  And $\exists G \in \mathbb{R}^{+}, \forall t \in \mathcal{T}, 
  \|\alpha_{t} \hat{V}_{t}^{-\frac{1}{2}} \hat{m}_{t}\| \leq G$.
  Then the CG-like-Adam(Alg.\ref{alg:CG-like-Adam}) satisfies
  \begin{equation}
      \begin{aligned}
      \label{eq17}
      & \mathbb{E}\left[\sum_{t=1}^{T}\left\langle\nabla f\left(x_{t}\right), \alpha_{t} \hat{V}_{t}^{-\frac{1}{2}} \nabla f(x_{t})\right\rangle\right]\\
      &
       \leq C_{1} \mathbb{E}\left[\sum_{t=2}^{T} \sum_{k=1}^{d}\left|\mu_{t} \hat{v}_{t, k}^{-\frac{1}{2}} - \mu_{t-1} \hat{v}_{t-1, k}^{-\frac{1}{2}}\right|\right] \\
      & \quad\,  +C_{2} \mathbb{E}\left[\sum_{t=1}^{T-1}\left\|\alpha_{t} \hat{V}_{t}^{-\frac{1}{2}} d_{t}\right\|^{2}\right]
      +C_{3} + C_{4} \sum_{t=1}^{T} \frac{\alpha_{t}\left|\gamma_{t}\right|}{t^{a}} ,
      \end{aligned}
  \end{equation}
  where $C_{1}$,$C_{2}$,$C_{3}$ and $C_{4}$ are both constant independent of $T$, 
  $\mu_{t}=\frac{\alpha_{t}\left(1-\beta_{1 t}\right)}{\xi_{t}}$, 
  $\xi_{t}=\left(1-\beta_{11}^{t}\right)-\beta_{1 t}\left(1-\beta_{11}^{t-1}\right)$.
  The expectation $\mathbb{E}$ is taken with respect to all the randomness corresponding to $g_{t}$.
\end{myTheo}

\begin{proof}
  See Appendix \ref{prth31}.
\end{proof}

\begin{myTheo}\label{th3.2}
  Suppose that the assumptions A\ref{ass31}-A\ref{ass34} and the conditions of Th.\ref{th3.1}
  are satisfied. When $T \rightarrow +\infty$, there is 
  \begin{equation}\label{eq18}
      {\min _{t \in \mathcal{T}}\left[\mathbb{E}\left\|\nabla f\left(x_{t}\right)\right\|^{2}\right]=O\left(\frac{S_{1}(T)}{S_{2}(T)}\right)} ,
  \end{equation}
  where when $T \rightarrow +\infty$,
  \begin{equation}
    \begin{aligned}
    O(S_{1}(T))
    & = C_{1} \mathbb{E}\left[\sum_{t=2}^{T} \sum_{k=1}^{d}\left|\mu_{t} \hat{v}_{t, k}^{-\frac{1}{2}} - \mu_{t-1} \hat{v}_{t-1, k}^{-\frac{1}{2}}\right|\right] \\
    &\quad\, +C_{2} \mathbb{E}\left[\sum_{t=1}^{T-1}\left\|\alpha_{t} \hat{V}_{t}^{-\frac{1}{2}} d_{t}\right\|^{2}\right]
    + C_{3} +{C_{4} \sum_{t=1}^{T} 		\frac{\alpha_{t}\left|\gamma_{t}\right|}{t^{a}}},
\end{aligned}
\end{equation}
  ${\Omega\left(S_{2}(T)\right)=\sum_{t=1}^{T} \tau_{t}}$ in which 
  $\tau_{t}:=\underset{k \in[d]}{\min} \left\{ \underset{\{g_{i}\}_{i=1}^{t}}{\min} \frac{\alpha _{t}}{\sqrt{\hat{v}_{t, k}}} \right\}$ 
  denotes the minimum possible value of effective stepsize at time $t$ 
  over all possible coordinate and the past noisy gradients $\{g_{i}\}_{i=1}^{t}$.
  And $S_{1}(T)$, $S_{2}(T)$ are functions that are all unrelated to random variables.
\end{myTheo}

\begin{proof}
  See Appendix \ref{prth32}.
\end{proof}

The theorem \ref{th3.2} is the direct result of theorem \ref{th3.1}, which
provides a sufficient condition that guarantees convergence of our CG-like-Adam algorithm:
the Right Hand Side(RHS) of Eq.\eqref{eq17} increases much slower than the sum of the 
minimum possible value of effective stepsize as $T \rightarrow +\infty$. Equivalently,
when $T \rightarrow +\infty$, $S_{2}(T)$ grows slower than $S_{1}(T)$: $S_{2}(T)=o(S_{1}(T))$.

The conclusions of theorem \ref{th3.1} and \ref{th3.2} are similar to those of \citep{ref16}, 
but our theorem is extended to the case where the moment estimation is unbiased, 
which makes the conclusion more universal.

If the learning rate $\alpha_{t}$ is specified as $\frac{\alpha}{t^{b}}$ and 
$\beta_{1t}$ is a constant for all $t \in \mathcal{T}$, then the 
following corollary describes the convergence rate of 
CG-like-Adam(Alg.\ref{alg:CG-like-Adam}):
\begin{myCor}\label{cor31}
  Suppose that the assumptions A\ref{ass31}-A\ref{ass34} and the 
  conditions of Th.\ref{th3.1} are satisfied.
  $\forall t \in \mathcal{T}, \beta_{1 t}=\beta_{1 (t+1)},\alpha_{t} = \frac{\alpha}{t^{b}},b \in [\frac{1}{2},1)$, and 
  $\exists c \in \mathbb{R}^{+}, \forall i \in [d], |g_{1,i}| \geqslant c$, 
  then the following holds:
  \begin{equation}\label{eq19}
      \min _{t \in \mathcal{T}} \mathbb{E} \left[\left\|\nabla f\left(x_{t}\right)\right\|^{2}\right]
      \leqslant \frac{Q_{1}}{T^{1-b}} \left(Q_{2}\ln T+Q_{3}\right) ,
  \end{equation}
  where $Q_{1}$, $Q_{2}$ and $Q_{3}$ are both constant independent of $T$.
\end{myCor}

\begin{proof}
  See Appendix \ref{prcor31}.
\end{proof}

The corollary \ref{cor31} indicates that the best convergence rate 
of CG-like-Adam(Alg.\ref{alg:CG-like-Adam}) with the learning rate in the form of
$\alpha_{t} = \frac{\alpha}{t^{b}}$ will be $\frac{\ln T}{\sqrt{T}}$ 
when $b=\frac{1}{2}$. The derived best convergence rate still falls short of the
fastest rate of first-order methods($\frac{1}{\sqrt{T}}$) due to an additional 
factor $\log T$. 
$\alpha_{t}=\alpha (\forall t \in \mathcal{T})$ is usually adopted in practice 
to mitigate the slowdown\citep{ref16},
yet it is still an open problem of convergence rate analysis in this case.

In our theoretical analysis, the positive definite matrix 
$\varepsilon I$ in the algorithm did not be taken into account for the reason of convenience.
Practical implementations may require adding $\varepsilon I$ 
for numerical stability. However, that does not lead to the loss of generality of our 
analysis since $\varepsilon I$ can be very easy converted and incorporated
into $\hat{v}_{t}$. Therefore, the assumption 
$|g_{1,i}| \geqslant c \ (\forall i \in [d])$ in the corollary \ref{cor31} 
is a mild assumption which means it could be easy to hold in practice.

\section{Experiments}
\label{}
In this section, we conduct several experiments of image classification task
on the benchmark dataset CIFAR-10\citep{ref26} and CIFAR-100\citep{ref26} using popular 
network VGG-19\citep{ref27} and ResNet-34\citep{ref28} to numerically compare 
CG-like-Adam(Alg.\ref{alg:CG-like-Adam}) with CoBA\citep{ref23} and Adam\citep{ref12}.
Specifically, VGG-19 was trained on CIFAR-10 and ResNet-34 was trained on CIFAR-100.

The CIFAR-10 and CIFAR-100 dataset both consist of 60000 32$\times$32 colour images 
with 10 and 100 classes respectively. Both two datasets are 
divided into training and testing datasets, which includes 50000 training images and 
10000 test images respectively.
The test batch contains exactly 100 or 1000 randomly-selected images from each class.

VGG-19, proposed by Karen Simonyan\citep{ref27} in 2014, consisting of 16 
convolution layers and 3 fully-connected layers, is a concise network structure 
made up of 5 vgg-blocks with consecutive 3$\times$3 convolutional kernels in each vgg-block.
The final fully-connected layer is 1000-way with a softmax function.
The ResNet-34 network, proposed by He\citep{ref28}, incorporates residual unit through a short-cut connection, 
enhancing the network's learning ability. ResNet-34 is organized as a
7$\times$7 convolution layer, four convolution-blocks including total 32 convolution layers 
with 3$\times$3 convolutional kernels and finally a 1000-way-fully-connected layer 
with a softmax function. VGG-19 and ResNet-34 have a strong ability to extract image features.
All networks was trained for 200 epochs on NVIDIA Tesla V100 GPU.

$\beta_{1t}=0.9(\forall t \in \mathcal{T})$ and $\beta_{2}=0.999$ are set as default values
for all optimizers, Adam, CoBA and CG-like-Adam(ours, Alg.\ref{alg:CG-like-Adam}). 
$\lambda=2$ of HZ(see Eq.\eqref{eq8}) and $a=1+10^{-5}$ are set as default values for CoBA 
and CG-like-Adam. $M=10^{-4}$ is the default value for CoBA.
The cross entropy is used as the loss function for all experiments for 
the reason of commonly used strategy in image classification.
Besides, batch size is set 512 as default.

\subsection{Compare CG-like-Adam with CoBA}
We firstly compare CG-like-Adam with CoBA. All the algorithms are run under different learning rates 
$\alpha_{t} \in \{10^{-3},10^{-4},10^{-5},10^{-6}\}, \forall t \in \mathcal{T}$.
Figure \ref{fig1}-\ref{fig5}, \ref{fig6}-\ref{fig10} show the results of the experiments of 
VGG-19 on CIFAR-10, and ResNet-34 on CIFAR-100, respectively.

\begin{figure}[htbp]
	\centering
	\subfloat[(Log.) Training Loss\label{fig1a}]{
		\centering
		\includegraphics[width=0.31\linewidth]{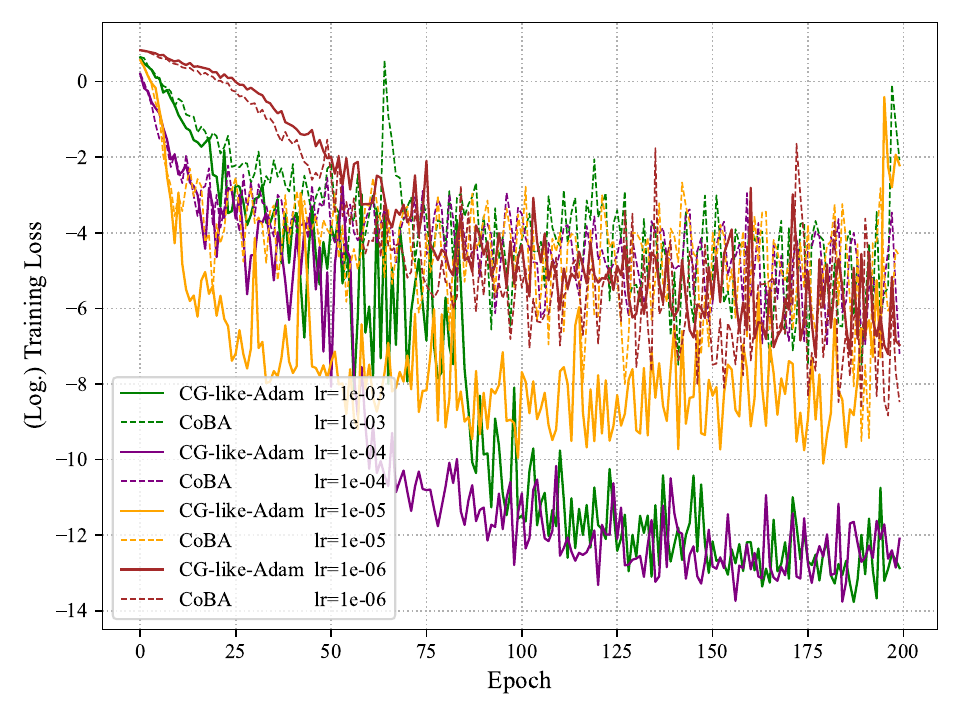}
	 }
    \centering
	\subfloat[Training Accuracy\label{fig1b}]{
		\centering
		\includegraphics[width=0.31\linewidth]{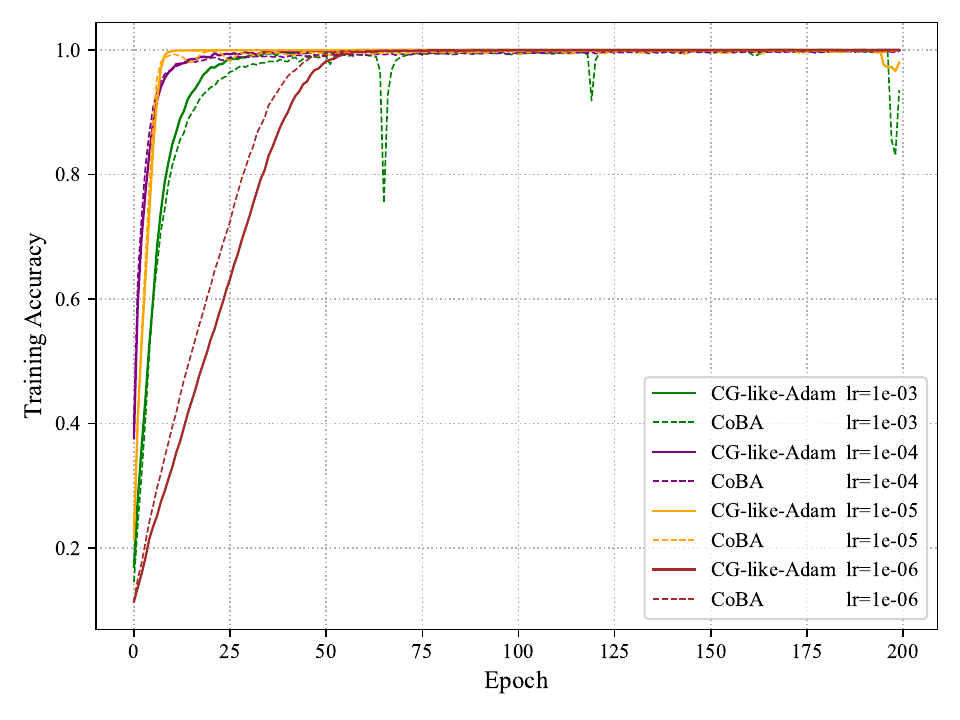}
     }
    \centering
	\subfloat[Testing Accuracy\label{fig1c}]{
		\centering
		\includegraphics[width=0.31\linewidth]{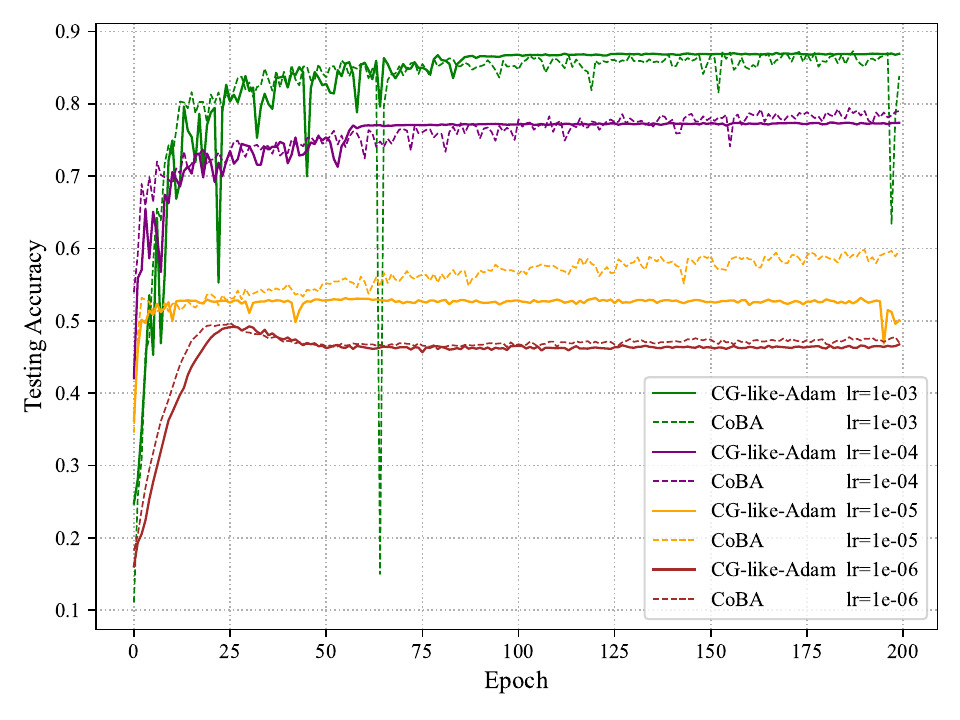}
     }
	% \caption{CG-like-Adam V.S. CoBA under different learning rates. Train VGG-19 on CIFAR-10. 
    %         The conjugate coefficient is calculated by HS(\ref{eq4}).}
    \caption{CG-like-Adam V.S. CoBA under different learning rates. 
    (VGG-19, CIFAR-10, HS(\ref{eq4}))}
	\label{fig1}
\end{figure}

\begin{figure}[htbp]
	\centering
	\subfloat[(Log.) Training Loss\label{fig2a}]{
		\centering
		\includegraphics[width=0.31\linewidth]{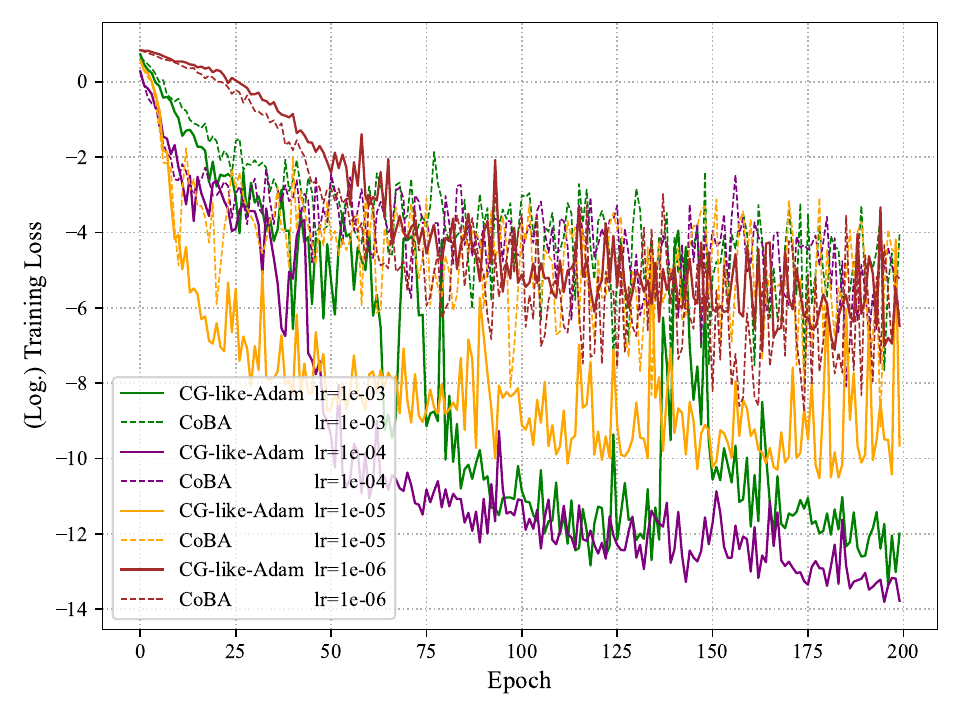}
    }
    \centering
	\subfloat[Training Accuracy\label{fig2b}]{
		\centering
		\includegraphics[width=0.31\linewidth]{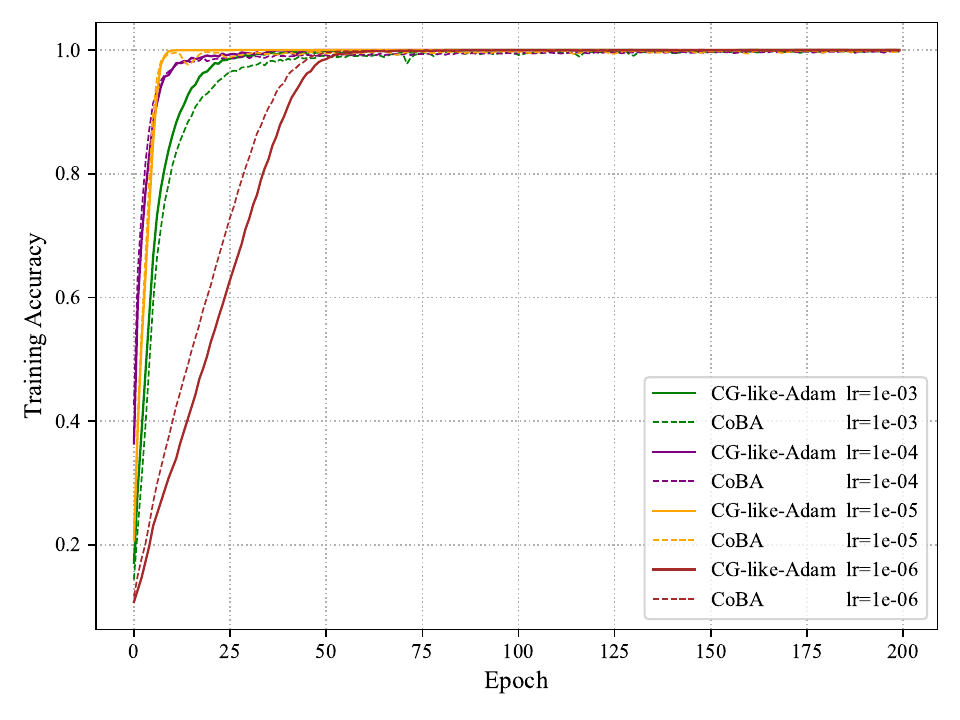}
    }
    \centering
	\subfloat[Testing Accuracy\label{fig2c}]{
		\centering
		\includegraphics[width=0.31\linewidth]{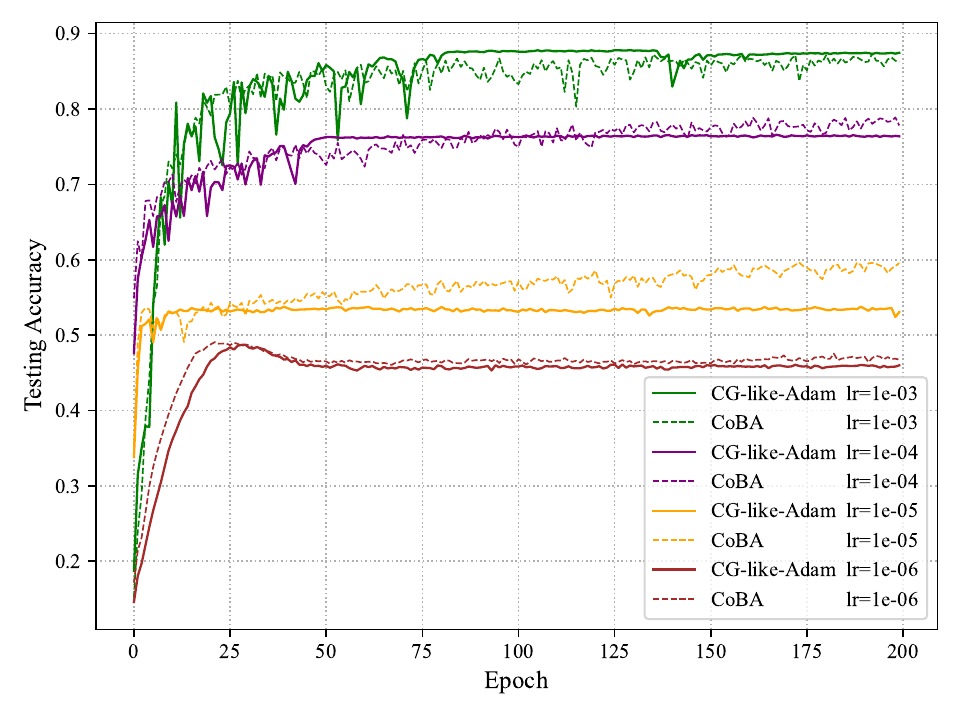}
     }
	% \caption{CG-like-Adam V.S. CoBA under different learning rates. Train VGG-19 on CIFAR-10. 
    %         The conjugate coefficient is calculated by FR(\ref{eq5}).}
    \caption{CG-like-Adam V.S. CoBA under different learning rates. 
    (VGG-19, CIFAR-10, FR(\ref{eq5}))}
	\label{fig2}
\end{figure}

\begin{figure}[htbp]
	\centering
	\subfloat[(Log.) Training Loss\label{fig3a}]{
		\centering
		\includegraphics[width=0.31\linewidth]{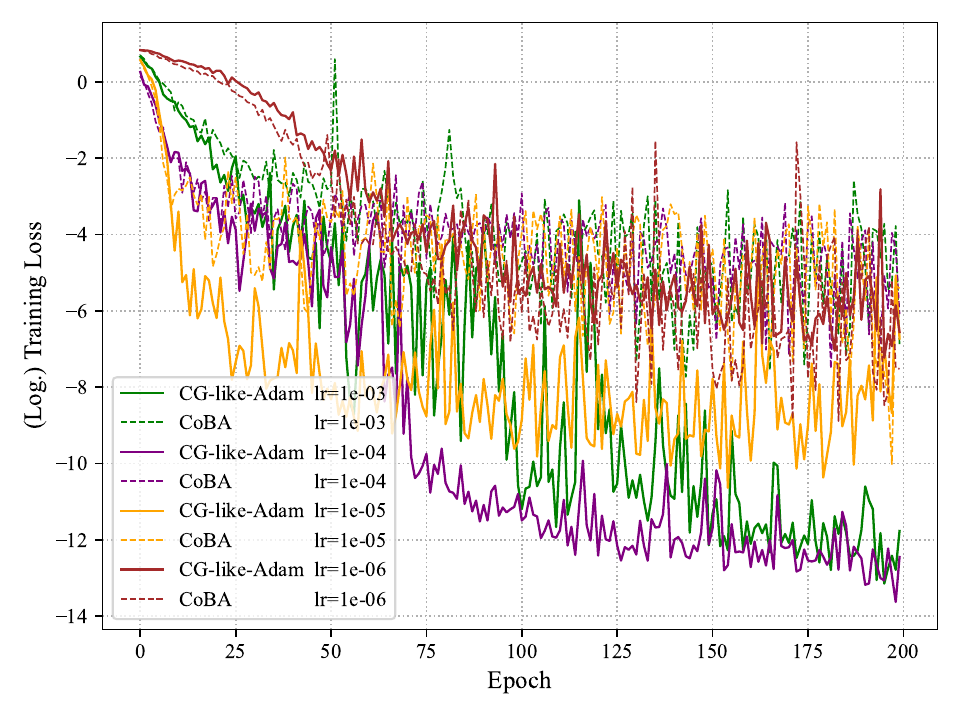}
    }
    \centering
	\subfloat[Training Accuracy\label{fig3b}]{
		\centering
		\includegraphics[width=0.31\linewidth]{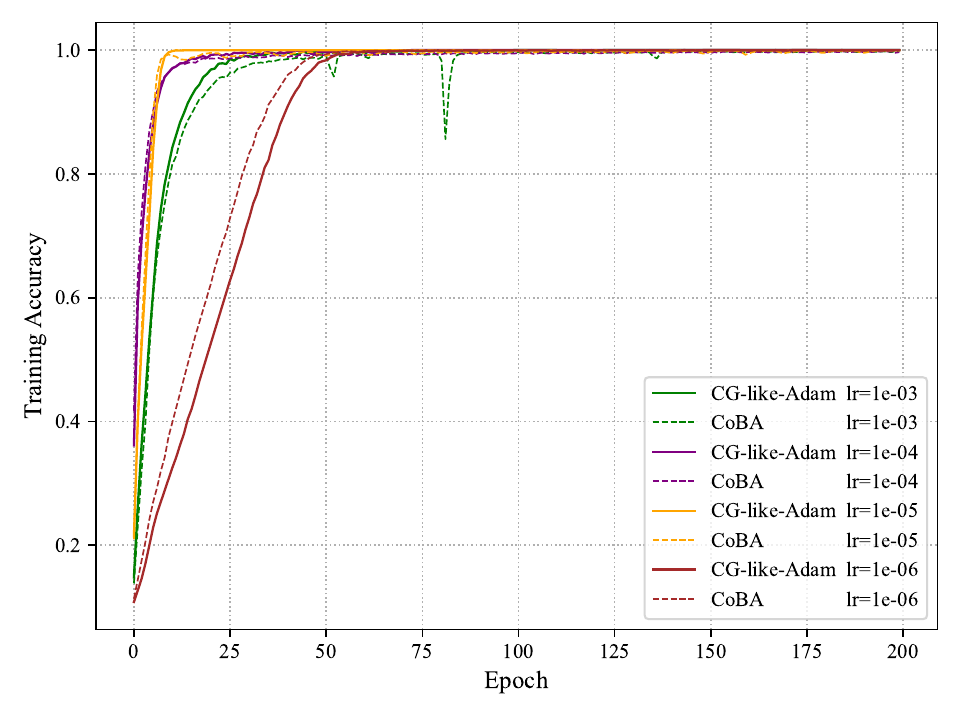}
    }
    \centering
	\subfloat[Testing Accuracy\label{fig3c}]{
		\centering
		\includegraphics[width=0.31\linewidth]{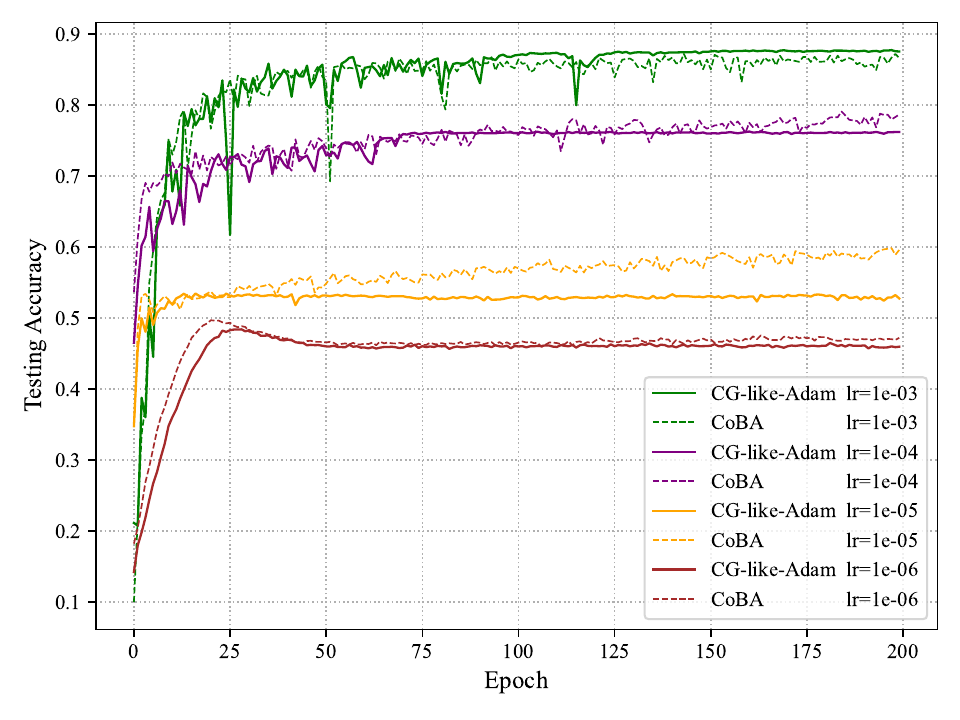}
    }
	% \caption{CG-like-Adam V.S. CoBA under different learning rates. Train VGG-19 on CIFAR-10. 
    %         The conjugate coefficient is calculated by PRP(\ref{eq6}).}
    \caption{CG-like-Adam V.S. CoBA under different learning rates. 
    (VGG-19, CIFAR-10, PRP(\ref{eq6}))}
	\label{fig3}
\end{figure}

\begin{figure}[htbp]
	\centering
	\subfloat[(Log.) Training Loss\label{fig4a}]{
		\centering
		\includegraphics[width=0.31\linewidth]{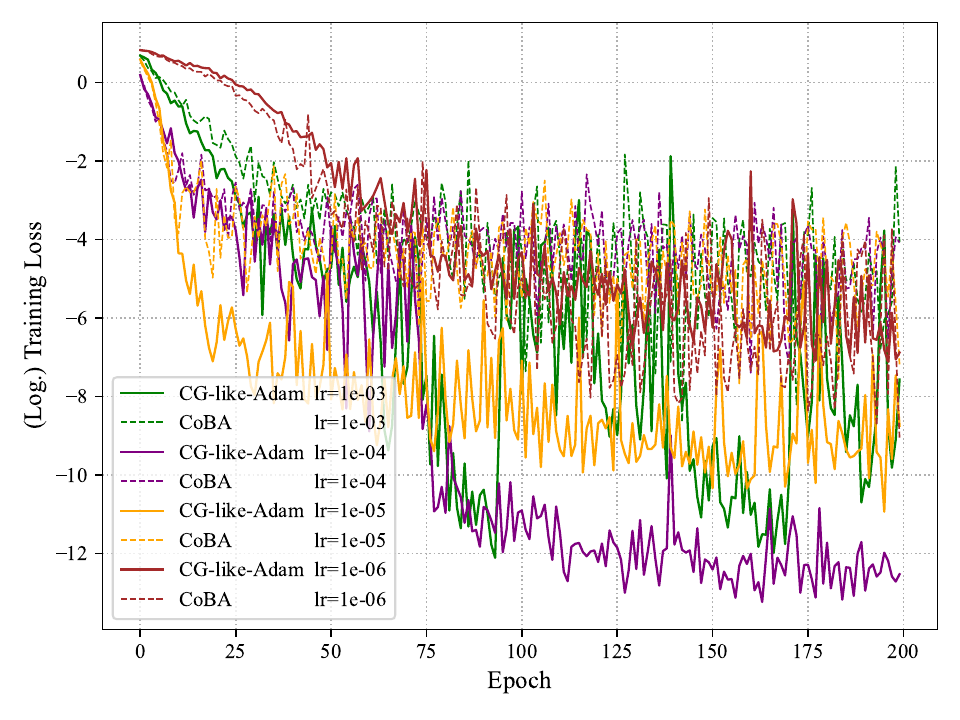}
}
    \centering
	\subfloat[Training Accuracy\label{fig4b}]{
		\centering
		\includegraphics[width=0.31\linewidth]{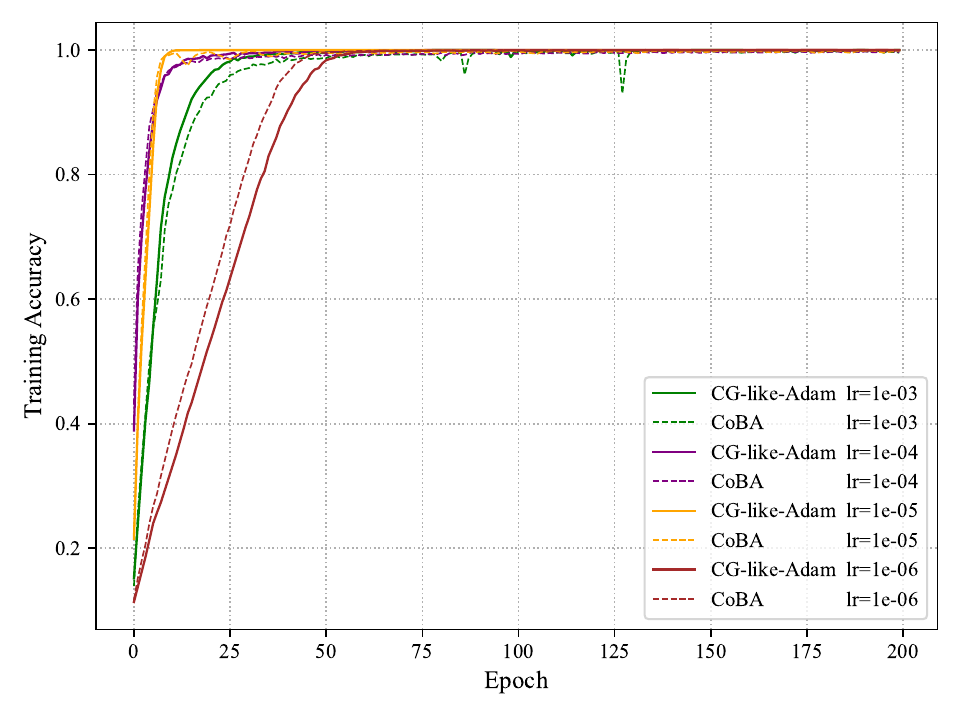}
}
    \centering
	\subfloat[Testing Accuracy\label{fig4c}]{
		\centering
		\includegraphics[width=0.31\linewidth]{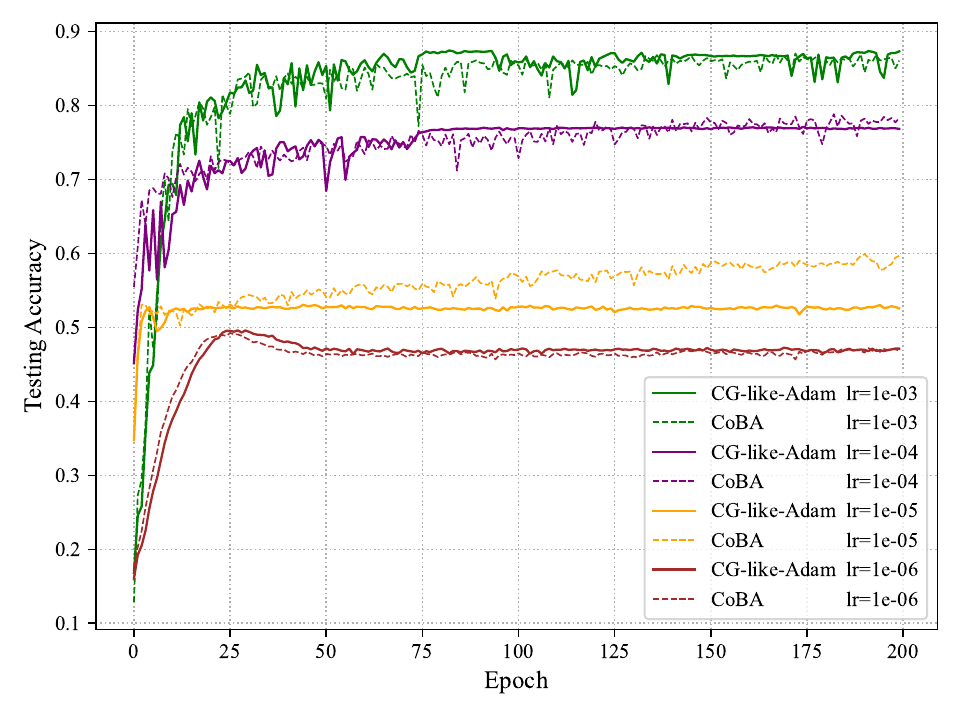}
}
	% \caption{CG-like-Adam V.S. CoBA under different learning rates. Train VGG-19 on CIFAR-10. 
    %         The conjugate coefficient is calculated by DY(\ref{eq7}).}
    \caption{CG-like-Adam V.S. CoBA under different learning rates. 
    (VGG-19, CIFAR-10, DY(\ref{eq7}))}
	\label{fig4}
\end{figure}

\begin{figure}[htbp]
	\centering
	\subfloat[(Log.) Training Loss\label{fig5a}]{
		\centering
		\includegraphics[width=0.31\linewidth]{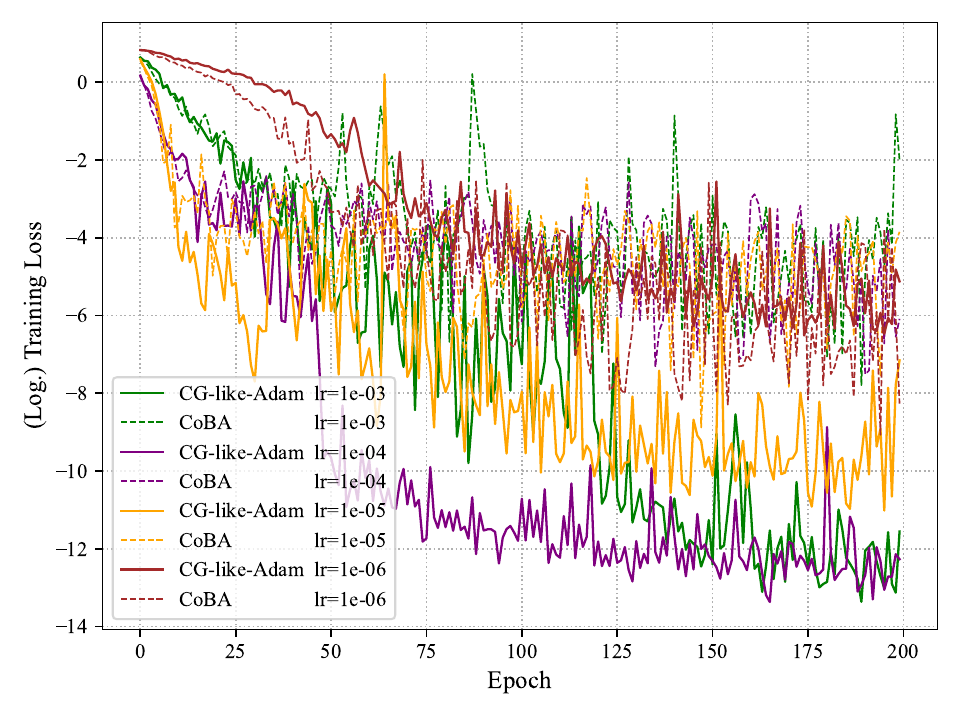}
}
    \centering
	\subfloat[Training Accuracy\label{fig5b}]{
		\centering
		\includegraphics[width=0.31\linewidth]{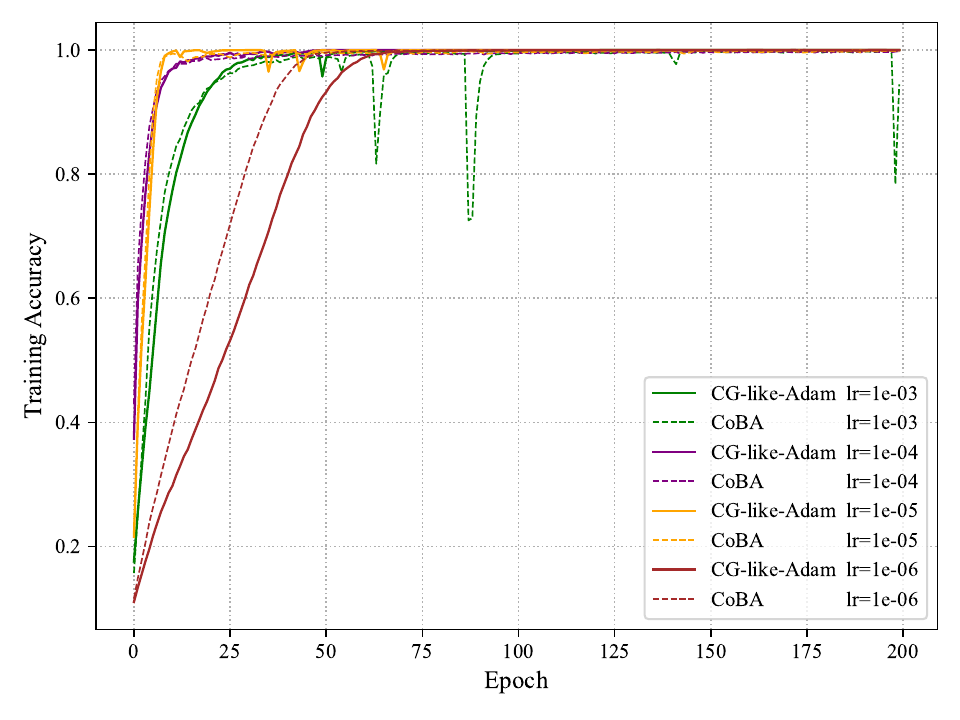}
}
    \centering
	\subfloat[Testing Accuracy\label{fig5c}]{
		\centering
		\includegraphics[width=0.31\linewidth]{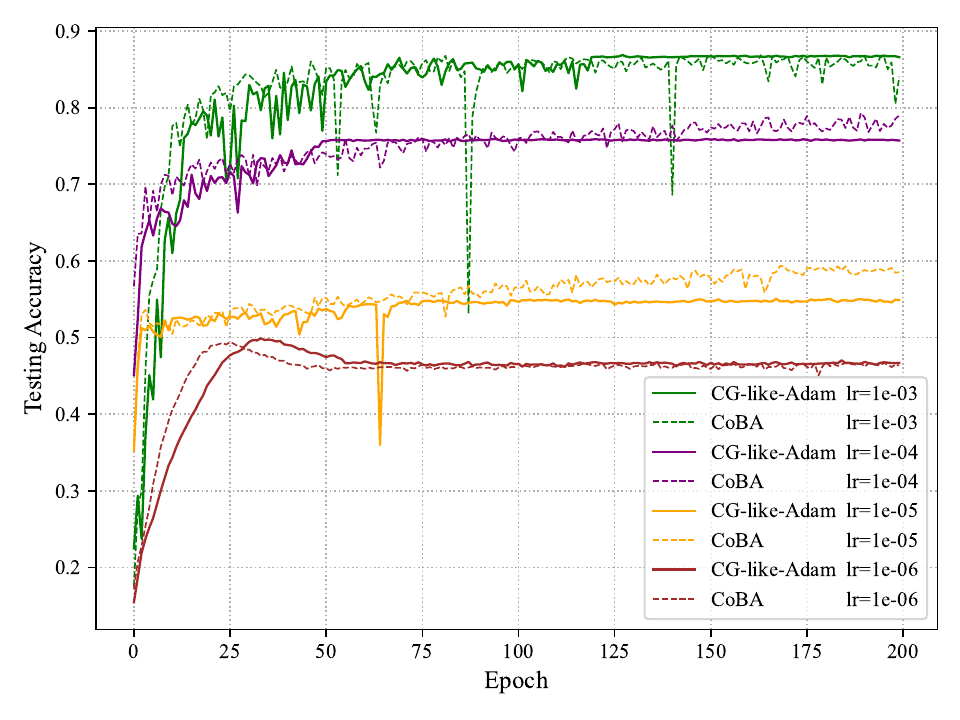}
}
	% \caption{CG-like-Adam V.S. CoBA under different learning rates. Train VGG-19 on CIFAR-10. 
    %         The conjugate coefficient is calculated by HZ(\ref{eq8}).}
    \caption{CG-like-Adam V.S. CoBA under different learning rates. 
    (VGG-19, CIFAR-10, HZ(\ref{eq8}))}
	\label{fig5}
\end{figure}

From figure \ref{fig1}-\ref{fig5}, as you can see, no matter what type of conjugate coefficient 
calculation method is employed, CG-like-Adam keeps the training loss of VGG-19 to the minimum, 
except the learning rate $\alpha_{t}=10^{-6}$. The training accuracy finally hits 100\% 
through the optimization of CG-like-Adam and CoBA, however, CG-like-Adam achieves this 
goal faster and more stable. More importantly, our algorithm performs better than CoBA 
on test dataset when the learning rate $\alpha_{t}=10^{-3}$, although at other learning rates, 
the testing accuracy of CG-like-Adam is similar to or a little worse than that of CoBA.

\begin{figure}[htbp]
	\centering
	\subfloat[(Log.) Training Loss\label{fig6a}]{
		\centering
		\includegraphics[width=0.31\linewidth]{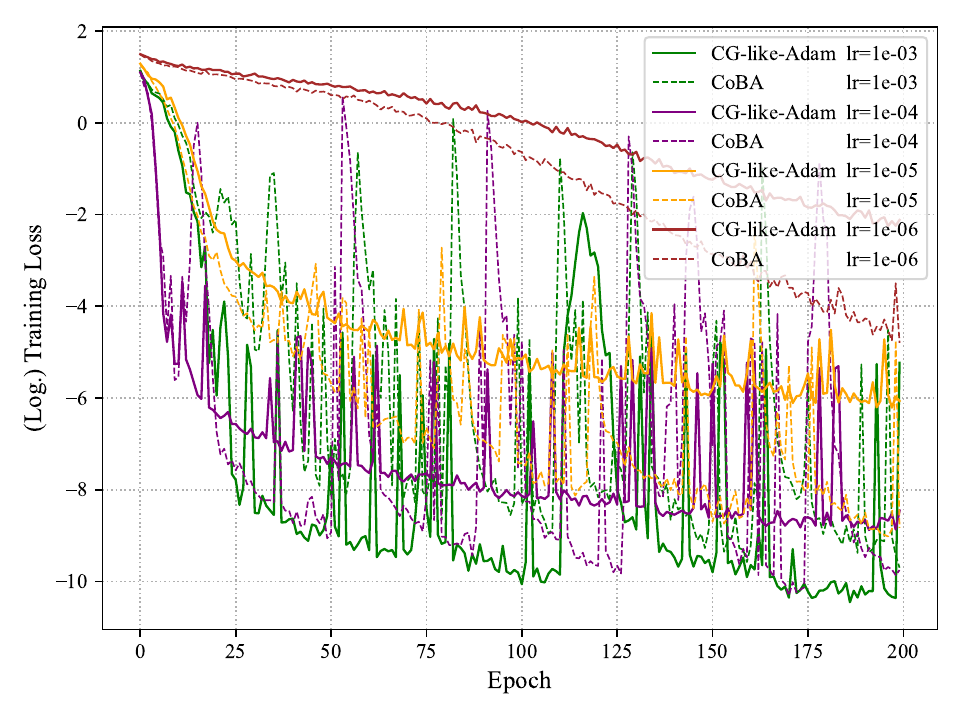}
}
    \centering
	\subfloat[Training Accuracy\label{fig6b}]{
		\centering
		\includegraphics[width=0.31\linewidth]{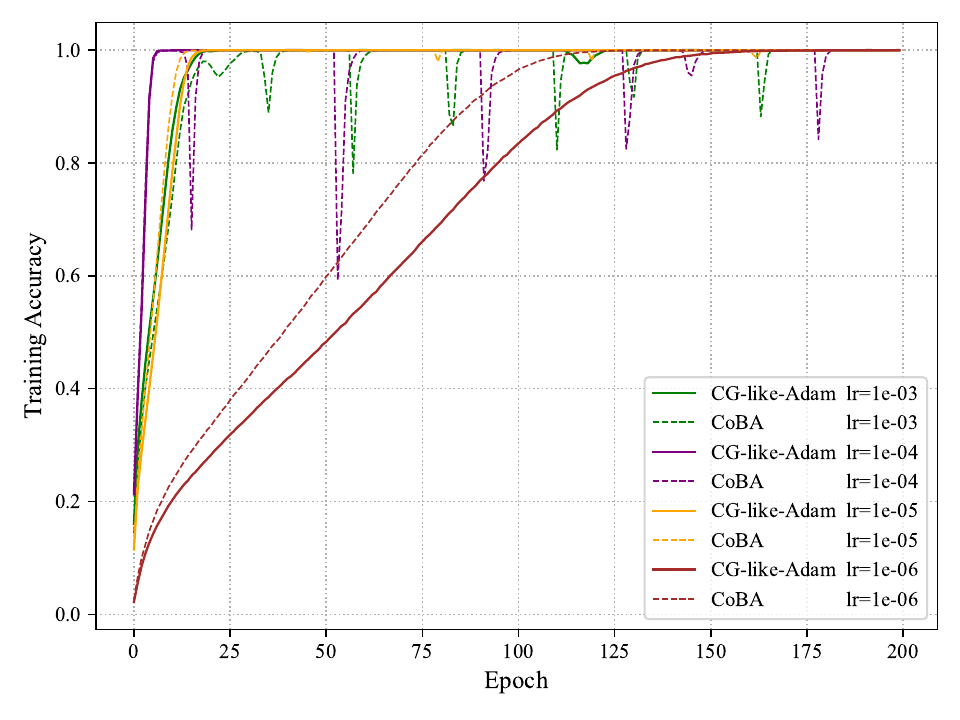}
}
    \centering
	\subfloat[Testing Accuracy\label{fig6c}]{
		\centering
		\includegraphics[width=0.31\linewidth]{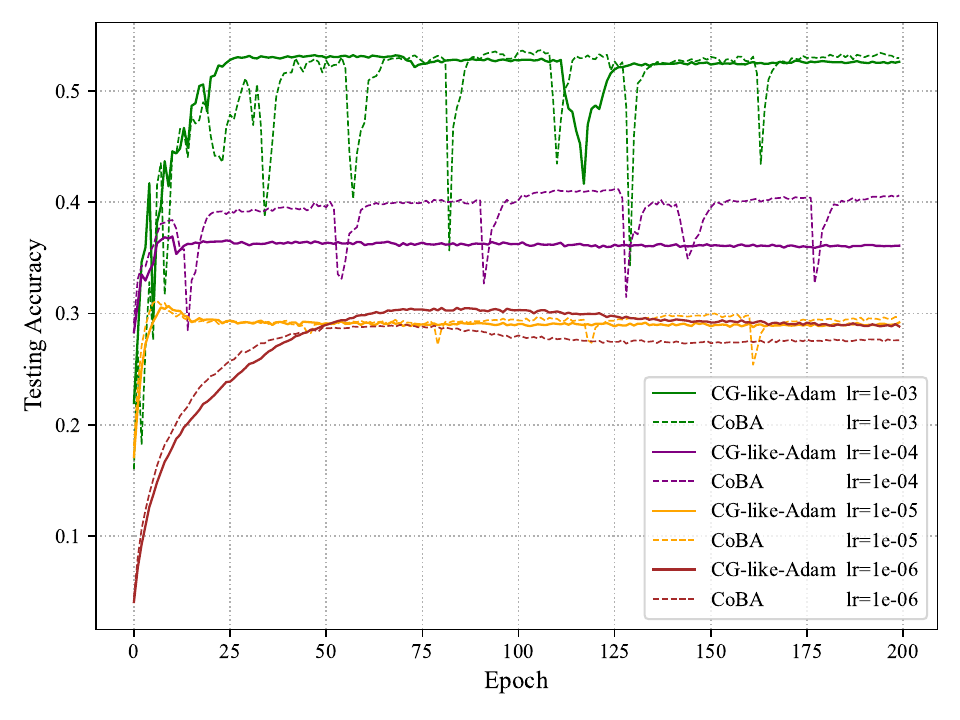}
}
	% \caption{CG-like-Adam V.S. CoBA under different learning rates. Train ResNet-34 on CIFAR-100. 
    %         The conjugate coefficient is calculated by HS(\ref{eq4}).}
    \caption{CG-like-Adam V.S. CoBA under different learning rates. 
    (ResNet-34, CIFAR-100, HS(\ref{eq4}))}
	\label{fig6}
\end{figure}

\begin{figure}[htbp]
	\centering
	\subfloat[(Log.) Training Loss\label{fig7a}]{
		\centering
		\includegraphics[width=0.31\linewidth]{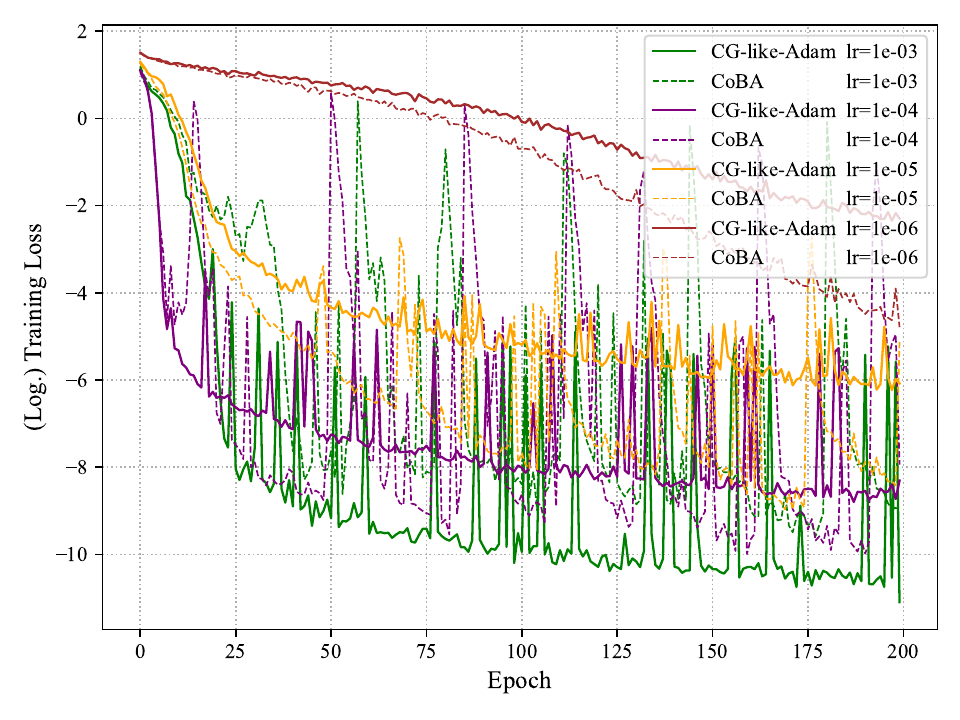}
}
    \centering
	\subfloat[Training Accuracy\label{fig7b}]{
		\centering
		\includegraphics[width=0.31\linewidth]{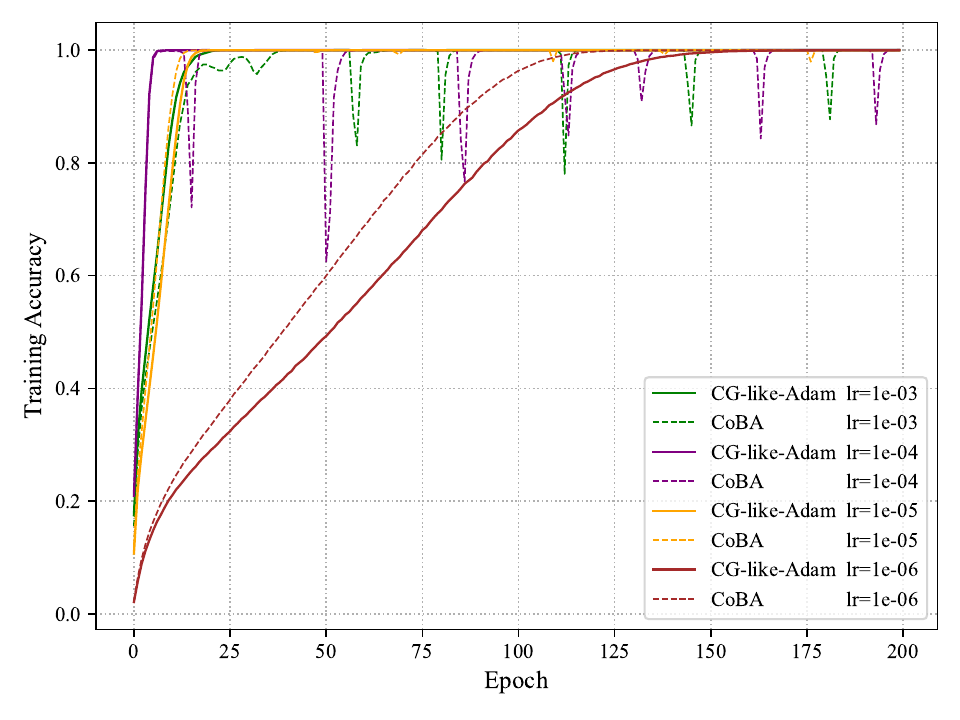}
}
    \centering
	\subfloat[Testing Accuracy\label{fig7c}]{
		\centering
		\includegraphics[width=0.31\linewidth]{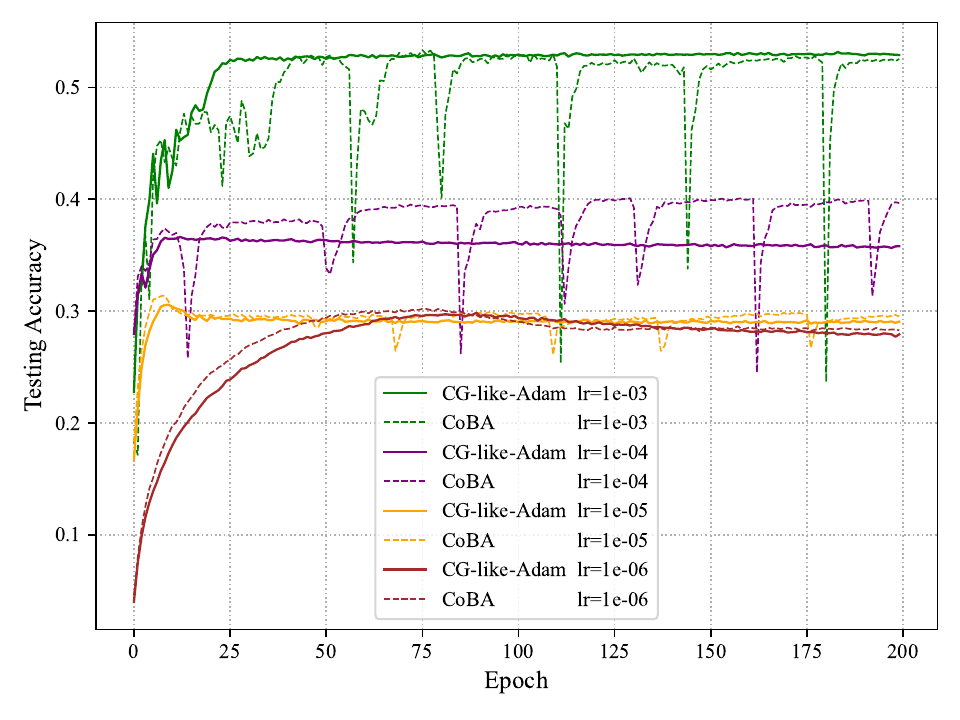}
}
	% \caption{CG-like-Adam V.S. CoBA under different learning rates. Train ResNet-34 on CIFAR-100. 
    %         The conjugate coefficient is calculated by FR(\ref{eq5}).}
    \caption{CG-like-Adam V.S. CoBA under different learning rates. 
    (ResNet-34, CIFAR-100, FR(\ref{eq5}))}
	\label{fig7}
\end{figure}

\begin{figure}[htbp]
	\centering
	\subfloat[(Log.) Training Loss\label{fig8a}]{
		\centering
		\includegraphics[width=0.31\linewidth]{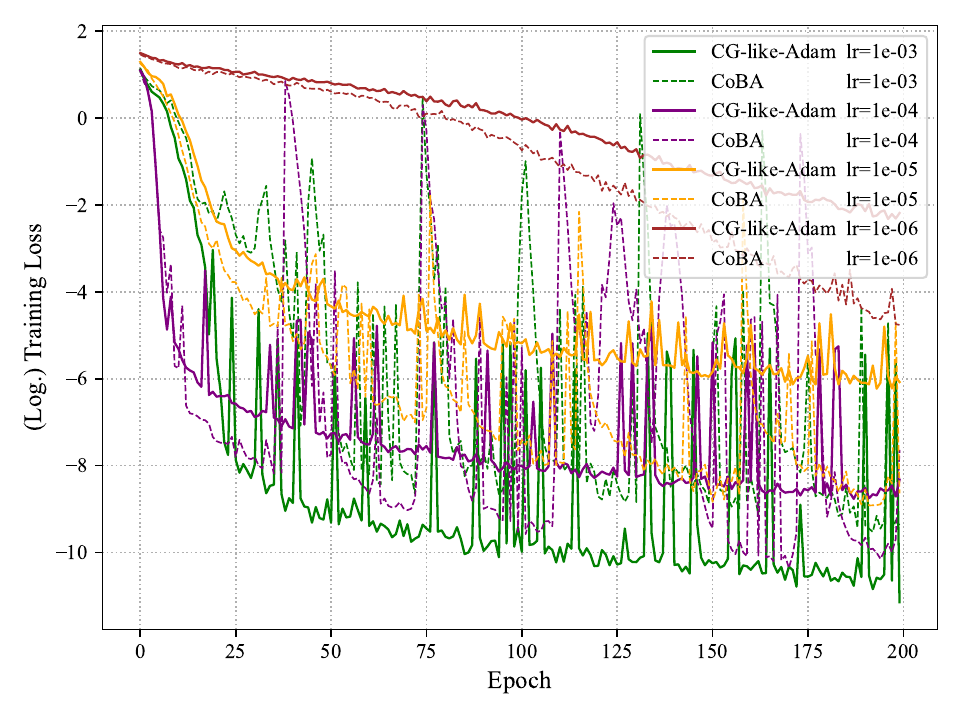}
}
    \centering
	\subfloat[Training Accuracy\label{fig8b}]{
		\centering
		\includegraphics[width=0.31\linewidth]{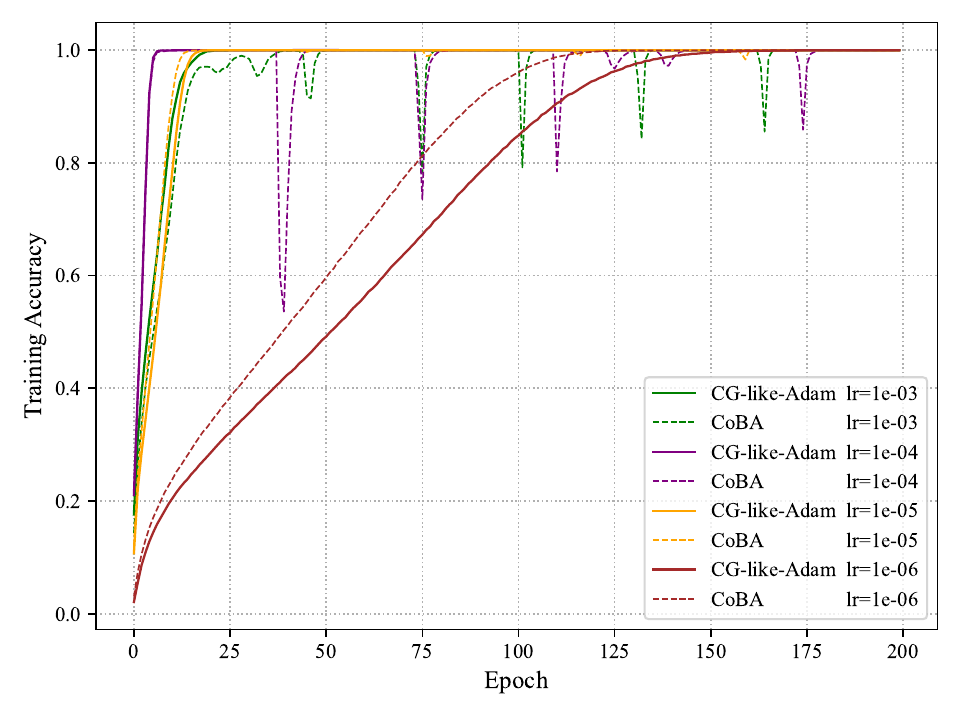}
}
    \centering
	\subfloat[Testing Accuracy\label{fig8c}]{
		\centering
		\includegraphics[width=0.31\linewidth]{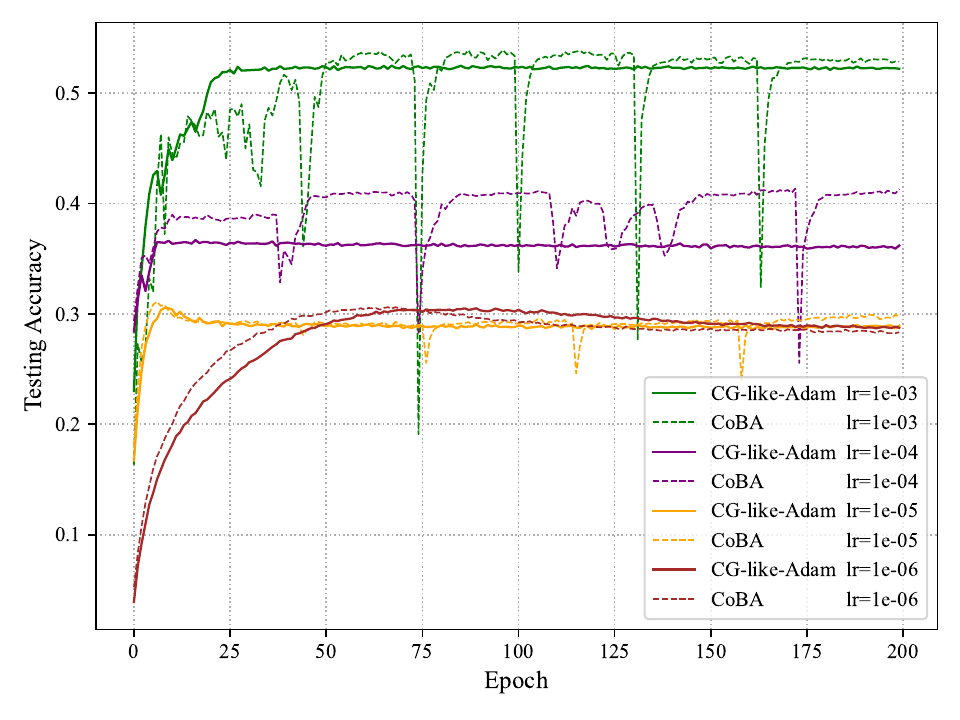}
}
	% \caption{CG-like-Adam V.S. CoBA under different learning rates. Train ResNet-34 on CIFAR-100. 
    %         The conjugate coefficient is calculated by PRP(\ref{eq6}).}
    \caption{CG-like-Adam V.S. CoBA under different learning rates. 
    (ResNet-34, CIFAR-100, PRP(\ref{eq6}))}
	\label{fig8}
\end{figure}

Figure \ref{fig6}-\ref{fig10} show the results of the experiments of ResNet-34 on
CIFAR-100. Although training loss failed to reach the minimum unless the learning rate 
$\alpha_{t}=10^{-3}$, CG-like-Adam obtained 100\% training accuracy in less than 10 epochs
when the learning rate is not too small, which, in our opinion, leads to overfitting and 
thus the testing accuracy is inferior(or similar) to CoBA. 

\begin{figure}[htbp]
	\centering
	\subfloat[(Log.) Training Loss\label{fig9a}]{
		\centering
		\includegraphics[width=0.31\linewidth]{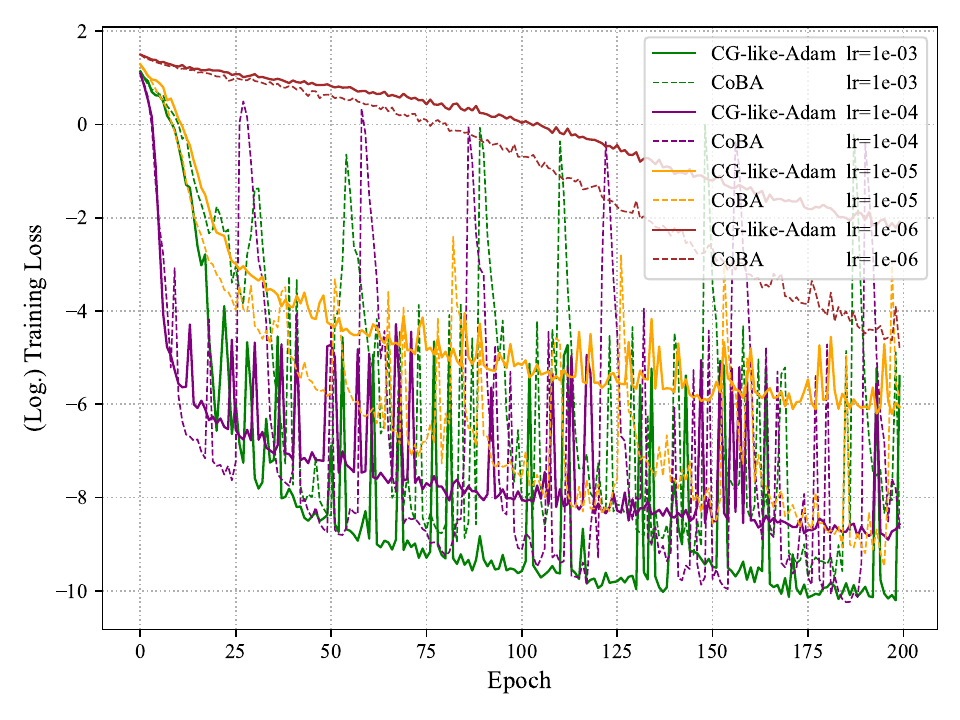}
}
    \centering
	\subfloat[Training Accuracy\label{fig9b}]{
		\centering
		\includegraphics[width=0.31\linewidth]{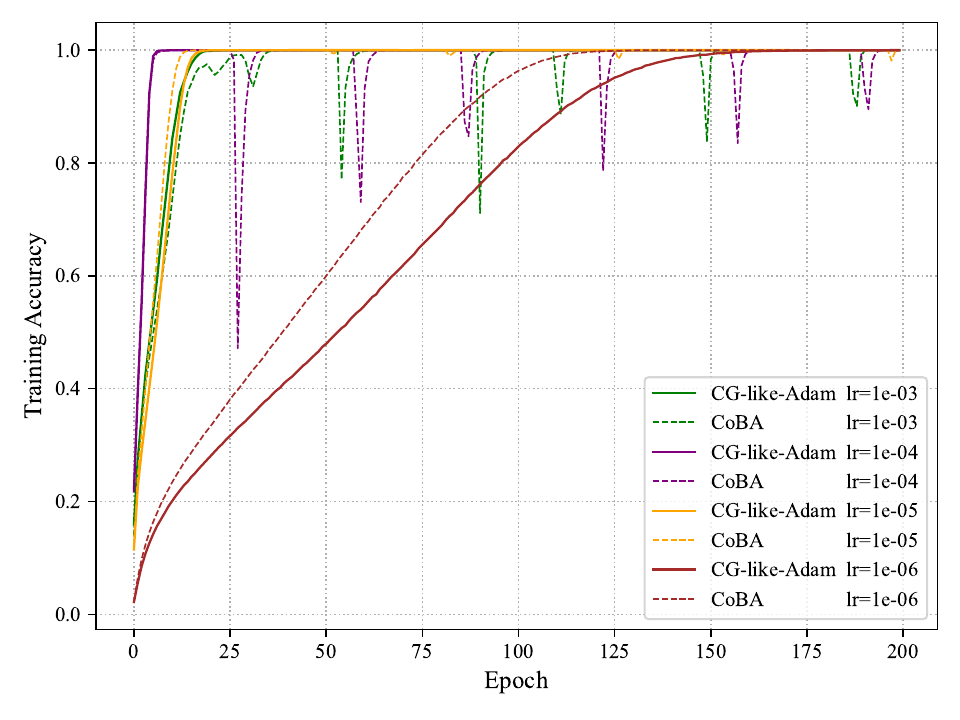}
}
    \centering
	\subfloat[Testing Accuracy\label{fig9c}]{
		\centering
		\includegraphics[width=0.31\linewidth]{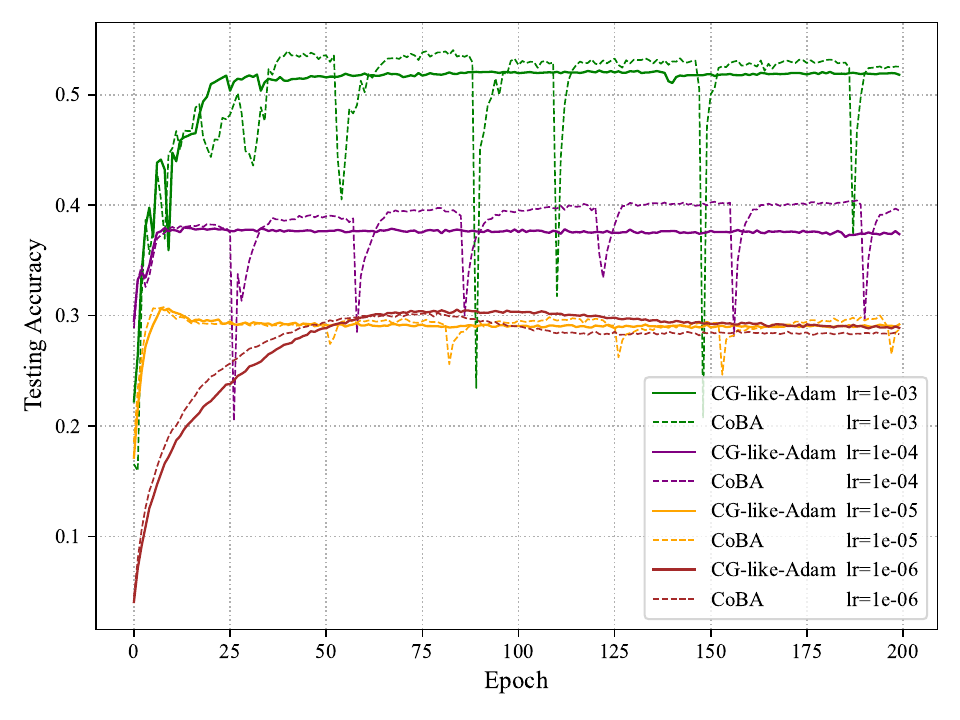}
}
	% \caption{CG-like-Adam V.S. CoBA under different learning rates. Train ResNet-34 on CIFAR-100. 
    %         The conjugate coefficient is calculated by DY(\ref{eq7}).}
    \caption{CG-like-Adam V.S. CoBA under different learning rates. 
    (ResNet-34, CIFAR-100, DY(\ref{eq7}))}
	\label{fig9}
\end{figure}

\begin{figure}[htbp]
	\centering
	\subfloat[(Log.) Training Loss\label{fig10a}]{
		\centering
		\includegraphics[width=0.31\linewidth]{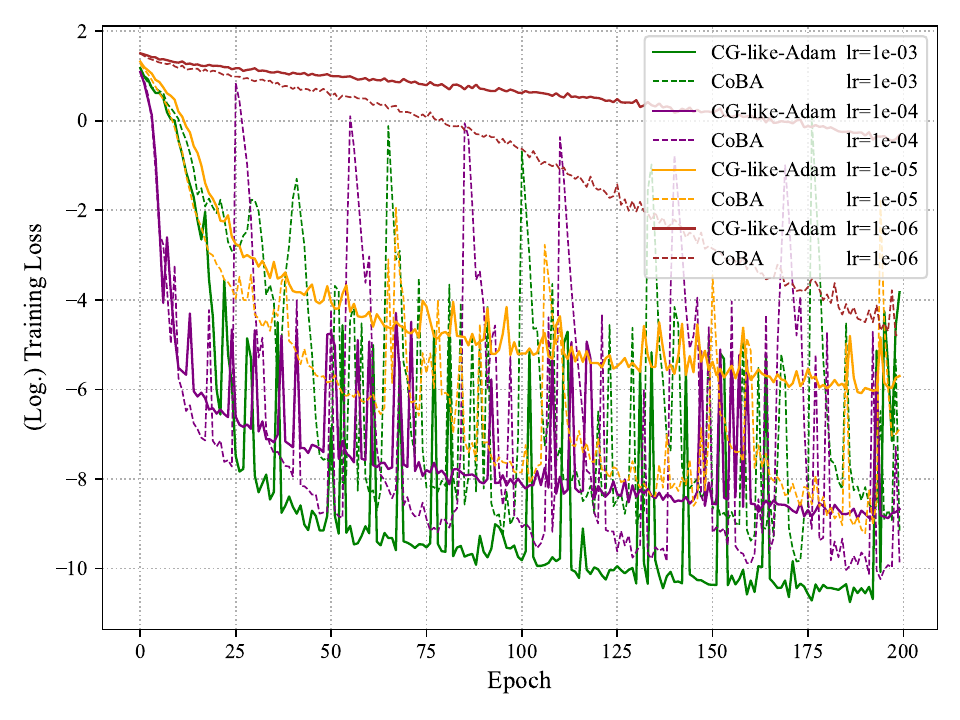}
}
    \centering
	\subfloat[Training Accuracy\label{fig10b}]{
		\centering
		\includegraphics[width=0.31\linewidth]{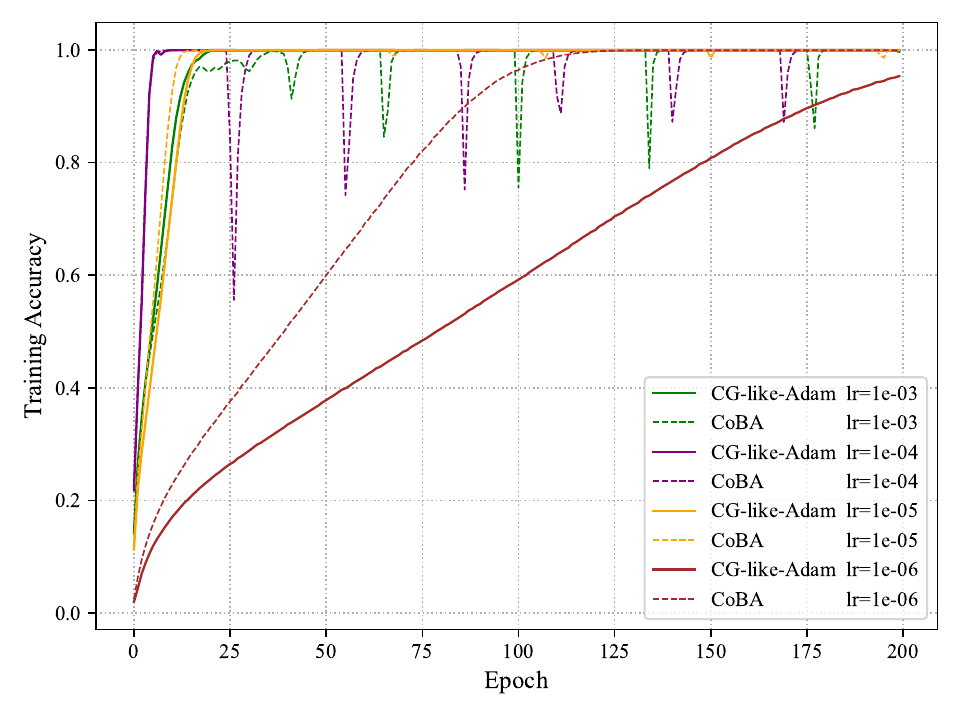}
}
    \centering
	\subfloat[Testing Accuracy\label{fig10c}]{
		\centering
		\includegraphics[width=0.31\linewidth]{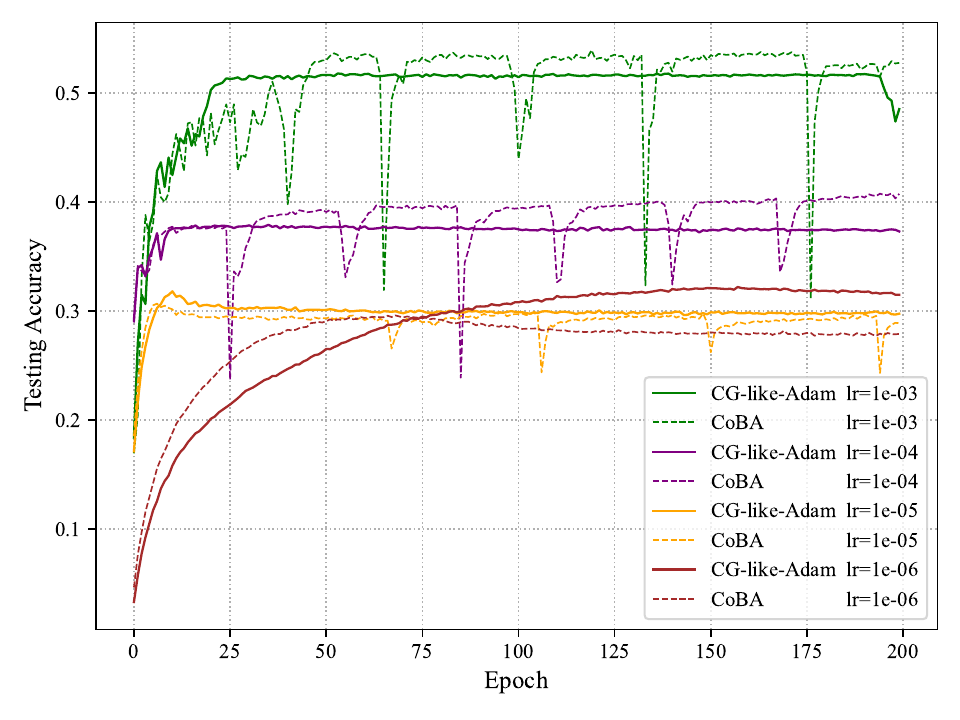}
}
	% \caption{CG-like-Adam V.S. CoBA under different learning rates. Train ResNet-34 on CIFAR-100. 
    %         The conjugate coefficient is calculated by HZ(\ref{eq8}).}
    \caption{CG-like-Adam V.S. CoBA under different learning rates. 
    (ResNet-34, CIFAR-100, HZ(\ref{eq8}))}
	\label{fig10}
\end{figure}

\subsection{Compare CG-like-Adam with Adam}
$\alpha_{t}=10^{-3}(\forall t \in \mathcal{T})$ is set as default value for this experiment.
Figure \ref{fig11} and figure \ref{fig12} show the results of the experiments of 
VGG-19 on CIFAR-10, ResNet-34 on CIFAR-100, respectively.
\begin{figure}[htbp]
	\centering
	\subfloat[(Log.) Training Loss\label{fig11a}]{
		\centering
		\includegraphics[width=0.31\linewidth]{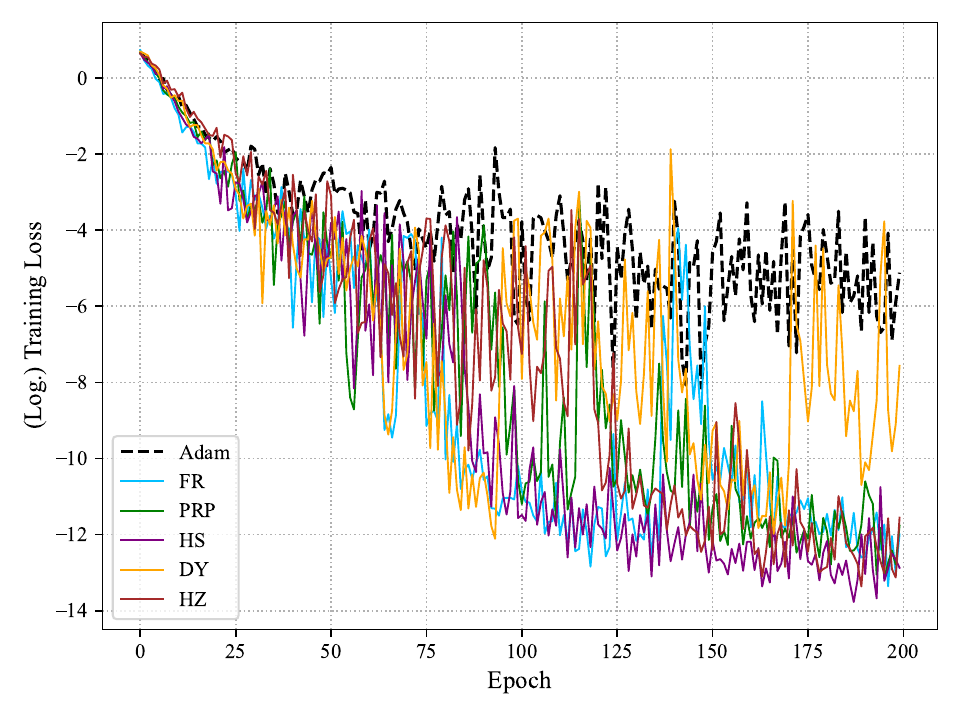}
}
    \centering
	\subfloat[Training Accuracy\label{fig11b}]{
		\centering
		\includegraphics[width=0.31\linewidth]{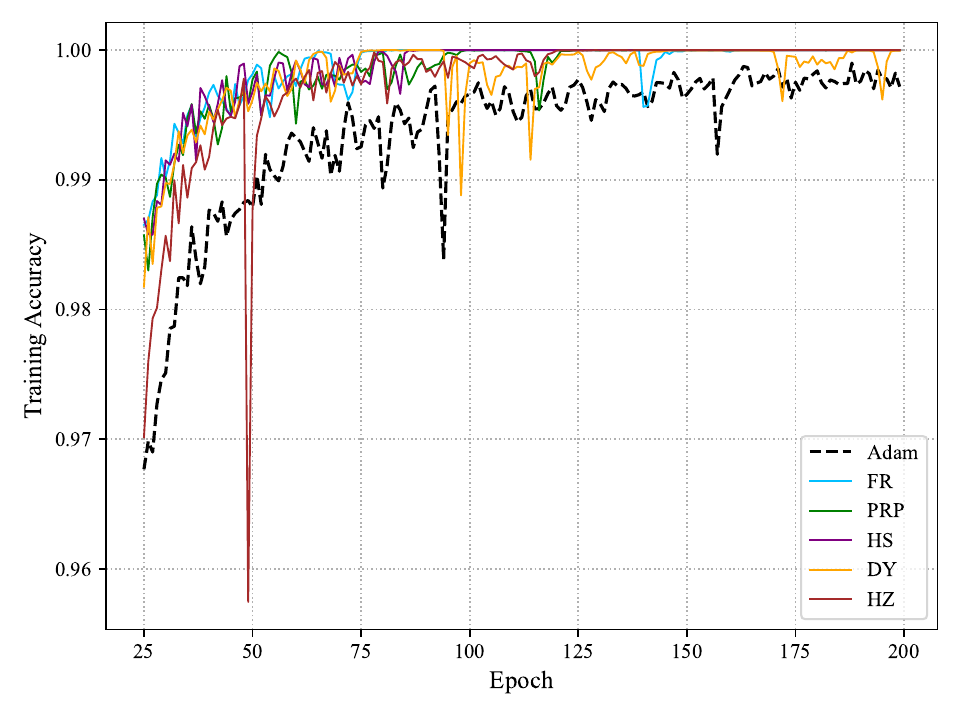}
}
    \centering
	\subfloat[Testing Accuracy\label{fig11c}]{
		\centering
		\includegraphics[width=0.31\linewidth]{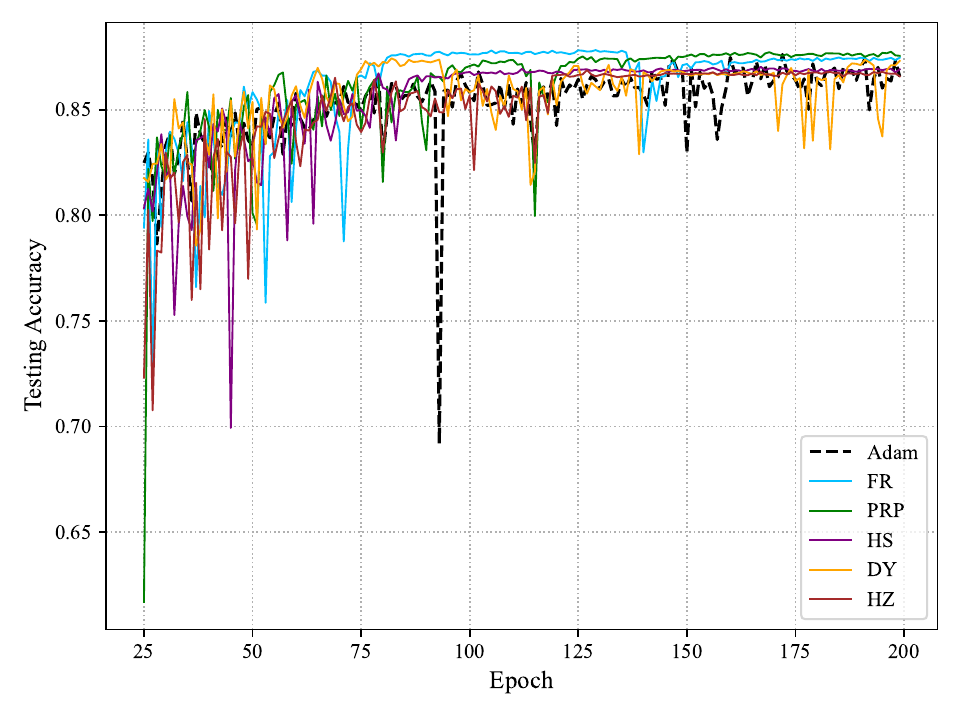}
}
	\caption{Adam V.S. CG-like-Adam with different conjugate coefficient. Train VGG-19 on CIFAR-10.}
	\label{fig11}
\end{figure}
From figure \ref{fig11}, CG-like-Adam is superior to Adam in the criterion of training loss,
training accuracy and testing accuracy. Figure \ref{fig12} tells that it also attains the 
minimum of training loss although it has some vibrations which, however, is more stable than Adam. In addition,
CG-like-Adam defeats Adam in terms of training accuracy and testing accuracy and performs more consistently.

We trained VGG-19 on CIFAR-100 as well. Conclusion can be drawn from the results 
showed as figure \ref{fig13} that CG-like-Adam performed better than Adam once again 
except the conjugate coefficient HZ(Eq.\eqref{eq8}), which may be upgraded by adjusting its
$\lambda$.

\begin{figure}[h]
	\centering
	\subfloat[(Log.) Training Loss\label{fig12a}]{
		\centering
		\includegraphics[width=0.31\linewidth]{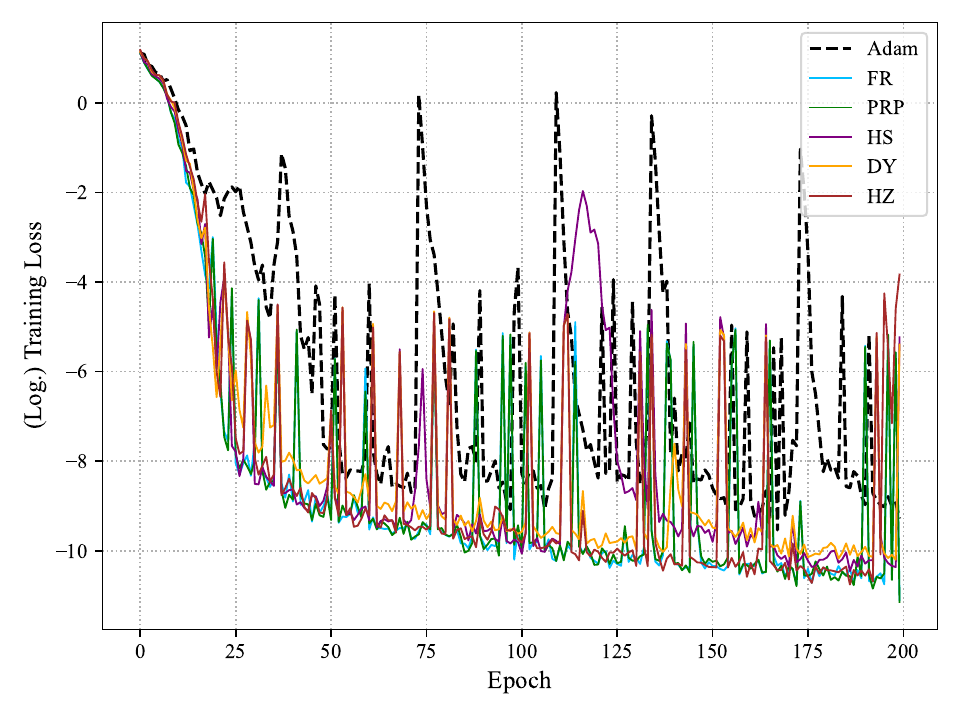}
}
    \centering
	\subfloat[Training Accuracy\label{fig12b}]{
		\centering
		\includegraphics[width=0.31\linewidth]{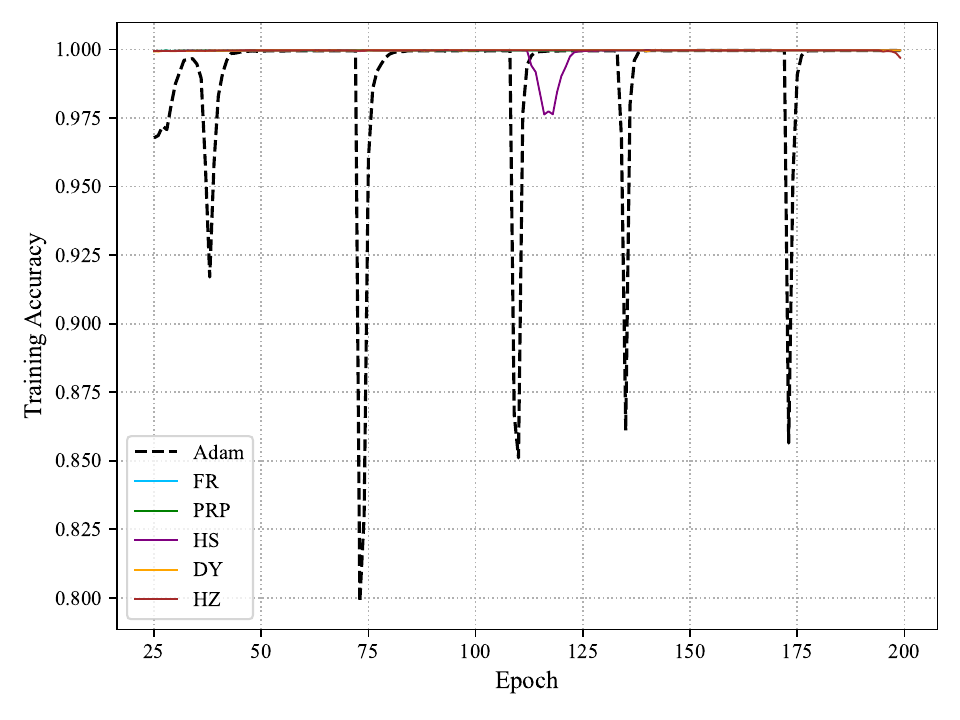}
}
    \centering
	\subfloat[Testing Accuracy\label{fig12c}]{
		\centering
		\includegraphics[width=0.31\linewidth]{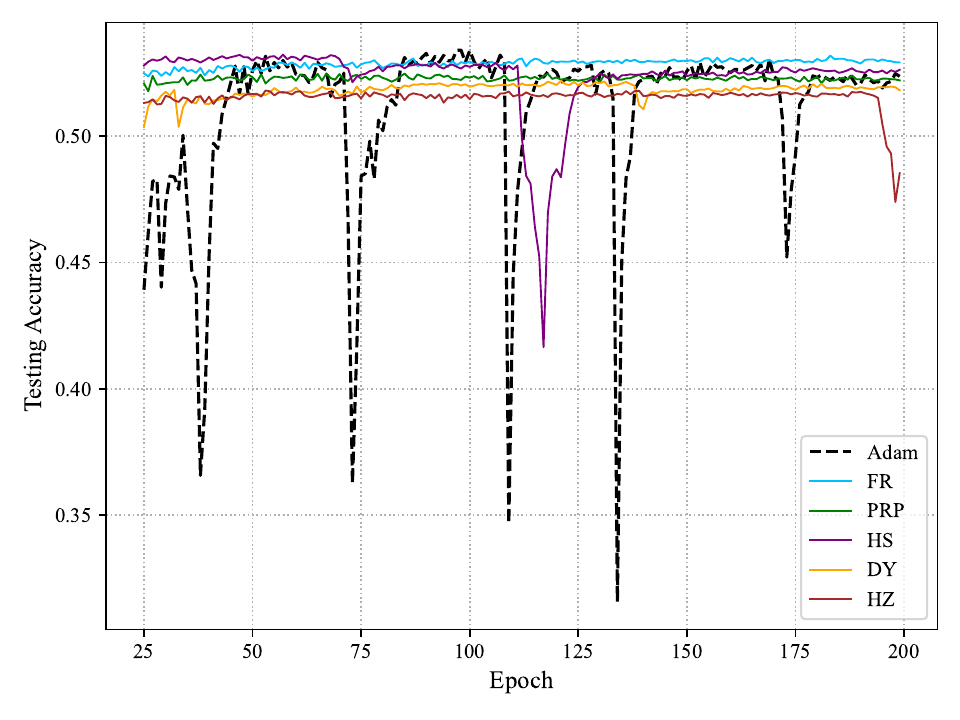}
}
	\caption{Adam V.S. CG-like-Adam with different conjugate coefficient. Train ResNet-34 on CIFAR-100.}
	\label{fig12}
\end{figure}

\newpage
\begin{figure}[h]
	\centering
	\subfloat[(Log.) Training Loss\label{fig13a}]{
		\centering
		\includegraphics[width=0.31\linewidth]{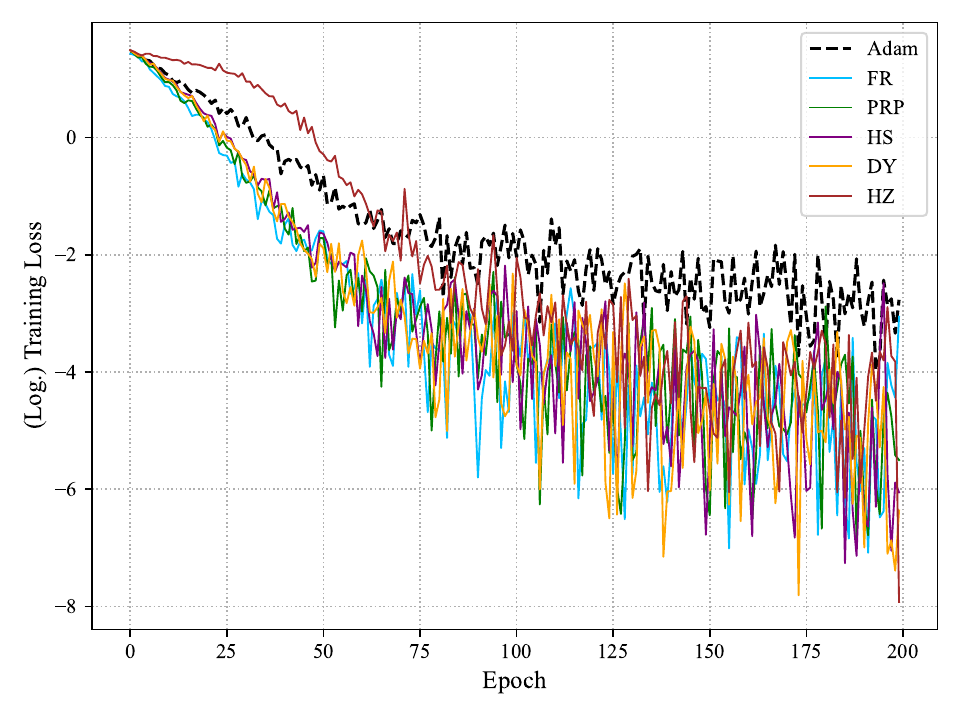}
}
    \centering
	\subfloat[Training Accuracy\label{fig13b}]{
		\centering
		\includegraphics[width=0.31\linewidth]{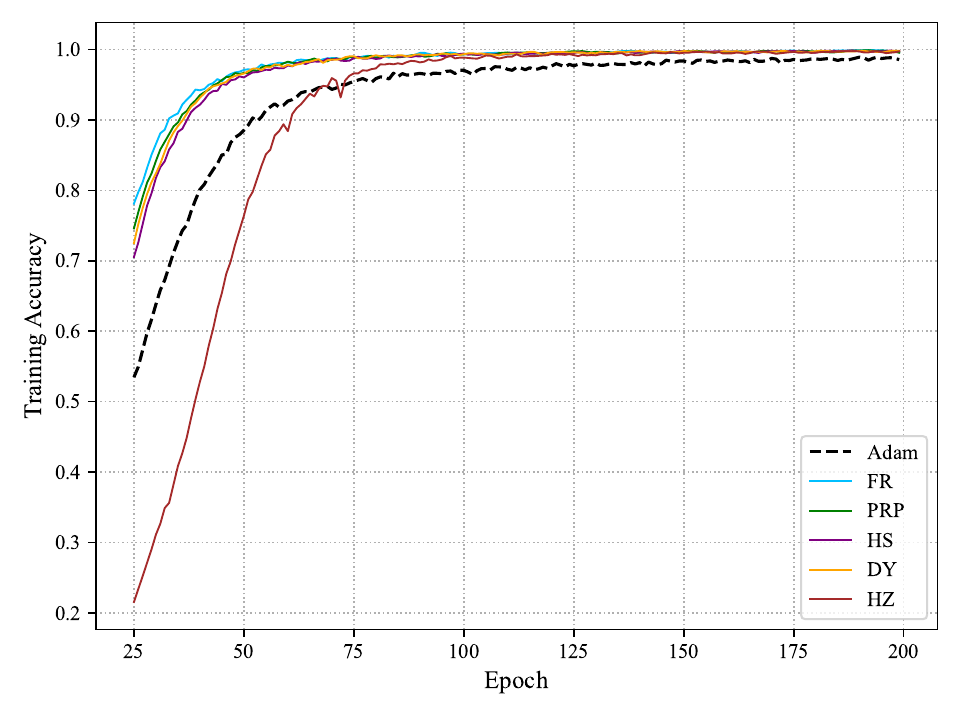}
}
    \centering
	\subfloat[Testing Accuracy\label{fig13c}]{
		\centering
		\includegraphics[width=0.31\linewidth]{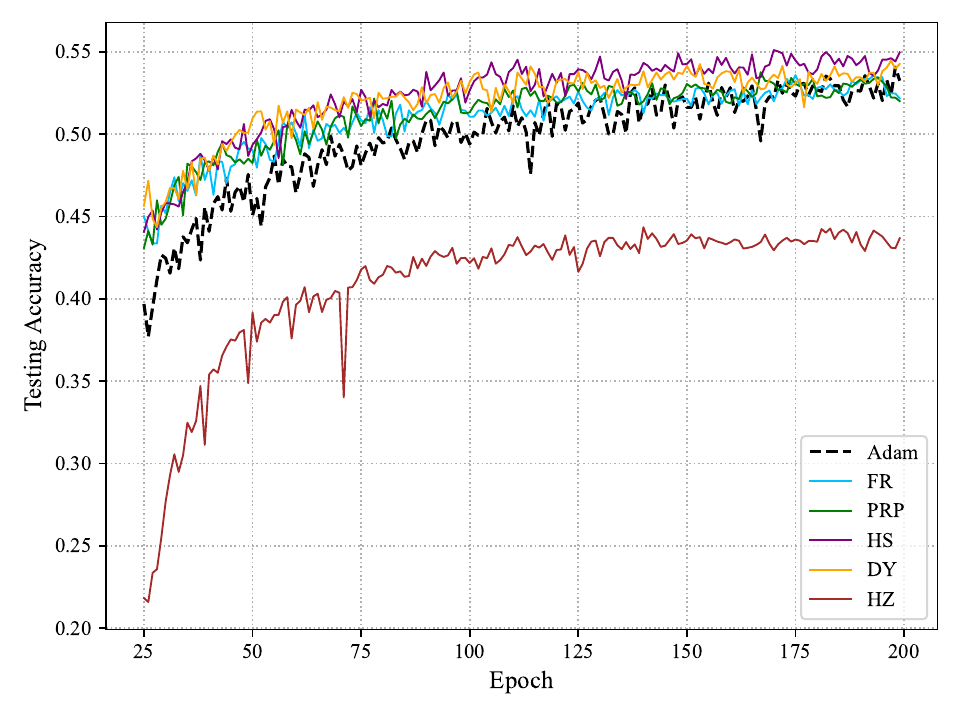}
}
	\caption{Adam V.S. CG-like-Adam with different conjugate coefficient. Train VGG-19 on CIFAR-100.}
	\label{fig13}
\end{figure}

\section{Conclusion}
\label{}
In this paper, for the purpose of accelerating deep neural networks training and 
helping the network find more optimal parameters, the conjugate-gradient-like is 
incorporated into the generic Adam, which is named CG-like-Adam. 
The conjugate-gradient-like is modified from vanilla conjugate gradient, via scaling 
the conjugate coefficient by using a decreasing sequence over time step.
The first-order and the second-order moment estimation of CG-like-Adam 
are both adopting conjugate-gradient-like. We theoretically prove the convergence
of our algorithm and manage to not only provide convergence for the non-convex but 
also deal with the cases where the exponential moving average coefficient of the 
first-order moment estimation is constant and the first-order moment estimation is unbiased.
Numerical experiments of training VGG-19/ResNet-34 on CIFAR-10/100 for image classification 
demonstrate effectiveness and better performance. Our algorithm performs more stable and 
arrives at 100\% training accuracy faster than Adam. Higher testing accuracy provides strong evidence 
that our algorithm obtains more optimal parameters of deep neural networks.
More future work includes conducting various experiments for deep learning task, 
applying variance reduction technique for the conjugate-gradient-like, exploring 
line-search for finding a suitable stepsize, etc.

\newpage

% references
\bibliographystyle{elsarticle-num}
\bibliography{references.bib}
\biboptions{sort&compress}

\newpage

%\appendix
%\section{AGWE}
%\label{}

\appendix

\appendix
\section{Proof of Some Lemmas} %Appendix A: 
\label{prlem}
Before beginning the proof of theorem \ref{th3.1}, \ref{th3.2} and corollary \ref{cor31},
all the necessary lemmas as well as the proof of them are firstly provided at once.
For the convenience of the following proof and without the loss of generality, let 
$\left(\frac{\alpha_{1}}{\sqrt{\hat{V}_{1}}}-\frac{\alpha_{0}}{\sqrt{\hat{V}_{0}}}\right) \hat{m}_{0}=0$.

\begin{myLem}\label{lem8}
    Suppose $\forall t \in \mathcal{T}$, $\beta_{1 t} \in[0,1)$, 
    $\beta_{1(t+1)} \leqslant \beta_{1 t}$. 
    Defining $\eta_{t}=\frac{\beta_{1 t}\left(1-\beta_{11}^{t-1}\right)}{\xi_{t}}$ 
    and $\xi_{t}=\left(1-\beta_{11}^{t}\right)\left[1-\frac{\beta_{1t}\left(1-\beta_{11}^{t-1}\right)}{1-\beta_{11}^{t}}\right]$,
    then $\forall t \in \mathcal{T}$, the following holds:
    \begin{equation}\label{eqA42}
        \beta_{1t}<\frac{1-\beta_{11}^{t}}{1-\beta_{11}^{t-1}}, \
        \left(1-\beta_{11}\right)^{2} \leqslant \xi_{t} \leqslant 1, \
        0 \leqslant \eta_{t} \leqslant \frac{1}{1-\beta_{11}}.
    \end{equation}
\end{myLem}

\begin{proof}
  It is obvious that $\frac{1-\beta_{11}^{t}}{1-\beta_{11}^{t-1}} \geqslant \frac{1-\beta_{11}^{t-1}}{1-\beta_{11}^{t-1}}=1$.
  Because of $\beta_{1 t} \in[0,1)$, hence $\forall t \in \mathcal{T}$, $\beta_{1 t}<\frac{1-\beta_{11}^{t}}{1-\beta_{11} ^{t-1}}$,
  which is equivalent to $\frac{\beta_{1 t}\left(1-\beta_{11}^{t-1}\right)}{1-\beta_{11}^{t}}<1$.
  So then $\xi_{t} \leqslant 1-\beta^{t}_{11} \leqslant 1$.

  If $\forall t \in \mathcal{T}$, $\beta_{1t}=0$, obviously, $\xi_{t}=1$, $\eta_{t}=0$.

  If $\beta_{1t} \in (0,1)$,
  \begin{equation*}
      \begin{aligned}
          \xi_{t}
          & =\left(1-\beta_{11}^{t}\right)\left[1-\frac{\beta_{1t}\left(1-\beta_{11}^{t-1}\right)}{1-\beta_{11}^{t}}\right]
          %& %=\left(1-\beta_{11}^{t}\right)-\bet%a_{1 %t}\left(1-\beta_{11}^{t-1}\right) %\\
          \geqslant\left(1-\beta_{11}^{t-1}\right)-\beta_{1 t}\left(1-\beta_{11}^{t-1}\right) \\
          & =\left(1-\beta_{1t}\right)\left(1-\beta_{11}^{t-1}\right) \geqslant\left(1-\beta_{11}\right)^{2} , 
      \end{aligned}
  \end{equation*}

  \begin{equation*}
      \begin{aligned}
          0 < \eta_{t} \leqslant \frac{\beta_{1 t}\left(1-\beta_{11}^{t-1}\right)}{\left(1-\beta_{1 t}\right)\left(1-\beta_{11}^{t-1}\right)}=\frac{\beta_{1 t}}{1-\beta_{1 t}} \leqslant \frac{1}{1-\beta_{11}}.        
      \end{aligned}
  \end{equation*}
    
  The proof is over.
\end{proof}
\begin{myLem}\label{lem9}
  Suppose $\forall t \in \mathcal{T}$, $\beta_{1 t} \in[0,1)$, 
  $\beta_{1(t+1)} \leqslant \beta_{1 t}$. 
  Define $h(t)=\frac{\left(1-\beta_{11}^{t-1}\right)\left(1-\beta_{11}^{t+1}\right)}{\left(1-\beta_{11}^{t}\right)^{2}}$.
  $\forall t \in \mathcal{T}$, if ${\beta_{1(t+1)} \leqslant(\geqslant) \beta_{1 t} h(t)}$, 
  then $\eta_{t+1} \leqslant(\geqslant) \eta_{t}$ holds.
\end{myLem}

\begin{proof}
  $\forall t \in \mathcal{T}$, $h(t) > 0$, 
  \begin{equation*}
      \begin{aligned}
          \left(1-\beta_{11}^{t-1}\right)\left(1-\beta_{11}^{t+1}\right)-\left(1-\beta_{11}^{t}\right)^{2} 
%            & =1-\beta_{11}^{t-1}-\beta_{11}^{t+1}+\beta_{11}^{2 t}-\left(1-2 \beta_{11}^{t}+\beta_{11}^{2 t}\right) \\
%            & =2 \beta_{11}^{t}-\beta_{11}^{t-1}-\beta_{11}^{t+1}\\
          & =\beta_{11}^{t-1}\left(2 \beta_{11}-1-\beta_{11}^{2}\right).
      \end{aligned}
  \end{equation*}

  Because of $\beta_{11} \in [0,1)$, there is $-1 \leqslant 2 \beta_{11}-1-\beta_{11}^{2}<0$. 
  Such that
  \begin{align}
    \left(1-\beta_{11}^{t-1}\right)\left(1-\beta_{11}^{t+1}\right) \leqslant\left(1-\beta_{11}^{t}\right)^{2} , \notag
  \end{align}
  which is equivalent to $h(t) \leqslant 1$.

  From the lemma \ref{lem8}, $\forall t \in \mathcal{T}$, $\xi_{t} > 0$.
  Let $\delta_{t}=\beta_{1(t+1)}\left(1-\beta_{11}^{t}\right) \xi_{t}-\beta_{1 t}\left(1-\beta_{11}^{t-1}\right) \xi_{t+1}$,
  then 
  \begin{equation*}
      \begin{aligned}
          \eta_{t+1}-\eta_{t}
%            & =\frac{\beta_{1(t+1)}\left(1-\beta_{11}^{t}\right) \xi_{t}-\beta_{1 t}\left(1-\beta_{11}^{t-1}\right) \xi_{t+1}}{\xi_{t} \xi_{t+1}}\\
          & =\frac{\delta_{t}}{\xi_{t} \xi_{t+1}},
      \end{aligned}
  \end{equation*}
  \begin{equation*}
      \begin{aligned}
          \delta_{t}
%            & =\beta_{1(t+1)}\left(1-\beta_{11}^{t}\right)\left[\left(1-\beta_{11}^{t}\right)-\beta_{1 t}\left(1-\beta_{11}^{t-1}\right)\right] \\ 
%            & \quad\, -\beta_{1 t}\left(1-\beta_{11}^{t-1}\right)\left[\left(1-\beta_{11}^{t+1}\right)-\beta_{1(t+1)}\left(1-\beta_{11}^{t}\right)\right] \\
          & =\beta_{1(t+1)}\left(1-\beta_{11}^{t}\right)^{2}-\beta_{1t}\left(1-\beta_{11}^{t-1}\right)\left(1-\beta_{11}^{t+1}\right).
      \end{aligned}
  \end{equation*}

  Then $\forall t \in \mathcal{T}$, different situations will be discussed below:

  (1) If $\beta_{1t}=0$, then $\delta_{t}=0$, $\eta_{t+1}-\eta_{t}=0$.

  (2) If $\beta_{1t}=\beta_{1(t+1)} \neq 0$, then $h(t) \leqslant 1=\frac{\beta_{1(t+1)}}{\beta_{1t}}$,
  $\delta_{t} \geqslant 0$, hence $\eta_{t+1} \geqslant \eta_{t}$.

  (3) If $\beta_{1t} \in (0,1)$, 
  because of ${\beta_{1(t+1)} \leqslant(\geqslant) \beta_{1 t} h(t)}$, 
  such that 
  $\delta_{t} \leqslant(\geqslant)0$,
  which leads to $\eta_{t+1} \leqslant(\geqslant) \eta_{t}$.

  The proof is over.
\end{proof}
\begin{myLem}\label{lem10}
  Suppose the assumption \ref{ass34} is satisfied 
  and $\exists \bar{\gamma} \in \mathbb{R}^{+}$, $\forall t \in \mathcal{T}$, 
  $\left|\gamma_{t}\right| \leqslant \bar{\gamma}$.
  Further suppose $\exists t_{0} \in \mathcal{T}$, $\exists \bar{H} \in \mathbb{R}^{+}$,
  such that $\bar{H}=\max \left\{2 H, \ \underset{t \in\left\{1, \cdots, t_{0}-1\right\}}{\max} \left\|d_{t}\right\|\right\}$.
  Then $\forall t \in \mathcal{T}$, the following holds:
  \begin{equation}\label{eqA43}
      \left\|d_{t}\right\| \leqslant \bar{H}.
  \end{equation}
\end{myLem}

\begin{proof}
  Due to $\underset{t \rightarrow+\infty}{\lim} \frac{\left|\gamma_{t}\right|}{t^{a}}=0$, 
  thus $\exists t_{0} \in \mathcal{T}$, $\forall t \geqslant t_{0}$, the following holds:
  \begin{equation*}
      \frac{\left|\gamma_{t}\right|}{t^{a}} \leqslant \frac{1}{2}.
  \end{equation*}

  By the definition of $\bar{H}$, it is apparent that $\forall t<t_{0}$, 
  $\left\|d_{t}\right\| \leqslant \bar{H}$.

  When $ t=t_{0}$, $\left\|d_{t_{0}-1}\right\| \leqslant \bar{H}$.
  From the update rule of $d_{t_{0}}$ and triangular inequality, 
  $\left\|d_{t_{0}}\right\|$ is bounded as follows:
  \begin{equation*}
      \begin{aligned}
          \left\|d_{t_{0}}\right\| 
          & \leqslant\left\|g_{t_{0}}\right\|+\frac{\left|\gamma_{t_{0}}\right|}{t_{0}^{a}}\left\|d_{t_{0}-1}\right\| \\
          & \leqslant\left\|g_{t_{0}}\right\|+\frac{1}{2}\left\|d_{t_{0}-1}\right\| \\
          & \leqslant \bar{H}.
      \end{aligned}
  \end{equation*}

  Suppose $\exists j > t_{0}$, $\left\|d_{j-1}\right\| \leqslant \bar{H}$.
  $\left\|d_{j}\right\|$ is bounded as follows:
  \begin{equation*}
    \begin{aligned}
      \left\|d_{j}\right\| 
      & \leqslant\left\|g_{j}\right\|+\frac{\left|\gamma_{j}\right|}{j^{a}}\left\|d_{j-1}\right\| \\
      & \leqslant\left\|g_{j}\right\|+\frac{1}{2}\left\|d_{j-1}\right\| \\
          & \leqslant \bar{H}.
    \end{aligned}
  \end{equation*}
  
  By the mathematical induction, it completes the proof.
\end{proof}
\begin{myLem}\label{lem11}
  Suppose $\forall t \in \mathcal{T}$, $a_{t} \geqslant 0$ and $\beta_{11} \in [0,1)$.
  Let $b_{t}=\sum_{i=1}^{t} \beta_{11}^{t-i} \sum_{l=i+1}^{t} a_{l}$.
  Then the following inequality holds:
  \begin{equation}\label{eqA44}
      \begin{aligned}
          \sum_{t=1}^{T-1} b_{t}^{2} \leqslant \frac{1}{\left(1-\beta_{11}\right)^{4}} \sum_{t=2}^{T-1} a_{t}^{2}.
      \end{aligned}
  \end{equation}
\end{myLem}

\begin{proof}
  \begin{equation*}
      \begin{aligned}
          \sum_{t=1}^{T-1} b_{t}^{2}
          & =\sum_{t=1}^{T-1}\left(\sum_{i=1}^{t} \beta_{11}^{t-i} \sum_{l=i+1}^{t} a_{l}\right)^{2} =\sum_{t=1}^{T-1}\left(\sum_{l=2}^{t} \sum_{i=1}^{l} \beta_{11}^{t-i} a_{l}\right)^{2} \\
          & =\sum_{t=1}^{T-1}\left(\sum_{l=2}^{t} \sum_{i=1}^{l} \beta_{11}^{t-l} \beta_{11}^{l-i} a_{l}\right)^{2}
          =\sum_{t=1}^{T-1}\left(\sum_{l=2}^{t} \beta_{11}^{t-l} a_{l} \sum_{i=1}^{l} \beta_{11}^{l-i}\right)^{2} \\
          & \leqslant \frac{1}{\left(1-\beta_{11}\right)^{2}} \sum_{t=1}^{T-1}\left(\sum_{l=2}^{t} \beta_{11}^{t-l} a_{l}\right)^{2} \\
          & =\frac{1}{\left(1-\beta_{11}\right)^{2}} \sum_{t=1}^{T-1}\left(\sum_{l=2}^{t} \beta_{11}^{t-l} a_{l}\right)\left(\sum_{m=2}^{t} \beta_{11}^{t-m} a_{m}\right) \\
          & =\frac{1}{\left(1-\beta_{11}\right)^{2}} \sum_{t=1}^{T-1}\left(\sum_{l=2}^{t} \sum_{m=2}^{t} \beta_{11}^{t-l} a_{l} \beta_{11}^{t-m} a_{m}\right) \\
          & \leqslant \frac{1}{\left(1-\beta_{11}\right)^{2}} \sum_{t=1}^{T-1}\left[\sum_{l=2}^{t} \sum_{m=2}^{t} \beta_{11}^{t-l} \beta_{11}^{t-m} \cdot \frac{1}{2}\left(a_{l}^{2}+a_{m}^{2}\right)\right] \\
          & =\frac{1}{\left(1-\beta_{11}\right)^{2}} \sum_{t=1}^{T-1}\left(\sum_{l=2}^{t} \sum_{m=2}^{t} \beta_{11}^{t-l} \beta_{11}^{t-m} a_{l}^{2}\right) \\
          & \leqslant \frac{1}{\left(1-\beta_{11}\right)^{3}} \sum_{t=1}^{T-1} \sum_{l=2}^{t} \beta_{11}^{t-l} a_{l}^{2}
          =\frac{1}{\left(1-\beta_{11}\right)^{3}} \sum_{l=2}^{T-1} \sum_{t=l}^{T-1} \beta_{11}^{t-l} a_{l}^{2} \\
          & \leqslant \frac{1}{\left(1-\beta_{11}\right)^{4}} \sum_{l=2}^{T-1} a_{l}^{2}.
      \end{aligned}
  \end{equation*}
  The first, the third and the last inequality sign are both due to 
  $\sum_{k=0}^{K} \beta_{11}^{k} \leqslant \frac{1}{1-\beta_{11}}$,
  and the second inequality sign is due to $ab \leqslant \frac{1}{2} (a^{2}+b^{2})$.
  The last equal sign is because of the symmetry of $t$ and  $l$ in the summation.

  The proof is over.
\end{proof}
\begin{myLem}\label{lem1}
  Let $x_{0} \triangleq x_{1}$, $\beta_{11} \neq \frac{1}{2}$ in the CG-like-Adam(Alg.\ref{alg:CG-like-Adam}).
  Consider the sequence
  $z_{t}=x_{t}+\eta_{t}\left(x_{t}-x_{t-1}\right), \forall t \in \mathcal{T}$,
  then the following holds:
  \begin{equation}
      \begin{aligned}\label{eqA1}
          z_{t+1}-z_{t}=
          &-\left(\eta_{t+1}-\eta_{t}\right) \alpha_{t} \hat{V}_{t}^{-\frac{1}{2}} \hat{m}_{t}\\
          & -\eta_{t}\left(\alpha_{t} \hat{V}_{t}^{-\frac{1}{2}}-\alpha_{t-1} \hat{V}_{t-1}^{-\frac{1}{2}}\right) \hat{m}_{t-1} \\
          &-\frac{\alpha_{t}\left(1-\beta_{1t}\right)}{\xi_{t}} \hat{V}_{t}^{-\frac{1}{2}} d_{t}    
      \end{aligned}
  \end{equation}
  and 
  \begin{equation*}
      \begin{aligned}
          z_{2}-z_{1}=
          -\left(\eta_{2}-\eta_{1}\right) \alpha_{1} \hat{V}_{1}^{-\frac{1}{2}} \hat{m}_{1}
          -\alpha_{1} \hat{V}_{1}^{-\frac{1}{2}} d_{1} ,
      \end{aligned}
  \end{equation*}
  where 
  $\eta_{t}=\frac{\beta_{1 t}\left(1-\beta_{11}^{t-1}\right)}{\xi_{t}}$,
  $\xi_{t}=\left(1-\beta_{11}^{t}\right)\left[1-\frac{\beta_{1t}\left(1-\beta_{11}^{t-1}\right)}{1-\beta_{11}^{t}}\right]$.
\end{myLem}

\begin{proof}
  From the update rule of the algorithm, for all $t>1$, there is
  \begin{equation}
      \begin{aligned}\label{eqA2}
          x_{t+1}-x_{t} 
          & =-\alpha_{t} \hat{V}_{t}^{-\frac{1}{2}} \hat{m}_{t} 
          =-\frac{\alpha_{t}}{1-\beta_{11}^{t}} \hat{V}_{t}^{-\frac{1}{2}} m_{t} \\
          & =-\frac{\alpha_{t}}{1-\beta_{11}^{t}} \hat{V}_{t}^{-\frac{1}{2}}\left[\beta_{1 t} m_{t-1}+\left(1-\beta_{1 t}\right) d_{t}\right] \\
%            & =-\frac{\alpha_{t}}{1-\beta_{11}^{t}} \hat{V}_{t}^{-\frac{1}{2}}\left[\beta_{1 t} \frac{-\hat{V}_{t-1}^{\frac{1}{2}}\left(x_{t}-x_{t-1}\right)\left(1-\beta_{11}^{t-1}\right)}{\alpha_{t-1}}+\left(1-\beta_{1t}\right) d_{t}\right] \\
          & =\frac{\beta_{1 t}\left(1-\beta_{11}^{t-1}\right)}{1-\beta_{11}^{t}}\left(x_{t}-x_{t-1}\right)-\frac{\alpha_{t}\left(1-\beta_{1t}\right)}{1-\beta_{11}^{t}} \hat{V}_{t}^{-\frac{1}{2}} d_{t}\\            
          & \quad\, +\frac{\beta_{1t}\left(1-\beta_{11}^{t-1}\right)}{1-\beta_{11}^{t}}\left(\frac{\alpha_{t} \hat {V}_{t}^{-\frac{1}{2}} \hat {V}_{t-1}^{\frac{1}{2}}}{\alpha_{t-1}} -I\right)\left(x_{t}-x_{t-1}\right) \\
          & =\frac{\beta_{1 t}\left(1-\beta_{11}^{t-1}\right)}{1-\beta_{11}^{t}}\left(x_{t}-x_{t-1}\right)-\frac{\alpha_{t}\left(1-\beta_{1t}\right)}{1-\beta_{11}^{t}} \hat{V}_{t}^{-\frac{1}{2}} d_{t} \\
          & \quad\, -\frac{\beta_{1t}\alpha_{t-1}}{1-\beta_{11}^{t}} \left(\frac{\alpha_{t} \hat {V}_{t}^{-\frac{1}{2}} \hat {V}_{t-1}^{\frac{1}{2}}}{\alpha_{t-1}} -I\right) \hat{V}_{t-1}^{-\frac{1}{2}} m_{t-1} \\
          & = \frac{\beta_{1 t}\left(1-\beta_{11}^{t-1}\right)}{1-\beta_{11}^{t}}\left(x_{t}-x_{t-1}\right)-\frac{\alpha_{t}\left(1-\beta_{1t}\right)}{1-\beta_{11}^{t}} \hat{V}_{t}^{-\frac{1}{2}} d_{t}\\ 
          & \quad\, -\frac{\beta_{1t}}{1-\beta_{11}^{t}}\left(\alpha_{t} \hat {V}_{t}^{-\frac{1}{2}}-\alpha_{t-1} \hat{V}_{t-1}^{-\frac{1}{2}}\right) m_{t-1}. \\
      \end{aligned}
  \end{equation}
  
  Since $\left[1-\frac{\beta_{1t}\left(1-\beta_{11}^{t+1}\right)}{1-\beta_{11}^{t}}\right] x_{t+1}
  +\frac{\beta_{1t}\left(1-\beta_{11}^{t+1}\right)}{1-\beta_{11}^{t}}\left(x_{t+1}-x_{t}\right)
  =\left[1-\frac{\beta_{1t}\left(1-\beta_{11}^{t+1}\right)}{1-\beta_{11}^{t}}\right] x_{t}+\left(x_{t+1}-x_{t}\right)$,
  combining with Eq.\eqref{eqA2} we have 
  \begin{equation}\label{eqA3}
      \begin{aligned}
         & \left[1-\frac{\beta_{1 t}\left(1-\beta_{11}^{t-1}\right)}{1-\beta_{11}^{t}}\right] x_{t+1}+\frac{\beta_{1t}\left(1-\beta_{11}^{t-1}\right)}{1-\beta_{11}^{t}}\left(x_{t+1}-x_{t}\right)\\
         & =\left[1-\frac{\beta_{1 t}\left(1-\beta_{11}^{t-1}\right)}{1-\beta_{11}^{t}}\right] x_{t}+\frac{\beta_{1 t}\left(1-\beta_{11}^{t-1}\right)}{1-\beta_{11}^{t}}\left(x_{t}-x_{t-1}\right) \\
         & \quad\, -\frac{\beta_{1 t}}{1-\beta_{11}^{t}}\left(\alpha_{t} \hat {V}_{t}^{-\frac{1}{2}}-\alpha_{t-1} \hat{V}_{t-1}^{-\frac{1}{2}}\right) m_{t-1}-\frac{\alpha_{t}\left(1-\beta_{1 t}\right)}{1-\beta_{11}^{t}} \hat{V}_{t}^{-\frac{1}{2}} d_{t}.
      \end{aligned}
  \end{equation}
  
  From the lemma \ref{lem8},
  $\forall t \in \mathcal{T}, \ \beta_{1 t}<\frac{1-\beta_{11}^{t}}{1-\beta_{11}^{t-1}}$.
  So as long as $\beta_{11} \neq \frac{1}{2}$, there is 
  $\forall t \in \mathcal{T},\ 1-\frac{\beta_{1t}\left(1-\beta_{11}^{t-1}\right)}{1-\beta_{11}^{t}} \neq 0$.
  Divide both sides of Eq.\eqref{eqA3} by $\left[1-\frac{\beta_{1t}(1-\beta_{11}^{t-1})}{1-\beta_{11}^{t}}\right]$, 
  the following holds:
  \begin{equation}\label{eqA4}
      \begin{aligned}
          & x_{t+1}+\eta_{t}\left(x_{t+1}-x_{t}\right) \\
          & =x_{t}+\eta_{t}\left(x_{t}-x_{t-1}\right)
          -\frac{\beta_{1t}}{\xi_{t}}\left(\alpha_{t} \hat{V}_{t}^{-\frac{1}{2}}-\alpha_{t-1} \hat{V}_{t-1}^{-\frac{1}{2}}\right) m_{t-1}
          -\frac{\alpha_{t}\left(1-\beta_{1 t}\right)}{\xi_{t}} \hat{V}_{t}^{-\frac{1}{2}} d_{t}.
      \end{aligned}
  \end{equation}

  Defining sequence $z_{t}=x_{t}+\eta_{t}(x_{t}-x_{t-1})$,
  $\eta_{t}=\frac{\beta_{1 t}\left(1-\beta_{11}^{t-1}\right)}{\xi_{t}}$ in which
  \begin{dmath*}
    \xi_{t}=\left(1-\beta_{11}^{t}\right)\left[1-\frac{\beta_{1 t}\left(1-\beta_{11}^{t-1}\right)}{1-\beta_{11}^{t}}\right] .
  \end{dmath*}
  % where $\xi_{t}=\left(1-\beta_{11}^{t}\right)\left[1-\frac{\beta_{1 t}\left(1-\beta_{11}^{t-1}\right)}{1-\beta_{11}^{t}}\right]$.
  Then the above Eq.\eqref{eqA4} can be converted into 
  \begin{equation}\label{eqA5}
      \begin{aligned}
          & x_{t+1}+\eta_{t}\left(x_{t+1}-x_{t}\right)+\eta_{t+1}\left(x_{t+1}-x_{t}\right)-\eta_{t+1}\left(x_{t+1}-x_{t}\right) \\
          & =z_{t+1}+\left(\eta_{t}-\eta_{t+1}\right)\left(x_{t+1}-x_{t}\right) \\ 
          & =z_{t}-\frac{\eta_{t}}{1-\beta_{11}^{t-1}}\left(\alpha_{t} \hat{V}_{t}^{-\frac{1}{2}}-\alpha_{t-1} \hat{V}_{t-1}^{\frac{1}{2}}\right) m_{t-1}\\
          & \quad\, -\frac{\alpha_{t}\left(1-\beta_{1t}\right)}{\xi_{t}} \hat{V}_{t}^{-\frac{1}{2}} d_{t}.
      \end{aligned}
  \end{equation}
  The Eq.\eqref{eqA5} can be written as 
  \begin{equation*}
      \begin{aligned}
          z_{t+1}-z_{t}=
%            & =\left(\eta_{t+1}-\eta_{t}\right)\left(x_{t+1}-x_{t}\right) -\frac{\eta_{t}}{1-\beta_{11}^{t-1}}\left(\alpha_{t} \hat{V}_{t}^{-\frac{1}{2}}-\alpha_{t-1} \hat{V}_{t-1}^{-\frac{1}{2}}\right) m_{t-1}  -\frac{\alpha_{t}\left(1-\beta_{1t}\right)}{\xi_{t}} \hat{V}_{t}^{-\frac{1}{2}} d_{t} \\
          & -\left(\eta_{t+1}-\eta_{t}\right) \alpha_{t} \hat{V}_{t}^{-\frac{1}{2}} \hat{m}_{t}-\frac{\alpha_{t}\left(1-\beta_{1t}\right)}{\xi_{t}} \hat{V}_{t}^{-\frac{1}{2}} d_{t}  \\ 
          & -\eta_{t}\left(\alpha_{t} \hat{V}_{t}^{-\frac{1}{2}}-\alpha_{t-1} \hat{V}_{t-1}^{-\frac{1}{2}}\right) \hat{m}_{t-1}.
      \end{aligned}
  \end{equation*}

  When $t=1$, $z_{1}=x_{1}+\eta_{1}\left(x_{1}-x_{0}\right)=x_{1}$, there is
  \begin{equation*}
      \begin{aligned}
          z_{2}-z_{1}
          &=x_{2}+\eta_{2}\left(x_{2}-x_{1}\right)-x_{1} \\
%            & =x_{2}+\left(\eta_{2}-\eta_{1}\right)\left(x_{2}-x_{1}\right)+\eta_{1}\left(x_{2}-x_{1}\right)-x_{1} \\
          & =\left(\eta_{2}-\eta_{1}\right)\left(x_{2}-x_{1}\right)+\left(1+\eta_{1}\right)\left(x_{2}-x_{1}\right) \\
%            & =-\left(\eta_{2}-\eta_{1}\right) \alpha_{1} \hat{V}_{1}^{-\frac{1}{2}} \hat{m}_{1}-\left(1+\eta_{1}\right) \frac{\alpha_{1}}{1-\beta_{11}} \hat{V}_{1}^{-\frac{1}{2}}\left[\left(1-\beta_{11}\right) d_{1}\right] \\
          & =-\left(\eta_{2}-\eta_{1}\right) \alpha_{1} \hat{V}_{1}^{-\frac{1}{2}} \hat{m}_{1} 
            -\alpha_{1} \hat{V}_{1}^{-\frac{1}{2}} d_{1}.
      \end{aligned}
  \end{equation*}

  The proof is over.
\end{proof}
\begin{myLem}\label{lem2}
  Suppose the conditions in theorem \ref{th3.1} are satisfied, then 
  \begin{equation}\label{eqA6}
      \begin{aligned}
          \mathbb{E}\left[f\left(z_{T+1}\right)-f\left(z_{1}\right)\right] \leqslant \sum_{i=1}^{6} T_{i} ,
      \end{aligned}
  \end{equation}    
  where 
  \begin{equation}\label{eqA7}
      \begin{aligned}
          T_{1}=-\mathbb{E}\left[\sum_{t=1}^{T}\left\langle\nabla f\left(z_{t}\right), \eta_{t}\left(\alpha_{t} \hat{V}_{t}^{-\frac{1}{2}}-\alpha_{t-1} \hat{V}_{t-1}^{-\frac{1}{2}}\right) \hat{m}_{t-1}\right\rangle\right],
      \end{aligned}
  \end{equation}  

  \begin{equation}\label{eqA8}
      \begin{aligned}
          T_{2}=-\mathbb{E}\left[\sum_{t=1}^{T}\left\langle\nabla f\left(z_{t}\right), \frac{\alpha_{t}\left(1-\beta_{1t}\right)}{\xi_{t}} \hat{V}_{t}^{-\frac{1}{2}} d_{t}\right\rangle\right],        
      \end{aligned}
  \end{equation}  

  \begin{equation}\label{eqA9}
      \begin{aligned}
          T_{3}=-\mathbb{E}\left[\sum_{t=1}^{T}\left\langle\nabla f\left(z_{t}\right),\left(\eta_{t+1}-\eta_{t}\right) \alpha_{t} \hat{V}_{t}^{-\frac{1}{2}} \hat{m}_{t}\right\rangle\right],        
      \end{aligned}
  \end{equation}  

  \begin{equation}\label{eqA10}
      \begin{aligned}
          T_{4}=\frac{3 L}{2} \mathbb{E}\left[\sum_{t=1}^{T}\left\|\left(\eta_{t+1}-\eta_{t}\right) \alpha_{t} \hat{V}_{t}^{-\frac{1}{2}} \hat{m}_{t}\right\|^{2}\right],        
      \end{aligned}
  \end{equation}  

  \begin{equation}\label{eqA11}
      \begin{aligned}
          T_{5}=\frac{3 L}{2} \mathbb{E}\left[\sum_{t=1}^{T}\left\|\eta_{t}\left(\alpha_{t} \hat{V}_{t}^{-\frac{1}{2}}-\alpha_{t-1} \hat{V}_{t-1}^{-\frac{1}{2}}\right) \hat{m}_{t}\right\|^{2}\right],        
      \end{aligned}
  \end{equation}  

  \begin{equation}\label{eqA12}
      \begin{aligned}
          T_{6}=\frac{3 L}{2} \mathbb{E}\left[\sum_{t=1}^{T}\left\|\frac{\alpha_{t}\left(1-\beta_{1 t}\right)}{\xi_{t}} \hat{V}_{t}^{-\frac{1}{2}} d_{t}\right\|^{2}\right].    
      \end{aligned}
  \end{equation}  
\end{myLem}

\begin{proof}
  Since $\nabla f$ is Lipschitz smooth, then 
  \begin{equation}\label{eqA13}
      f\left(z_{t+1}\right) \leqslant f\left(z_{t}\right)+\left\langle\nabla f\left(z_{t}\right), r_{t}\right\rangle+\frac{L}{2}\left\|r_{t}\right\|^{2} ,
  \end{equation}
  where $r_{t}=z_{t+1}-z_{t}$. By the lemma \ref{lem1},
  \begin{equation}\label{eqA14}
      \begin{aligned}
      r_{t}&=z_{t+1}-z_{t}\\
      & =-\left(\eta_{t+1}-\eta_{t}\right) \alpha_{t} \hat{V}_{t}^{-\frac{1}{2}} \hat{m}_{t}
      -\frac{\alpha_{t}\left(1-\beta_{1t}\right)}{\xi_{t}} \hat{V}_{t}^{-\frac{1}{2}} d_{t}  \\
      & \quad\, -\eta_{t}\left(\alpha_{t} \hat{V}_{t}^{-\frac{1}{2}}-\alpha_{t-1} \hat{V}_{t-1}^{-\frac{1}{2}}\right) \hat{m}_{t-1}.
      \end{aligned}
  \end{equation}
  Combining Eq.\eqref{eqA13} and Eq.\eqref{eqA14} gets 
  \begin{equation}\label{eqA15}
      \begin{aligned}
         & \mathbb{E}\left[f\left(z_{T+1}\right)-f\left(z_{1}\right)\right] \\
%           &
%           =\mathbb{E}\left[\sum_{t=1}^{T}\left(f\left(z_{t+1}\right)-f\left(z_{t}\right)\right)\right]\\  
         & \leqslant \mathbb{E}\left[\sum_{t=1}^{T}\left\langle\nabla f\left(z_{t}\right), r_{t}\right\rangle\right]+\frac{L}{2} \mathbb{E}\left[\sum_{t=1}^{T} \| r_{t}||^{2}\right] \\
         & =-\mathbb{E}\left[\sum_{t=1}^{T}\left\langle\nabla f\left(z_{t}\right), \eta_{t}\left(\alpha_{t} \hat{V}_{t}^{-\frac{1}{2}}-\alpha_{t-1} \hat{V}_{t-1}^{-\frac{1}{2}}\right) \hat{m}_{t-1}\right\rangle\right] \\
         &\quad\, -\mathbb{E}\left[\sum_{t=1}^{T}\left\langle\nabla f\left(z_{t}\right), \frac{\alpha_{t}\left(1-\beta_{1t}\right)}{\xi_{t}} \hat{V}_{t}^{-\frac{1}{2}} d_{t}\right\rangle\right] \\
         & \quad\,-\mathbb{E}\left[\sum_{t=1}^{T}\left\langle\nabla f\left(z_{t}\right),\left(\eta_{t+1}-\eta_{t}\right) \alpha_{t} \hat{V}_{t}^{-\frac{1}{2}} \hat{m}_{t}\right\rangle\right] \\ 
         &\quad\, +\frac{L}{2} \mathbb{E}\left[\sum_{t=1}^{T}\left\|r_{t}\right\|^{2}\right] \\
         & =T_{1}+T_{2}+T_{3}+\frac{L}{2} \mathbb{E}\left[\sum_{i=1}^{T}\left\|r_{t}\right\|^{2}\right]. \\
      \end{aligned}
  \end{equation}
  
  Further more, by using the inequality $\|a+b+c\|^{2} \leq 3\|a\|^{2}+3\|b\|^{2}+3\|c\|^{2}$, 
  the last term of RHS of Eq.\eqref{eqA15} can be bounded as follows:
  \begin{equation}\label{eqA16}
      \begin{aligned}
          \mathbb{E}\left[\sum_{t=1}^{T}\left\|r_{t}\right\|^{2}\right] 
          & \leq 3 \mathbb{E}\left[\sum_{t=1}^{T}\left\|\left(\eta_{t+1}-\eta_{t}\right) \alpha_{t} \hat{V}_{t}^{-\frac{1}{2}} \hat{m}_{t}\right\|^{2}\right] \\
          & +3 \mathbb{E}\left[\sum_{t=1}^{T}\left\|\eta_{t}\left(\alpha_{t} \hat{V}_{t}^{-\frac{1}{2}}-\alpha_{t-1} \hat{V}_{t-1}^{-\frac{1}{2}}\right) \hat{m}_{t-1}\right\|^{2}\right] \\
          & +3 \mathbb{E}\left[\sum_{t=1}^{T}\left\|\frac{\alpha_{t}\left(1-\beta_{1t}\right)}{\xi_{t}} \hat{V}_{t}^{-\frac{1}{2}} d_{t}\right\|^{2}\right].
      \end{aligned}
  \end{equation}
  
  Combining Eq.\eqref{eqA15} and Eq.\eqref{eqA16} leads to Eq.\eqref{eqA6}, which completes the proof.
\end{proof}
The following lemmas \ref{lem3}-\ref{lem7} separately bound the terms Eq.\eqref{eqA7}-Eq.\eqref{eqA11}.
\begin{myLem}\label{lem3}
  Suppose the conditions in theorem \ref{th3.1} are satisfied, then the term $T_{1}$(see Eq.\eqref{eqA7}) 
  in the lemma \ref{lem2} satisfies the following inequality:
  \begin{equation}\label{eqA17}
      \begin{aligned}
          T_{1}
          & =-\mathbb{E}\left[\sum_{t=1}^{T}\left\langle\nabla f\left(z_{t}\right), \eta_{t}\left(\alpha_{t} \hat{V}_{t}^{-\frac{1}{2}}-\alpha_{t-1} \hat{V}_{t-1}^{-\frac{1}{2}}\right) \hat{m}_{t-1}\right\rangle\right] \\
          & \leqslant \frac{\alpha_{1} H \bar{M}}{1-\beta_{11}} E\left[\sum_{i=1}^{d} \hat{v}_{1, i}^{-\frac{1}{2}}\right] ,
      \end{aligned}
  \end{equation}
  where $\bar{M}$ is a constant independent of $T$.
\end{myLem}

\begin{proof}
  By the lemma \ref{lem10}, $\forall t \in \mathcal{T}$, $\left\|d_{t}\right\| \leqslant \bar{H}$.
  Letting ${M \triangleq \max \left\{H, \bar{H}\right\}=\bar{H}}$ and supposing $\left\|m_{t-1}\right\| \leqslant M$,
  then by the update rule of the algorithm,

  \begin{equation*}
    \begin{gathered}
        \left\|m_{t}\right\| = \left\|\beta_{1 t} m_{t-1}+\left(1-\beta_{1 t}\right) d_{t}\right\| , \\
        \max \left\{\left\|d_{t}\right\|,\left\|m_{t-1}\right\|\right\} \leqslant M.    
    \end{gathered}
  \end{equation*}

  By $m_{0} = 0$, $\left\|m_{0}\right\|=0 \leqslant M$ and mathematical induction, there is 
  $\forall t \in \mathcal{T}$, $\left\|m_{t}\right\| \leqslant M$. 
  Further more, $\left\|\hat{m}_{t}\right\|$ can be bounded as follows:
  \begin{equation*}
      \left\|\hat{m}_{t}\right\|=\left\|\frac{m_{t}}{1-\beta_{11}^{t}}\right\|=\frac{\left\|m_{t}\right\|}{1-\beta_{11}^{t}} \leqslant \frac{\left\|m_{t}\right\|}{1-\beta_{11}} \leqslant \frac{M}{1-\beta_{11}} \triangleq  \bar{M}.
  \end{equation*}

  Using Cauchy-Schwarz inequality and assumption 
  $\left(\frac{\alpha_{1}}{\sqrt{\hat{V}_{1}}}-\frac{\alpha_{0}}{\sqrt{\hat{V}_{0}}}\right) \hat{m}_{0}=0$,
  we further have
  \begin{equation}\label{eqA18}
      \begin{aligned}
          T_{1}
          &=-\mathbb{E}\left[\sum_{t=2}^{T}\left\langle\nabla f\left(x_{t}\right), \eta_{t}\left(\alpha_{t} \hat{V}_{t}^{-\frac{1}{2}}-\alpha_{t-1} \hat{V}_{t-1}^{-\frac{1}{2}}\right) \hat{m}_{t-1}\right\rangle\right] \\
          & \leqslant \mathbb{E}\left[\sum_{t=2}^{T}\left\|\nabla f\left(x_{t}\right)\right\| \cdot\left\|\eta_{t}\left(\alpha_{t} \hat{V}_{t}^{-\frac{1}{2}}-\alpha_{t-1} \hat{V}_{t-1}^{-\frac{1}{2}}\right) \hat{m}_{t-1}\right\|\right] \\
          & \leqslant \frac{1}{1-\beta_{11}} \mathbb{E}\left[\sum_{t=2}^{T}\left\|\nabla f\left(x_{t}\right)\right\| \cdot\left\|\left(\alpha_{t} \hat{V}_{t}^{-\frac{1}{2}}-\alpha_{t-1} \hat{V}_{t-1}^{-\frac{1}{2}}\right) \hat{m}_{t-1}\right\|\right] \\
%            & =\frac{1}{1-\beta_{11}} \mathbb{E}\left[\sum_{t=2}^{T}\left\|\nabla f\left(x_{t}\right)\right\| \cdot \sqrt{\sum_{i=1}^{d}\left(\alpha_{t-1} \hat{V}_{t-1, i}^{-\frac{1}{2}}-\alpha_{t} \hat{V}_{t, i}^{-\frac{1}{2}}\right)^{2} \hat{m}_{t-1, i}^{2}}\right] \\
%            & \leqslant \frac{1}{1-\beta_{11}} \mathbb{E}\left[\sum_{t=2}^{T}\left\|\nabla f\left(x_{t}\right)\right\| \cdot \sqrt{\sum_{i=1}^{d}\left(\alpha_{t-1} \hat{V}_{t-1, i}^{-\frac{1}{2}}-\alpha_{t} \hat{V}_{t, i}^{-\frac{1}{2}}\right)^{2}} \cdot \sqrt{\sum_{i=1}^{d} \hat{m}_{t-1, i}^{2}}\right] \\
          & \leqslant \frac{1}{1-\beta_{11}} \mathbb{E}\left[\sum_{t=2}^{T}\left\|\nabla f\left(x_{t}\right)\right\| \cdot\left\|\hat{m}_{t-1}\right\| \right.
           \\
          &\hspace{3.95cm} \cdot \left. \sqrt{\sum_{i=1}^{d}\left(\alpha_{t-1} \hat{v}^{-\frac{1}{2}}_{t-1,i}-\alpha_{t} \hat{v}_{t,i}^{-\frac{1}{2}}\right)^{2}}\right] \\
%            & \leqslant \frac{1}{1-\beta_{11}} \mathbb{E}\left[\sum_{t=2}^{T}\left\|\nabla f\left(x_{t}\right)\right\| \cdot\left\|\hat{m}_{t-1}\right\| \cdot \sqrt{\left(\sum_{i=1}^{d}\left(\alpha_{t-1} \hat{v}_{t-1,i}^{-\frac{1}{2}}-\alpha_{t} \hat{v}_{t, i}^{-\frac{1}{2}}\right)\right)^{2}}\right] \\
          & \leqslant \frac{H \bar{M}}{1-\beta_{11}} \mathbb{E}\left[\sum_{t=2}^{T} \sum_{i=1}^{d}\left(\alpha_{t-1} \hat{v}_{t-1, i}^{-\frac{1}{2}}-\alpha_{t} \hat{v}_{t, i}^{-\frac{1}{2}}\right)\right] \\
          & =\frac{H \bar{M}}{1-\beta_{11}} \mathbb{E}\left[\sum_{i=1}^{d}\left(\alpha_{1} \hat{v}_{1, i}^{-\frac{1}{2}}-\alpha_{T} \hat{v}_{T, i}^{-\frac{1}{2}}\right)\right]
          \leqslant \frac{\alpha_{1} H \bar{M}}{1-\beta_{11}} \mathbb{E}\left[\sum_{i=1}^{d} \hat{v}_{1, i}^{-\frac{1}{2}}\right].
      \end{aligned}
  \end{equation}
  The second inequality sign is because of the lemma \ref{lem8}. 
  The fourth inequality sign is due to the fact that $\forall t \in \mathcal{T}$, 
  $\forall i \in [d]$, 
  ${\alpha_{t-1} \hat{v}_{t-1, i}^{-\frac{1}{2}}-\alpha_{t} \hat{v}_{t, i}^{-\frac{1}{2}} \geq 0}$
  and the fact that when $a \geqslant 0$ and $b \geqslant 0$, $\left(a^{2}+b^{2}\right) \leqslant(a+b)^{2}$.

  The proof is over.
\end{proof}
\begin{myLem}\label{lem4}
  Suppose the conditions in theorem \ref{th3.1} are satisfied, then the term $T_{3}$(see Eq.\eqref{eqA9}) 
  in the lemma \ref{lem2} satisfies the following inequality:
  \begin{align}
    T_{3}
    =-\mathbb{E}\left[\sum_{t=1}^{T}\left\langle\nabla f\left(z_{t}\right),\left(\eta_{t+1}-\eta_{t}\right) \alpha_{t} \hat{V}_{t}^{-\frac{1}{2}} \hat{m}_{t}\right\rangle\right]
    \leq \frac{1}{2}\left(H^{2}+G^{2}\right)\left|\eta_{T}-\eta_{1}\right| \label{eqA19} .
  \end{align}
\end{myLem}
\begin{proof}
  By the triangular inequality, there is 
  \begin{align}
          T_{3} 
          & \leqslant \mathbb{E}\left[\sum_{t=1}^{T}\left|\eta_{t+1}-\eta_{t}\right| \cdot\frac{1}{2}\left(\left\|\nabla f\left(z_{t}\right)\right\|^{2}+\left\|\alpha_{t} \hat{V}_{t}^{-\frac{1}{2}} \hat{m}_{t}\right\|^{2}\right)\right] \notag \\
          & \leqslant \frac{1}{2} \mathbb{E}\left[\sum_{t=1}^{T}\left|\eta_{t+1}-\eta_{t}\right|\left(H^{2}+G^{2}\right)\right] \label{eqA20} \\
          & =\frac{1}{2}\left(H^{2}+G^{2}\right) \mathbb{E}\left[\sum_{t=1}^{T}\left|\eta_{t+1}-\eta_{t}\right|\right]. \notag
  \end{align}
  From the lemma \ref{lem9}, Eq.\eqref{eqA20} can be further bounded as follows:
  \begin{equation}\label{eqA21}
  T_{3} \leqslant\left\{\begin{array}{ll}
      \frac{1}{2}\left(H^{2}+G^{2}\right)\left(\eta_{T}-\eta_{1}\right), & \forall t \in \mathcal{T}, \beta_{1(t+1)} \leqslant h(t) \beta_{1t} , \\
      \frac{1}{2}\left(H^{2}+G^{2}\right)\left(\eta_{1}-\eta_{T}\right), & \forall t \in \mathcal{T}, \beta_{1(t+1)} \geqslant h(t) \beta_{1t} .
      \end{array}\right.
  \end{equation}
  Eq.\eqref{eqA21} is equivalent to $T_{3} \leqslant \frac{1}{2}\left(H^{2}+G^{2}\right)\left|\eta_{T}-\eta_{1}\right|$.

  The proof is over.
\end{proof}

What needs to remind readers is that, if $\exists T_{0}\in \mathcal{T}$, $\forall t \in [T_{0}]$,
$\beta_{1(t+1)} \leqslant (\geqslant) h(t) \beta_{1t}$, and $\forall t \in \mathcal{T} \setminus [T_{0}]$, 
$\beta_{1(t+1)} \geqslant (\leqslant) h(t) \beta_{1t}$, then it is obvious that  
\begin{equation*}
\begin{aligned}
  T_{3} 
  \leqslant \frac{1}{2}\left(H^{2}+G^{2}\right)\left|2 \eta_{T_{0}}-\eta_{T}-\eta_{1}\right|
  =\frac{1}{2}\left(H^{2}+G^{2}\right)\left|-2 \eta_{T_{0}}+\eta_{1}+\eta_{T}\right|.
\end{aligned}
\end{equation*}

\begin{myLem}\label{lem5}
  Suppose the conditions in theorem \ref{th3.1} are satisfied, then the term $T_{4}$(see Eq.\eqref{eqA10}) 
  in the lemma \ref{lem2} satisfies the following inequality:
  \begin{equation}\label{eqA22}
    \begin{aligned}
      T_{4}
      =\frac{3 L}{2} \mathbb{E}\left[\sum_{t=1}^{T}\left\|\left(\eta_{t+1}-\eta_{t}\right) \alpha_{t} \hat{V}_{t}^{-\frac{1}{2}} \hat{m}_{t}\right\|^{2}\right]
      \leqslant \frac{3 L G^{2}}{1-\beta_{11}}\left|\eta_{T}-\eta_{1}\right|.
    \end{aligned}
  \end{equation}
\end{myLem}

\begin{proof}
  \begin{equation}\label{eqA23}
      \begin{aligned}
          T_{4}
%            & =\frac{3 L}{2} \mathbb{E}\left[\sum_{t=1}^{T}\left\|\left(\eta_{t+1}-\eta_{t}\right) \alpha_{t} \hat{V}_{t}^{-\frac{1}{2}} \hat{m}_{t}\right\|^{2}\right] \\
          & =\frac{3 L}{2} \mathbb{E}\left[\sum_{t=1}^{T}\left(\eta_{t+1}-\eta_{t}\right)^{2}\left\|\alpha_{t} \hat{V}_{t}^{-\frac{1}{2}} \hat{m}_{t}\right\|^{2}\right]
          \leqslant \frac{3 L G^{2}}{2} \mathbb{E}\left[\sum_{t=1}^{T}\left(\eta_{t+1}-\eta_{t}\right)^{2}\right] \\
%            & =\frac{3 L G^{2}}{2} \mathbb{E}\left[\sum_{t=1}^{T}\left|\eta_{t+1}-\eta_{t}\right| \cdot\left|\eta_{t+1}-\eta_{t}\right|\right] \\
          & \leqslant \frac{3 L G^{2}}{2}  \cdot\frac{2}{1-\beta_{11}} \mathbb{E}\left[\sum_{t=1}^{T}\left|\eta_{t+1}-\eta_{t}\right|\right]
          \leqslant \frac{3 L G^{2}}{1-\beta_{11}}\left|\eta_{T}-\eta_{1}\right| ,
      \end{aligned}
  \end{equation}
  where the penultimate inequality sign is because of the lemma \ref{lem8} and $\left|\eta_{t+1}-\eta_{t}\right| \leqslant \eta_{t+1}+\eta_{t}$, 
  while the last one is similar to the proof of lemma \ref{lem4}.

  The proof is over.
\end{proof}
\begin{myLem}\label{lem6}
  Suppose the conditions in theorem \ref{th3.1} are satisfied, then the term $T_{5}$(see Eq.\eqref{eqA11}) 
  in the lemma \ref{lem2} satisfies the following inequality:
  \begin{equation}\label{eqA24}
    \begin{aligned}
      T_{5}
      =\frac{3 L}{2} \mathbb{E}\left[\sum_{t=1}^{T}\left\|\eta_{t}\left(\alpha_{t} \hat{V}_{t}^{-\frac{1}{2}}-\alpha_{t-1} \hat{V}_{t-1}^{-\frac{1}{2}}\right) \hat{m}_{t}\right\|^{2}\right] 
      \leqslant \frac{3 \alpha_{1}^{2} L \bar{M}^{2}}{2\left(1-\beta_{11}\right)^{2}} \mathbb{E}\left[\sum_{i=1}^{d} \hat{v}_{1, i}^{-1}\right].
    \end{aligned}
  \end{equation}
\end{myLem}

\begin{proof}
  \begin{align}
    T_{5}
    & \leqslant \frac{3 L}{2\left(1-\beta_{11}\right)^{2}} \mathbb{E}\left[\sum_{t=1}^{T}\left\|\left(\alpha_{t} \hat{V}_{t}^{-\frac{1}{2}}-\alpha_{t-1} \hat{V}_{t-1}^{-\frac{1}{2}}\right) \hat{m}_{t}\right\|^{2}\right] \notag \\
    & \leqslant \frac{3 L}{2\left(1-\beta_{11}\right)^{2}} \mathbb{E}\left[\sum_{t=2}^{T}\left(\sum_{i=1}^{d}\left(\alpha_{t} \hat{v}_{t, i}^{-\frac{1}{2}}-\alpha_{t-1} \hat{v}_{t-1, i}^{-\frac{1}{2}}\right)^{2}\right)\left\|\hat{m}_{t}\right\|^{2}\right] \notag \\
    & \leqslant \frac{3 L \bar{M}^{2}}{2\left(1-\beta_{11}\right)^{2}} \mathbb{E}\left[\sum_{t=2}^{T} \sum_{i=1}^{d}\left(\alpha_{t} \hat{v}_{t, i}^{-\frac{1}{2}}-\alpha_{t-1} \hat{v}_{t-1, i}^{-\frac{1}{2}}\right)^{2}\right] \label{eqA25} \\
%    & \leqslant \frac{3 L \bar{M}^{2}}{2\left(1-\beta_{11}\right)^{2}} \mathbb{E}\left[\sum_{t=2}^{T} \sum_{i=1}^{d}\left(\alpha_{t-1} \hat{v}_{t-1, i}^{-\frac{1}{2}}+\alpha_{t} \hat{v}_{t, i}^{-\frac{1}{2}}\right) \cdot \left(\alpha_{t-1} \hat{v}_{t-1, i}^{\frac{1}{2}}-\alpha_{t} \hat{v}_{t, i}^{-\frac{1}{2}}\right)\right] \\
    & \leqslant \frac{3 L \bar{M}^{2}}{2\left(1-\beta_{11}\right)^{2}} \mathbb{E}\left[\sum_{t=2}^{T} \sum_{i=1}^{d}\left(\alpha_{t-1}^{2} \hat{v}_{t-1, i}^{-1}-\alpha_{t}^{2} \hat{v}_{t, i}^{-1}\right)\right] \notag \\
    & =\frac{3 L \bar{M}^{2}}{2\left(1-\beta_{11}\right)^{2}} \mathbb{E}\left[\sum_{i=1}^{d}\left(\alpha_{1}^{2} \hat{v}_{1, i}^{-1}-\alpha_{T}^{2} \hat{v}_{T, i}^{-1}\right)\right]
    \leqslant \frac{3 \alpha_{1}^{2} L \bar{M}^{2}}{2\left(1-\beta_{11}\right)^{2}} \mathbb{E}\left[\sum_{i=1}^{d} \hat{v}_{1, i}^{-1}\right] , \notag
  \end{align}
  where the first inequality sign is because of the lemma \ref{lem8}.

  The proof is over.
\end{proof}
\begin{myLem}\label{lem7}
  Suppose the conditions in theorem \ref{th3.1} are satisfied, then the term $T_{2}$(see Eq.\eqref{eqA8}) 
  in the lemma \ref{lem2} satisfies the following inequality:
  \begin{align}
    T_{2}
    & =-\mathbb{E}\left[\sum_{t=1}^{T}\left\langle\nabla f\left(z_{t}\right), \frac{\alpha_{t}\left(1-\beta_{1t}\right)}{\xi_{t}} \hat{V}_{t}^{-\frac{1}{2}} d_{t}\right\rangle\right] \notag \\
    & \leqslant\frac{L^{2}}{\left(1-\beta_{11}\right)^{6}} \mathbb{E}\left[\sum_{t=1}^{T-1}\left\|\alpha_{t} \hat{V}_{t}^{-\frac{1}{2}} d_{t}\right\|^{2}\right] + \frac{\alpha_{1}^{2} L^{2} \bar{H}^{2}}{\left(1-\beta_{11}\right)^{8}} \mathbb{E}\left[\sum_{i=1}^{d} \hat{v}_{1, i}^{-1}\right] \notag \\
    & \quad\, +\frac{1}{2} \mathbb{E}\left[\sum_{t=2}^{T}\left\|\frac{\alpha_{t}\left(1-\beta_{1 t}\right)}{\xi_{t}} \hat{V}_{t}^{-\frac{1}{2}} d_{t}\right\|^{2}\right] + 2 \mu_{1} H^{2} \mathbb{E}\left[\sum_{i=1}^{d} \hat{v}^{-\frac{1}{2}}_{1,i}\right] \label{eqA26} \\
    & \quad\, +2 H^{2} \mathbb{E}\left[\sum_{t=2}^{T} \sum_{i=1}^{d}\left|\mu_{t} \hat{v}_{t, i}^{-\frac{1}{2}}-\mu_{t-1} \hat{v}_{t-1, i}^{-\frac{1}{2}}\right|\right] \notag \\
    & \quad\, -E\left[\sum_{t=1}^{T} \frac{\alpha_{t}\left(1-\beta_{1t}\right)}{\xi_{t}}\left\langle\nabla f\left(x_{t}\right), \hat{V}_{t}^{-\frac{1}{2}} c_{t}\right\rangle\right] , \notag 
  \end{align}
  where $\mu_{t}=\frac{\alpha_{t}\left(1-\beta_{1t}\right)}{\xi_{t}}$, 
  $c_{t}=\nabla f\left(x_{t}\right)-\frac{\gamma_{t}}{t^{a}} d_{t-1}$.
\end{myLem}

\begin{proof}
  Let 
  \begin{equation}\label{eqA27}
      T_{21}=-\mathbb{E}\left[\sum_{t=1}^{T}\left\langle\nabla f\left(x_{t}\right), \frac{\alpha_{t}\left(1-\beta_{1t}\right)}{\xi_{t}} \hat{V}_{t}^{-\frac{1}{2}} d_{t}\right\rangle\right],
  \end{equation}
  \begin{equation}\label{eqA28}
      T_{22}=-\mathbb{E}\left[\sum_{t=1}^{T}\left\langle\nabla f\left(z_{t}\right)-\nabla f\left(x_{t}\right), \frac{\alpha_{t}\left(1-\beta_{1t}\right)}{\xi_{t}} \hat{V}_{t}^{-\frac{1}{2}} d_{t}\right\rangle\right].
  \end{equation}
  Then $T_{2}$ can be written as 
  \begin{equation}\label{eqA29}
      T_{2}=T_{21}+T_{22}.
  \end{equation}

  The following is firstly to bound the term $T_{22}$(see Eq.\eqref{eqA28}). Since $z_{1}=x_{1}$, 
  $z_{t}-x_{t}=\eta_{t}\left(x_{t}-x_{t-1}\right)=-\eta_{t} \alpha_{t-1} \hat{V}_{t-1}^{-\frac{1}{2}} \hat{m}_{t-1}$
  and the triangular inequality, we get 
  \begin{equation}\label{eqA30}
      \begin{aligned}
          T_{22} 
          & \leqslant \frac{1}{2} \mathbb{E}\left[\sum_{t=2}^{T}\Bigg(\left\|\nabla f\left(z_{t}\right)-\nabla f\left(x_{t}\right)\right\|^{2}
          +\left\|\frac{\alpha_{t}\left(1-\beta_{1t}\right)}{\xi_{t}} \hat{V}_{t}^{-\frac{1}{2}} d_{t}\right\|^2\Bigg)\right] \\
          & \leqslant \frac{1}{2} \mathbb{E}\left[\sum_{t=2}^{T} L^{2}\left\|z_{t}-x_{t}\right\|^{2}\right]
          +\frac{1}{2} \mathbb{E}\left[\sum_{t=2}^{T}\left\|\frac{\alpha_{t}\left(1-\beta_{1 t}\right)}{\xi_{t}} \hat{V}_{t}^{-\frac{1}{2}} d_{t}\right\|^{2}\right] \\
          & =\frac{L^{2}}{2} \mathbb{E}\left[\sum_{t=2}^{T}\left\|\eta_{t} \alpha_{t-1} \hat{V}_{t-1}^{-\frac{1}{2}} \hat{m}_{t-1}\right\|^{2}\right]
          +\frac{1}{2} \mathbb{E}\left[\sum_{t=2}^{T}\left \|\frac{\alpha_{t}\left(1-\beta_{1t}\right)}{\xi_{t}} \hat{V}_{t}^{-\frac{1}{2}} d_{t}\right \|^{2}\right].
      \end{aligned}
  \end{equation}
  The second inequality sign is due to the assumption \ref{ass31}. We now define 
  \begin{dmath*} T_{7}=\mathbb{E}\left[\sum_{t=2}^{T}\left\|\eta_{t} \alpha_{t-1} \hat{V}_{t-1}^{-\frac{1}{2}} \hat{m}_{t-1}\right\|^{2}\right] \end{dmath*}
  and bound it as follows:
  \begin{align}
    T_{7} 
    & \leqslant \frac{1}{\left(1-\beta_{11}\right)^{2}} \mathbb{E}\left[\sum_{t=2}^{T}\left\|\alpha_{t-1} \hat{V}_{t-1}^{-\frac{1}{2}} \hat{m}_{t-1}\right\|^{2}\right] \notag \\
    & =\frac{1}{\left(1-\beta_{11}\right)^{2}} \mathbb{E}\left[\sum_{t=2}^{T}\Bigg\|\alpha_{t-1} \hat{V}_{t-1}^{-\frac{1}{2}}
    \frac{\left(\beta_{1(t-1)} m_{t-2}+\left(1-\beta_{1(t-1)}\right) d_{t-1}\right)}{1-\beta_{11}^{t-1}}\Bigg\|^{2}\right] \label{eqA31} \\
    & \leqslant \frac{1}{\left(1-\beta_{11}\right)^{4}} \mathbb{E}\left[\sum_{t=2}^{T}\Bigg\|\alpha_{t-1} \hat{V}_{t-1}^{-\frac{1}{2}}
    \left(\sum_{i=1}^{t-1}\left(\prod_{j=i+1}^{t-1} \beta_{1 j}\right)\left(1-\beta_{1 i}\right) d_{i}\right)\Bigg\|^{2}\right] \notag \\
    & \leqslant \frac{2}{\left(1-\beta_{11}\right)^{4}}(T_{71}+T_{72}) , \notag
  \end{align}
 where 
 \begin{align}
      &T_{71}=\mathbb{E}\left[\sum_{t=2}^{T}\left\|\sum_{i=1}^{t-1}\left(\prod_{j=i+1}^{t-1} \beta_{1 j}\right)
      \left(1-\beta_{1 i}\right)
      \alpha_{i} \hat{V}_{i}^{-\frac{1}{2}} d_{i}\right\|^{2}\right], \label{eqA32} \\ 
      T_{72}
      &=\mathbb{E}\left[\sum_{t=2}^{T} \left\|\sum_{i=1}^{t-1}\left(\prod_{j=i+1}^{t-1} \beta_{1 j}\right)\left(1-\beta_{1i}\right)
      \left(\alpha_{t-1} \hat{V}_{t-1}^{-\frac{1}{2}}-\alpha_{i} \hat{V}_{i}^{-\frac{1}{2}}\right) d_{i} \right\|^{2}\right]. \label{eqA33}
  \end{align}
  The first inequality sign is because of the lemma \ref{lem8} and the last is due to 
  $a^{2}=(b+a-b)^{2} \leqslant 2b^{2}+2(a-b)^{2}$.
  
  The following is next to bound the term $T_{71}$(see Eq.\eqref{eqA32}).
  \begin{align}
    T_{71}
    & =\mathbb{E}\left[\sum_{t=2}^{T} \sum_{k=1}^{d}\left(\sum_{i=1}^{t-1}\left(\prod_{j=i+1}^{t-1} \beta_{1 j}\right)\left(1-\beta_{1i}\right) \alpha_{i} \hat{v}_{i,k}^{-\frac{1}{2}} d_{i, k}\right) \right. \notag \\
    & \hspace{3cm} \cdot \left.  \left(\sum_{l=1}^{t-1}\left(\prod_{p=l+1}^{t-1} \beta_{1 p}\right)\left(1-\beta_{1l}\right) \alpha_{l} \hat{v}_{l, k}^{-\frac{1}{2}} d_{l, k}\right)\right] \notag \\
    & =\mathbb{E}\left[\sum_{t=2}^{T} \sum_{k=1}^{d} \sum_{i=1}^{t-1} \sum_{l=1}^{t-1}\left(\prod_{j=i+1}^{t-1} \beta_{1 j}\right)\left(1-\beta_{1i}\right)\left(\alpha_{i} \hat{v}_{i, k}^{-\frac{1}{2}} d_{i, k}\right) \right. \notag \\
    & \hspace{3cm} \cdot \left. \left(\prod_{p=l+1}^{t-1} \beta_{1p}\right)\left(1-\beta_{1l}\right)\left(\alpha_{l} \hat{v}_{l, k}^{-\frac{1}{2}} d_{l, k}\right)\right] \notag \\
    & \leqslant \mathbb{E}\left[\sum_{t=2}^{T} \sum_{k=1}^{d} \sum_{i=1}^{t-1} \sum_{l=1}^{t-1}\left(\prod_{j=i+1}^{t-1} \beta_{1 j}\right)\left(1-\beta_{1i}\right)\left(\prod_{p=l+1}^{t-1} \beta_{1 p}\right)  \right. \notag \\ 
    & \hspace{3cm} \cdot \left. \left(1-\beta_{1l}\right) \frac{1}{2}\left(\left(\alpha_{i} \hat{v}_{i, k}^{-\frac{1}{2}} d_{i, k}\right)^{2}+\left(\alpha_{l} \hat{v}_{l, k}^{-\frac{1}{2}} d_{l, k}\right)^{2}\right)\right] \label{eqA34} \\
%   & \leqslant \frac{1}{2} \mathbb{E}\left[\sum_{t=2}^{T} \sum_{k=1}^{d} \sum_{i=1}^{t-1} \sum_{l=1}^{t-1} \beta_{11}^{t-i-1} \beta_{11}^{t-l-1} \cdot    \\
%   & \quad \left(\left(\alpha_{i} \hat{v}_{i, k}^{-\frac{1}{2}} d_{i, k}\right)^{2}+\left(\alpha_{l} \hat{v}_{l, k}^{-\frac{1}{2}} d_{l, k}\right)^{2}\right)\right] \\
    & \leqslant \mathbb{E}\left[\sum_{t=2}^{T} \sum_{k=1}^{d} \sum_{i=1}^{t-1} \beta_{11}^{t-i-1}\left(\alpha_{i} \hat{v}_{i, k}^{-\frac{1}{2}} d_{i, k}\right)^{2} \sum_{l=1}^{t-1} \beta_{11}^{t-l-1}\right] \notag \\
    & \leqslant \frac{1}{1-\beta_{11}} \mathbb{E}\left[\sum_{t=2}^{T} \sum_{k=1}^{d} \sum_{i=1}^{t-1} \beta_{11}^{t-i-1}\left(\alpha_{i} \hat{v}_{i, k}^{-\frac{1}{2}} d_{i, k}\right)^{2}\right] \notag \\
    & =\frac{1}{1-\beta_{11}} \mathbb{E}\left[\sum_{i=1}^{T-1} \sum_{k=1}^{d} \sum_{t=i+1}^{T} \beta_{11}^{t-i-1}\left(\alpha_{i} \hat{v}_{i, k}^{-\frac{1}{2}} d_{i, k}\right)^{2}\right] \notag \\
    & \leqslant \frac{1}{\left(1-\beta_{11}\right)^{2}} \mathbb{E}\left[\sum_{i=1}^{T-1} \sum_{k=1}^{d}\left(\alpha_{i} \hat{v}_{i, k}^{-\frac{1}{2}} d_{i, k}\right)^{2}\right] \notag \\
    & =\frac{1}{\left(1-\beta_{11}\right)^{2}} \mathbb{E}\left[\sum_{t=1}^{T-1}\left\|\alpha_{t} \hat{V}_{t}^{-\frac{1}{2}} d_{t}\right\|^{2}\right]. \notag 
  \end{align}

  The first, second, third inequality sign are because of 
  $ab \leqslant \frac{1}{2}(a^{2}+b^{2})$; 
  $\forall t \in \mathcal{T}$, $\beta_{1t} \leqslant \beta_{11}$;
  $\sum_{l=1}^{t-1} \beta_{11}^{t-l-1} \leqslant \frac{1}{1-\beta_{11}}$, respectively.
  The third, fourth equal sign are due to the symmetry of $i$ and $l$ in the summation;
  exchanging order of summation, respectively.

  The next is to bound the term $T_{72}$(see Eq.\eqref{eqA33}).
  \begin{align}
    &T_{72}
    =\mathbb{E}\left[\sum_{t=2}^{T}\Bigg\|\sum_{i=1}^{t-1}\left(\prod_{j=i+1}^{t-1} \beta_{1 j}\right)\left(1-\beta_{1i}\right)
    \left(\alpha_{t-1} \hat{V}_{t-1}^{-\frac{1}{2}}-\alpha_{i} \hat{V}_{i}^{-\frac{1}{2}}\right) d_{i}\Bigg\|^{2}\right] \notag \\
%    & =\mathbb{E}\left[\sum_{t=2}^{T} \sum_{k=1}^{d}\left(\sum_{i=1}^{t-1}\left(\prod_{j=i+1}^{t-1} \beta_{1j}\right)\left(1-\beta_{1i}\right)\left(\alpha_{t-1} \hat{v}_{t-1, k}^{-\frac{1}{2}}-\alpha_{i} \hat{v}_{i, k}^{-\frac{1}{2}}\right) d_{i, k}\right)^{2}\right] \\
%    & \leqslant \mathbb{E}\left[\sum_{t=2}^{T} \sum_{k=1}^{d}\left(\sum_{i=1}^{t-1}\left(\prod_{j=i+1}^{t-1} \beta_{1j}\right)\left(1-\beta_{1i}\right)\left|\alpha_{t-1} \hat{v}_{t-1, k}^{-\frac{1}{2}}-\alpha_{i} \hat{v}_{i, k}^{-\frac{1}{2}}\right| \cdot \left|d_{i, k}\right|\right)^{2}\right] \\
    & \leqslant \bar{H}^{2} \mathbb{E}\left[\sum_{t=2}^{T} \sum_{k=1}^{d}\Bigg(\sum_{i=1}^{t-1}\left(\prod_{j=i+1}^{t-1} \beta_{1j}\right)\left(1-\beta_{1i}\right)
    \left|\alpha_{t-1} \hat{v}_{t-1, k}^{-\frac{1}{2}}-\alpha_{i} \hat{v}_{i, k}^{-\frac{1}{2}}\right| \Bigg)^{2}\right] \label{eqA35} \\
%    & \leqslant \bar{H}^{2} \mathbb{E}\left[\sum_{t=2}^{T} \sum_{k=1}^{d}\left(\sum_{i=1}^{t-1}  \beta^{t-i-1}_{11} \left|\alpha_{t-1} \hat{v}_{t-1, k}^{-\frac{1}{2}}-\alpha_{i} \hat{v}_{i, k}^{-\frac{1}{2}}\right| \right)^{2}\right] \\
    & \leqslant \bar{H}^{2} \mathbb{E}\left[\sum_{t=1}^{T-1} \sum_{k=1}^{d}\left(\sum_{i=1}^{t}  \beta^{t-i}_{11} \left|\alpha_{t} \hat{v}_{t, k}^{-\frac{1}{2}}-\alpha_{i} \hat{v}_{i, k}^{-\frac{1}{2}}\right| \right)^{2}\right] \notag \\
%    & =\bar{H}^{2} \mathbb{E}\left[\sum_{t=1}^{T-1} \sum_{k=1}^{d}\left(\sum_{i=1}^{t}  \beta^{t-i}_{11} \left| \sum_{l=i+1}^{t}\left(\alpha_{l} \hat{v}_{l, k}^{-\frac{1}{2}}-\alpha_{l-1} \hat{v}_{l-1, k}^{-\frac{1}{2}}\right) \right| \right)^{2}\right] \\
    & \leqslant \bar{H}^{2} \mathbb{E}\left[\sum_{t=1}^{T-1} \sum_{k=1}^{d}\left(\sum_{i=1}^{t}  \beta^{t-i}_{11}
    \sum_{l=i+1}^{t} \left| \alpha_{l} \hat{v}_{l, k}^{-\frac{1}{2}}-\alpha_{l-1} \hat{v}_{l-1, k}^{-\frac{1}{2}} \right| \right)^{2}\right]. \notag
  \end{align}
  The first and the last inequality signs are both because of triangular inequality. 
  By defining $a_{l}=\left|\alpha_{l} \hat{v}_{l, k}^{-\frac{1}{2}}-\alpha_{l-1} \hat{v}_{l-1,k}^{-\frac{1}{2}}\right|$ 
  and using the lemma \ref{lem11}, we have
  \begin{equation}\label{eqA36}
      \begin{aligned}
          T_{72} 
          & \leqslant  \frac{\bar{H}^{2}}{\left(1-\beta_{11}\right)^{4}} \mathbb{E}\left[\sum_{t=2}^{T-1} \sum_{k=1}^{d}\left(\alpha_{t} \hat{v}_{t, k}^{-\frac{1}{2}}-\alpha_{t-1} \hat{v}_{t-1, k}^{-\frac{1}{2}}\right)^{2}\right] \\
%          & \leqslant \frac{\bar{H}^{2}}{\left(1-\beta_{11}\right)^{4}} \mathbb{E}\left[\sum_{t=2}^{T-1} \sum_{k=1}^{d}\left(\alpha_{t} \hat{v}_{t, k}^{-\frac{1}{2}}+\alpha_{t-1} \hat{v}_{t-1, k}^{-\frac{1}{2}}\right)\left(\alpha_{t-1} \hat{v}_{t-1, k}^{-\frac{1}{2}}-\alpha_{t} \hat{v}_{t, k}^{-\frac{1}{2}}\right)\right] \\
          & \leqslant \frac{\bar{H}^{2}}{\left(1-\beta_{11}\right)^{4}} \mathbb{E}\left[\sum_{t=2}^{T-1} \sum_{k=1}^{d}\left(\alpha_{t-1}^{2} \hat{v}_{t-1, k}^{-1}-\alpha_{t}^{2} \hat{v}_{t, k}^{-1}\right)\right] \\
          & =\frac{\bar{H}^{2}}{\left(1-\beta_{11}\right)^{4}} \mathbb{E}\left[\sum_{k=1}^{d}\left(\alpha_{1}^{2} \hat{v}_{1, k}^{-1}-\alpha_{T}^{2} \hat{v}_{T, k}^{-1}\right)\right]
          \leqslant \frac{\alpha_{1}^{2} \bar{H}^{2}}{\left(1-\beta_{11}\right)^{4}} \mathbb{E}\left[\sum_{k=1}^{d} \hat{v}_{1, k}^{-1}\right].
      \end{aligned}
  \end{equation}

  Combining Eq.\eqref{eqA30}-Eq.\eqref{eqA36}, the term $T_{22}$(see Eq.\eqref{eqA28}) can be bounded as follows:
  \begin{equation}\label{eqA37}
      \begin{aligned}
          T_{22}
          & \leqslant \frac{L^{2}}{2} T_{7} + \frac{1}{2} \mathbb{E}\left[\sum_{t=2}^{T} \left\|\frac{\alpha_{t}\left(1-\beta_{1t}\right)}{\xi_{t}} \hat{V}_{t}^{-\frac{1}{2}} d_{t} \right\|^{2}\right] \\
          & \leqslant \frac{L^{2}}{\left(1-\beta_{11}\right)^{4}}(T_{71}+T_{72})
          +\frac{1}{2} \mathbb{E}\left[\sum_{t=2}^{T} \left\|\frac{\alpha_{t}\left(1-\beta_{1t}\right)}{\xi_{t}} \hat{V}_{t}^{-\frac{1}{2}} d_{t} \right\|^{2}\right] \\
          & \leqslant \frac{1}{2} \mathbb{E}\left[\sum_{t=2}^{T} \left\|\frac{\alpha_{t}\left(1-\beta_{1t}\right)}{\xi_{t}} \hat{V}_{t}^{-\frac{1}{2}} d_{t} \right\|^{2}\right]
          +\frac{L^{2}}{\left(1-\beta_{11}\right)^{6}} \mathbb{E}\left[\sum_{t=1}^{T-1}\left\|\alpha_{t} \hat{V}_{t}^{-\frac{1}{2}} a_{t}\right\|^{2}\right] \\
          & \quad\, +\frac{\alpha_{1}^{2} L^{2} \bar{H}^{2}}{\left(1-\beta_{11}\right)^{8}} \mathbb{E}\left[\sum_{k=1}^{d} v_{1, k}^{-1}\right].
      \end{aligned}
  \end{equation}
  
  It is now bounding the term $T_{21}$(see Eq.\eqref{eqA27}). From the assumption \ref{ass33}, we can get 
  $d_{t}=g_{t}-\frac{\gamma_{t}}{t^{a}} d_{t-1}=\nabla f\left(x_{t}\right)+\zeta_{t}-\frac{\gamma_{t}}{t^{a}} d_{t-1}=\nabla f\left(x_{t}\right)-\frac{\gamma_{t}}{t^{a}} d_{t-1}+\zeta_{t}=c_{t}+\zeta_{t}$,
  in which $c_{t}=\nabla f\left(x_{t}\right)-\frac{\gamma_{t}}{t^{a}} d_{t-1}$, 
  $\mathbb{E}[\zeta_{t}]=0$. Then the following holds:
  \begin{equation}\label{eqA38}
      \begin{aligned}
          T_{21} 
%         & =-\mathbb{E}\left[\sum_{t=1}^{T} \frac{\alpha_{t}\left(1-\beta_{1t}\right)}{\xi_{t}}\left\langle\nabla f\left(x_{t}\right), \hat{V}_{t}^{-\frac{1}{2}} d_{t}\right\rangle\right] \\
          & = -\mathbb{E}\left[\sum_{t=1}^{T} \frac{\alpha_{t}\left(1-\beta_{1t}\right)}{\xi_{t}}\left\langle\nabla f\left(x_{t}\right), \hat{V}_{t}^{-\frac{1}{2}} \left( c_{t} + \zeta_{t} \right)\right\rangle\right] \\
          & = -\mathbb{E}\left[\sum_{t=1}^{T} \frac{\alpha_{t}\left(1-\beta_{1t}\right)}{\xi_{t}}\left\langle\nabla f\left(x_{t}\right), \hat{V}_{t}^{-\frac{1}{2}} c_{t}\right\rangle\right] \\
          & \quad\, -\mathbb{E}\left[\sum_{t=1}^{T} \frac{\alpha_{t}\left(1-\beta_{1t}\right)}{\xi_{t}}\left\langle\nabla f\left(x_{t}\right), \hat{V}_{t}^{-\frac{1}{2}} \zeta_{t}\right\rangle\right].
      \end{aligned}
  \end{equation}
  Let $\mu_{t}=\frac{\alpha_{t}\left(1-\beta_{1t}\right)}{\xi_{t}}$, then
  \begin{equation}\label{eqA39}
      \begin{aligned}
          & -\mathbb{E}\left[\sum_{t=1}^{T} \frac{\alpha_{t}\left(1-\beta_{1t}\right)}{\xi_{t}}\left\langle\nabla f\left(x_{t}\right), \hat{V}_{t}^{-\frac{1}{2}} \zeta_{t}\right\rangle\right] \\
%         & = -\mathbb{E}\left[\sum_{t=1}^{T} \mu_{t} \left\langle\nabla f\left(x_{t}\right), \hat{V}_{t}^{-\frac{1}{2}} \zeta_{t}\right\rangle\right] \\
          & = -\mathbb{E}\left[\sum_{t=2}^{T}\left\langle\nabla f\left(x_{t}\right),\left(\mu_{t} \hat{V}_{t}^{-\frac{1}{2}}-\mu_{t-1} \hat{V}_{t-1}^{-\frac{1}{2}}\right) \zeta_{t}\right\rangle\right] \\
          & \quad\, -\mathbb{E}\left[\sum_{t=2}^{T} \mu_{t-1}\left\langle\nabla f\left(x_{t}\right), \hat{V}_{t-1}^{-\frac{1}{2}} \zeta_{t}\right\rangle\right]
          -\mathbb{E} \left[ \mu_{1} \left \langle \nabla f(x_{t}), \hat{V}^{-\frac{1}{2}}_{1} \zeta_{1}\right \rangle  \right] \\
          & \leqslant-\mathbb{E}\left[\sum_{t=2}^{T}\left\langle\nabla f\left(x_{t}\right),\left(\mu_{t} \hat{V}_{t}^{-\frac{1}{2}}-\mu_{t-1} \hat{V}_{t-1}^{-\frac{1}{2}}\right) \zeta_{t}\right\rangle\right]
          +2 \mu_{1} H^{2} \mathbb{E}\left[\sum_{k=1}^{d} \hat{v}^{-\frac{1}{2}}_{1,k}\right].
      \end{aligned}
  \end{equation}
  The last inequality sign is due to $\forall t \in \mathcal{T} \backslash \{1\}$, 
  $\mathbb{E}\left[\hat{V}_{t-1}^{-\frac{1}{2}} \zeta_{t} \mid x_{t}, \hat{V}_{t-1}\right]=0$,
  and the fact that 
  $\left\|\zeta_{t}\right\|-\left\|\nabla f\left(x_{t}\right)\right\| \leqslant \left\|\nabla f\left(x_{t}\right)+\zeta_{t}\right\|=\left\|g_{t}\right\| \leqslant H$,
  $\left\|\zeta_{t}\right\| \leqslant H+\left\|\nabla f\left(x_{t}\right)\right\| \leqslant 2 H$.
  Further more,
  \begin{equation}\label{eqA40}
      \begin{aligned}
          & -\mathbb{E}\left[\sum_{t=2}^{T}\left\langle\nabla f\left(x_{t}\right),\left(\mu_{t} \hat{V}_{t}^{-\frac{1}{2}}-\mu_{t-1} \hat{V}_{t-1}^{-\frac{1}{2}}\right) \zeta_{t}\right\rangle\right] \\
%         & =-\mathbb{E}\left[\sum_{t=2}^{T} \sum_{k=1}^{d} \left(\nabla f\left(x_{t}\right)\right)_{k}\left(\mu_{t} \hat{v}_{t, k}^{-\frac{1}{2}}-\mu_{t-1} \hat{v}_{t-1, k}^{-\frac{1}{2}}\right) \zeta_{t, k}\right] \\
          & \leqslant \mathbb{E}\left[\sum_{t=2}^{T} \sum_{k=1}^{d}\left|\left(\nabla f\left(x_{t}\right)\right)_{k}\right| \cdot\left|\mu_{t} \hat{v}_{t, k}^{-\frac{1}{2}}-\mu_{t-1} \hat{v}_{t-1, k}^{-\frac{1}{2}}\right| \cdot\left|\zeta_{t, k}\right|\right] \\
          & \leqslant 2 H^{2} \mathbb{E}\left[\sum_{t=2}^{T} \sum_{k=1}^{d}\left|\mu_{t} \hat{v}_{t,k}^{-\frac{1}{2}}-\mu_{t-1} \hat{v}_{t-1, k}^{-\frac{1}{2}}\right|\right].
      \end{aligned}
  \end{equation}

  Therefore, combining Eq.\eqref{eqA39} and Eq.\eqref{eqA40}, term $T_{21}$(see Eq.\eqref{eqA38}) can be bounded as 
  follows:
  \begin{equation}\label{eqA41}
      \begin{aligned}
          T_{21}
          & \leqslant -\mathbb{E}\left[\sum_{t=1}^{T} \frac{\alpha_{t}\left(1-\beta_{1t}\right)}{\xi_{t}}\left\langle\nabla f\left(x_{t}\right), \hat{V}_{t}^{-\frac{1}{2}} c_{t}\right\rangle\right] \\
          & \quad\, +2 H^{2} \mathbb{E}\left[\sum_{t=2}^{T} \sum_{k=1}^{d}\left|\mu_{t} \hat{v}_{t,k}^{-\frac{1}{2}}-\mu_{t-1} \hat{v}_{t-1, k}^{-\frac{1}{2}}\right|\right]
          +2 \mu_{1} H^{2} \mathbb{E}\left[\sum_{k=1}^{d} \hat{v}^{-\frac{1}{2}}_{1,k}\right].
      \end{aligned}
  \end{equation}

  Finally, the results of Eq.\eqref{eqA29}, Eq.\eqref{eqA37} and Eq.\eqref{eqA41} ensure Eq.\eqref{eqA26}.
  
  The proof is over.
\end{proof}
\newpage
\section{Proof of Theorem \ref{th3.1}}
\label{prth31} %Appendix B: 
\begin{proof}
    By the lemma \ref{lem2}, it is obvious that
    \begin{equation}\label{eqB1}
        \begin{aligned}
            \mathbb{E}\left[f\left(z_{t+1}\right)-f\left(z_{1}\right)\right] \leqslant \sum_{i=1}^{6} T_{i}.
%            & =-\mathbb{E}\left[\sum_{t=1}^{T}\left\langle\nabla f\left(z_{t}\right), \eta_{t}\left(\alpha_{t} \hat{V}_{t}^{-\frac{1}{2}}-\alpha_{t-1} \hat{V}_{t-1}^{-\frac{1}{2}}\right) \hat{m}_{t-1}\right\rangle\right] \\
%            & \quad\, -\mathbb{E}\left[\sum_{t=1}^{T}\left\langle\nabla f\left(z_{t}\right), \frac{\alpha_{t}\left(1-\beta_{1t}\right)}{\xi_{t}} \hat{V}_{t}^{-\frac{1}{2}} d_{t}\right\rangle\right] \\
%            & \quad\, -\mathbb{E}\left[\sum_{t=1}^{T}\left\langle\nabla f\left(z_{t}\right),\left(\eta_{t+1}-\eta_{t}\right) \alpha_{t} \hat{V}_{t}^{-\frac{1}{2}} \hat{m}_{t}\right\rangle\right] \\  
%            & \quad\,  +\frac{3 L}{2} \mathbb{E}\left[\sum_{t=1}^{T}\left\|\left(\eta_{t+1}-\eta_{t}\right) \alpha_{t} \hat{V}_{t}^{-\frac{1}{2}} \hat{m}_{t}\right\|^{2}\right] \\
%            & \quad\, +\frac{3 L}{2} \mathbb{E}\left[\sum_{t=1}^{T}\left\|\eta_{t}\left(\alpha_{t} \hat{V}_{t}^{-\frac{1}{2}}-\alpha_{t-1} \hat{V}_{t-1}^{-\frac{1}{2}}\right) \hat{m}_{t}\right\|^{2}\right] \\
%            & \quad\, +\frac{3 L}{2} \mathbb{E}\left[\sum_{t=1}^{T}\left\|\frac{\alpha_{t}\left(1-\beta_{1 t}\right)}{\xi_{t}} \hat{V}_{t}^{-\frac{1}{2}} d_{t}\right\|^{2}\right]. 
        \end{aligned}
    \end{equation}

    From the lemma \ref{lem3}-\ref{lem7}, Eq.\eqref{eqB1} can be further bounded as follows:
    \begin{equation}\label{eqB2}
        \begin{aligned}
            & \mathbb{E}\left[f\left(z_{t+1}\right)-f\left(z_{1}\right)\right] \\
            & \leqslant \left[\frac{\alpha_{1} H \bar{M}}{1-\beta_{11}}+2 \mu_{1} H^{2}\right] E\left[\sum_{i=1}^{d} \hat{v}_{1, i}^{-\frac{1}{2}}\right] \\
            & \quad\,  +\left[\frac{H^{2}+G^{2}}{2}+\frac{3 L G^{2}}{1-\beta_{11}}\right]\left|\eta_{T}-\eta_{1}\right|\\
            & \quad\, + \left[\frac{3 \alpha_{1}^{2} L \bar{M}^{2}}{2\left(1-\beta_{11}\right)^{2}}+\frac{\alpha_{1}^{2} L^{2} \bar{H}^{2}}{\left(1-\beta_{11}\right)^{8}}\right] \mathbb{E}\left[\sum_{i=1}^{d} \hat{v}_{1, i}^{-1}\right]\\
            & \quad\, +\frac{L^{2}}{\left(1-\beta_{11}\right)^{6}} \mathbb{E}\left[\sum_{t=1}^{T-1}\left\|\alpha_{t} \hat{V}_{t}^{-\frac{1}{2}} d_{t}\right\|^{2}\right] \\
            & \quad\, 
             + \frac{1}{2} \mathbb{E}\left[\sum_{t=2}^{T}\left\|\frac{\alpha_{t}\left(1-\beta_{1 t}\right)}{\xi_{t}} \hat{V}_{t}^{-\frac{1}{2}} d_{t}\right\|^{2}\right] \\
            & \quad\, +2 H^{2} \mathbb{E}\left[\sum_{t=2}^{T} \sum_{i=1}^{d}\left|\mu_{t} \hat{v}_{t, i}^{-\frac{1}{2}}-\mu_{t-1} \hat{v}_{t-1, i}^{-\frac{1}{2}}\right|\right]  \\
            & \quad\,-E\left[\sum_{t=1}^{T} \frac{\alpha_{t}\left(1-\beta_{1t}\right)}{\xi_{t}}\left\langle\nabla f\left(x_{t}\right), \hat{V}_{t}^{-\frac{1}{2}} c_{t}\right\rangle\right].
        \end{aligned}
    \end{equation}
    
    Rearranging Eq.\eqref{eqB2} and uniting like terms, we can get 
    \begin{align}
        & \mathbb{E}\left[\sum_{t=1}^{T} \frac{\alpha_{t}\left(1-\beta_{1t}\right)}{\xi_{t}}\left\langle\nabla f\left(x_{t}\right), \hat{V}_{t}^{-\frac{1}{2}} c_{t}\right\rangle\right] \notag \\
    %    & \leqslant \mathbb{E}\left[f\left(z_{1}\right)-f\left(z_{t+1}\right)\right] 
    %    +\frac{\alpha_{1} H \bar{M}}{1-\beta_{11}} \mathbb{E}\left[\sum_{i=1}^{d} \hat{v}_{1, i}^{-\frac{1}{2}}\right] \\
    %    & \quad\, +\frac{1}{2}\left(H^{2}+G^{2}\right)\left|\eta_{T}-\eta_{1}\right|+ \frac{3 L G^{2}}{1-\beta_{11}}\left|\eta_{T}-\eta_{1}\right| \\
    %    & \quad\,  +\frac{3 \alpha_{1}^{2} L \bar{M}^{2}}{2\left(1-\beta_{11}\right)^{2}} \mathbb{E}\left[\sum_{i=1}^{d} \hat{v}_{1, i}^{-1}\right]+ \frac{\alpha_{1}^{2} L^{2} \bar{H}^{2}}{\left(1-\beta_{11}\right)^{8}} \mathbb{E}\left[\sum_{i=1}^{d} \hat{v}_{1, i}^{-1}\right]\\
    %    & \quad\, +\frac{L^{2}}{\left(1-\beta_{11}\right)^{6}} \mathbb{E}\left[\sum_{t=1}^{T-1}\left\|\alpha_{t} \hat{V}_{t}^{-\frac{1}{2}} d_{t}\right\|^{2}\right]+2 \mu_{1} H^{2} \mathbb{E}\left[\sum_{i=1}^{d} \hat{v}^{-\frac{1}{2}}_{1,i}\right]\\
    %    & \quad\, +2 H^{2} \mathbb{E}\left[\sum_{t=2}^{T} \sum_{i=1}^{d}\left|\mu_{t} \hat{v}_{t, i}^{-\frac{1}{2}}-\mu_{t-1} \hat{v}_{t-1, i}^{-\frac{1}{2}}\right|\right] \\
    %    & \quad\, + \frac{1}{2} \mathbb{E}\left[\sum_{t=2}^{T}\left\|\frac{\alpha_{t}\left(1-\beta_{1 t}\right)}{\xi_{t}} \hat{V}_{t}^{-\frac{1}{2}} d_{t}\right\|^{2}\right]\\
        & \leqslant \left(\frac{\alpha_{1} H \bar{M}}{1-\beta_{11}} +2 \mu_{1} H^{2} \right) \mathbb{E}\left[\sum_{i=1}^{d} \hat{v}_{1, i}^{-\frac{1}{2}}\right] \notag \\
        & \quad\, +\left( \frac{H^{2}+G^{2}}{2} + \frac{3 L G^{2}}{1-\beta_{11}} \right) \frac{2}{1-\beta_{11}} \notag \\
        & \quad\,  + \left[\frac{3 \alpha_{1}^{2} L \bar{M}^{2}}{2\left(1-\beta_{11}\right)^{2}} + \frac{\alpha_{1}^{2} L^{2} \bar{H}^{2}}{\left(1-\beta_{11}\right)^{8}} \right] \mathbb{E}\left[\sum_{i=1}^{d} \hat{v}_{1, i}^{-1}\right] \label{eqB3} \\
        & \quad\, + \left[\frac{L^{2}}{\left(1-\beta_{11}\right)^{6}} + \frac{1}{2(1-\beta_{11})^{4}} \right] \mathbb{E}\left[\sum_{t=1}^{T-1}\left\|\alpha_{t} \hat{V}_{t}^{-\frac{1}{2}} d_{t}\right\|^{2}\right] \notag \\
        & \quad\, + 2 H^{2} \mathbb{E}\left[\sum_{t=2}^{T} \sum_{i=1}^{d}\left|\mu_{t} \hat{v}_{t, i}^{-\frac{1}{2}}-\mu_{t-1} \hat{v}_{t-1, i}^{-\frac{1}{2}}\right|\right] 
        + \mathbb{E}\left[f\left(z_{1}\right)-f\left(z^{*}\right)\right] , \notag 
    \end{align}
    where $z^{*}=\underset{z \in \mathcal{X}}{\arg \min } f(z)$.

    Let 
    \begin{equation*}
        \begin{aligned}
            C^{'}_{1}
            & =2 H^{2}, \ 
            C^{'}_{2}=\frac{L^{2}}{\left(1-\beta_{11}\right)^{6}}+\frac{1}{2\left(1-\beta_{11}\right)^{4}}, \\
            C^{'}_{3}
            & =\mathbb{E}\left[f\left(z_{1}\right)-f\left(z^{*}\right)\right]\\
            & \quad\, +\left(\frac{\alpha_{1} H \bar{M}}{1-\beta_{11}} + 2 \mu_{1} H^{2}\right) \mathbb{E}\left[\sum_{i=1}^{d} \hat{v}_{1,i}^{-\frac{1}{2}}\right] \\
            & \quad\, +\left(\frac{H^{2}+G^{2}}{2}+\frac{3 L G}{1-\beta_{11}}\right) \frac{2}{1-\beta_{11}} \\
            & \quad\, +\left[\frac{3 \alpha^{2}_{1} L \bar{M}}{2\left(1-\beta_{11}\right)^{2}}+\frac{\alpha_{1}^{2} L^{2} \bar{H}^{2}}{\left(1-\beta_{11}\right)^{8}}\right] \mathbb{E}\left[\sum_{i=1}^{d} \hat{v}_{1, i}^{-1}\right].\\
        \end{aligned}
    \end{equation*}
    Hence the following holds:
    \begin{equation*}
        \begin{aligned}
            & \mathbb{E}\left[\sum_{t=1}^{T} \frac{\alpha_{t}\left(1-\beta_{1t}\right)}{\xi_{t}}\left\langle\nabla f\left(x_{t}\right), \hat{V}_{t}^{-\frac{1}{2}} c_{t}\right\rangle\right] \\
            & \leqslant C^{'}_{1} \mathbb{E}\left[\sum_{t=2}^{T} \sum_{k=1}^{d}\left|\mu_{t} \hat{v}_{t, k}^{-\frac{1}{2}}-\mu_{t-1} \hat{v}_{t-1,k}^{-\frac{1}{2}}\right|\right]
            + C^{'}_{2} \mathbb{E}\left[\sum_{t=1}^{T-1}\left\|\alpha_{t} \hat{V}_{t}^{-\frac{1}{2}} d_{t}\right\|^{2}\right]+C^{'}_{3}
        \end{aligned}
    \end{equation*}
    and 
    \begin{equation}\label{eqB4}
        \begin{aligned}
            & \mathbb{E}\left[\sum_{t=1}^{T} \alpha_{t} \left\langle\nabla f\left(x_{t}\right), \hat{V}_{t}^{-\frac{1}{2}} c_{t}\right\rangle\right] \\
            & \leqslant \frac{\xi_{t} C^{'}_{1}}{1-\beta_{1t}} \mathbb{E}\left[\sum_{t=2}^{T} \sum_{k=1}^{d}\left|\mu_{t} \hat{v}_{t, k}^{-\frac{1}{2}}-\mu_{t-1} \hat{v}_{t-1,k}^{-\frac{1}{2}}\right|\right]
            + \frac{\xi_{t} C^{'}_{2}}{1-\beta_{1t}} \mathbb{E}\left[\sum_{t=1}^{T-1}\left\|\alpha_{t} \hat{V}_{t}^{-\frac{1}{2}} d_{t}\right\|^{2}\right]+ \frac{\xi_{t} C^{'}_{3}}{1-\beta_{1t}} \\
            & \leqslant \frac{C^{'}_{1}}{1-\beta_{11}} \mathbb{E}\left[\sum_{t=2}^{T} \sum_{k=1}^{d}\left|\mu_{t} \hat{v}_{t, k}^{-\frac{1}{2}}-\mu_{t-1} \hat{v}_{t-1,k}^{-\frac{1}{2}}\right|\right]
            + \frac{C^{'}_{2}}{1-\beta_{11}} \mathbb{E}\left[\sum_{t=1}^{T-1}\left\|\alpha_{t} \hat{V}_{t}^{-\frac{1}{2}} d_{t}\right\|^{2}\right] + \frac{C^{'}_{3}}{1-\beta_{11}}.
        \end{aligned}
    \end{equation}
    The last inequality sign is due to the lemma \ref{lem8} and $1-\beta_{1t} \geqslant 1-\beta_{11}$.

    Since $c_{t}=\nabla f\left(x_{t}\right)-\frac{\gamma_{t}}{t^{a}} d_{t-1}$,
    then there is
    \begin{equation}\label{eqB5}
        \begin{aligned}
            \mathbb{E}\left[\sum_{t=1}^{T} \alpha_{t} \left\langle\nabla f\left(x_{t}\right), \hat{V}_{t}^{-\frac{1}{2}} c_{t}\right\rangle\right]
            & = \mathbb{E}\left[\sum_{t=1}^{T}\left\langle\nabla f\left(x_{t}\right), \alpha_{t} \hat{V}_{t}^{-\frac{1}{2}} \nabla f\left(x_{t}\right)\right\rangle\right] \\
            & \quad\, -\mathbb{E}\left[\sum_{t=1}^{T}\left\langle\nabla f\left(x_{t}\right), \frac{\alpha_{t} \gamma_{t}}{t^{a}} \hat{V}_{t}^{-\frac{1}{2}} d_{t-1}\right\rangle\right].
        \end{aligned}
    \end{equation}

    Let 
    \begin{equation}\label{eqB6}
        \begin{aligned}
            R_{1} 
            & =\frac{C^{'}_{1}}{1-\beta_{11}} \mathbb{E}\left[\sum_{t=2}^{T} \sum_{k=1}^{d}\left|\mu_{t} \hat{v}_{t, k}^{-\frac{1}{2}}-\mu_{t-1} \hat{v}_{t-1,k}^{-\frac{1}{2}}\right|\right] \\
            & \quad\, + \frac{C^{'}_{2}}{1-\beta_{11}} \mathbb{E}\left[\sum_{t=1}^{T-1}\left\|\alpha_{t} \hat{V}_{t}^{-\frac{1}{2}} d_{t}\right\|^{2}\right]+ \frac{C^{'}_{3}}{1-\beta_{11}}. 
        \end{aligned}
    \end{equation}
    
    Combining Eq.\eqref{eqB4}, Eq.\eqref{eqB5} and Eq.\eqref{eqB6} comes to
    \begin{equation}\label{eqB7}
        \begin{aligned}
            & \mathbb{E}\left[\sum_{t=1}^{T} \left\langle\nabla f\left(x_{t}\right), \alpha_{t} \hat{V}_{t}^{-\frac{1}{2}} \nabla f\left(x_{t}\right)\right\rangle\right] \\ 
            & \leqslant R_{1} + \mathbb{E}\left[\sum_{t=1}^{T}\left\langle\nabla f\left(x_{t}\right), \frac{\alpha_{t} \gamma_{t}}{t^{a}} \hat{V}_{t}^{-\frac{1}{2}} d_{t-1}\right\rangle\right] \\
            & \leqslant R_{1} + \mathbb{E}\left[\sum_{t=1}^{T} \frac{\alpha_{t}\left|\gamma_{t}\right|}{t^{a}}\left\|\nabla f\left(x_{t}\right)\right\| \cdot\left\|\hat{V}_{t}^{-\frac{1}{2}} d_{t-1} \right\|\right]\\
            & \leqslant R_{1} + H \mathbb{E}\left[\sum_{t=1}^{T} \frac{\alpha_{t}\left|\gamma_{t}\right|}{t^{a}}\sqrt{\sum_{i=1}^{d} \hat{v}_{t, i}^{-1} d_{t-1, i}^{2}}\right] \\
            & \leqslant R_{1} + H \bar{H} \sqrt{\sum_{i=1}^{d} \hat{v}_{1, i}^{-1}} \mathbb{E}\left[\sum_{t=1}^{T} \frac{\alpha_{t}\left|\gamma_{t}\right|}{t^{a}}\right]. \\
        \end{aligned}
    \end{equation}
    The last inequality sgin is due to the lemma \ref{lem10} and $\forall t \in \mathcal{T}$,
    $\hat{v}_{t-1, i} \leqslant \hat{v}_{t, i}$.

    Let 
    \begin{equation}\label{eqB8}
        \begin{aligned}
            C_{1}
            & =\frac{C^{'}_{1}}{1-\beta_{11}}, \ 
            C_{2}=\frac{C^{'}_{2}}{1-\beta_{11}}, \  \\
            C_{3}
            & =\frac{C^{'}_{3}}{1-\beta_{11}}, \ 
            C_{4}=H \bar{H} \sqrt{\sum_{i=1}^{d} \hat{v}_{1, i}^{-1}}.
        \end{aligned}
    \end{equation}
    
    Combining Eq.\eqref{eqB6}, Eq.\eqref{eqB7} and Eq.\eqref{eqB8} completes the proof.
\end{proof}
\newpage
\section{Proof of Theorem \ref{th3.2}} %Appendix C: 
\label{prth32}
\begin{proof}
    Let 
    \begin{equation}\label{eqC1}
        \begin{aligned}
            R_{2}
            & =C_{1} \mathbb{E}\left[\sum_{t=2}^{T} \sum_{k=1}^{d}\left|\mu_{t} \hat{v}_{t, k}^{-\frac{1}{2}} - \mu_{t-1} \hat{v}_{t-1, k}^{-\frac{1}{2}}\right|\right]\\
            & \quad\,+C_{2} \mathbb{E}\left[\sum_{t=1}^{T-1}\left\|\alpha_{t} \hat{V}_{t}^{-\frac{1}{2}} d_{t}\right\|^{2}\right]+ C_{3}+{C_{4} \sum_{t=1}^{T} \frac{\alpha_{t}\left|\gamma_{t}\right|}{t^{a}}}.
        \end{aligned}
    \end{equation}

    On one hand, from the definition of $O(\cdot)$, $\Omega(\cdot)$, obviously 
    $\exists K_{1}, K_{2} \in \mathcal{R}^{+}$, $\exists T_{0} \in \mathcal{T}$,
    $\forall T \geqslant T_{0}$,
    \begin{equation}\label{eqC2}
        \begin{aligned}
            R_{2} \leqslant K_{1} S_{1}(T), \quad
            \sum_{t=1}^{T} \tau_{t} \geqslant K_{2} S_{2}(T)>0.
        \end{aligned}
    \end{equation}

    On the other hand, if $T \geqslant T_{0}$, then
    \begin{equation}\label{eqC3}
        \begin{aligned}
            & \mathbb{E}\left[\sum_{t=1}^{T} \left\langle\nabla f\left(x_{t}\right), \alpha_{t} \hat{V}_{t}^{-\frac{1}{2}} \nabla f\left(x_{t}\right)\right\rangle\right] \\
            & \geqslant \mathbb{E}\left[\sum_{t=1}^{T} \tau_{t}\left\|\nabla f\left(x_{t}\right)\right\|^{2}\right] \\
            & = \sum_{t=1}^{T} \tau_{t} \cdot \mathbb{E}\left[\left\|\nabla f\left(x_{t}\right)\right\|^{2}\right]\\
            & \geqslant \sum_{t=1}^{T} \tau_{t} \cdot \underset{t \in \mathcal{T}}{\min} \ \mathbb{E}\left[\left\|\nabla f\left(x_{t}\right)\right\|^{2}\right] \\
            & \geqslant K_{2} S_{2}(T) \cdot \underset{t \in \mathcal{T}}{\min} \ \mathbb{E}\left[\left\|\nabla f\left(x_{t}\right)\right\|^{2}\right].
        \end{aligned}
    \end{equation}

    Combining the theorem \ref{th3.1}, Eq.\eqref{eqC1}, Eq.\eqref{eqC2} and Eq.\eqref{eqC3}, 
    when $T \geqslant T_{0}$, we have
    \begin{equation}\label{eqC4}
        \begin{aligned}
            K_{2} S_{2}(T) \cdot \underset{t \in \mathcal{T}}{\min} \ \mathbb{E}\left[\left\|\nabla f\left(x_{t}\right)\right\|^{2}\right] 
            \leqslant R_{2} \leqslant K_{1} S_{1}(T).
        \end{aligned}
    \end{equation}

    It is equivalent to when $T \rightarrow +\infty$,
    \begin{equation}\label{eqC5}
        \begin{aligned}
            \underset{t \in \mathcal{T}}{\min} \ \mathbb{E}\left[\left\|\nabla f\left(x_{t}\right)\right\|^{2}\right] 
            \leqslant \frac{K_{1}}{K_{2}} \frac{S_{1}(T)}{S_{2}(T)}, \\
            \underset{t \in \mathcal{T}}{\min} \ \mathbb{E}\left[\left\|\nabla f\left(x_{t}\right)\right\|^{2}\right] 
            = O\left(\frac{S_{1}(T)}{S_{2}(T)}\right).
        \end{aligned}
    \end{equation}

    The proof is over.
\end{proof}
\newpage
\section{Proof of Corollary \ref{cor31}}
\label{prcor31} %Appendix D: 
\begin{proof}
    Firstly proof the following:
    \begin{equation}\label{eqD1}
        \begin{aligned}
            & \mathbb{E}\left[\sum_{t=1}^{T-1}\left\|\alpha_{t} \hat{V}_{t}^{-\frac{1}{2}} d_{t}\right\|^{2}\right] \leqslant \mathbb{E}\left[\sum_{t=1}^{T}\left\|\frac{\alpha_{t}} {c}  d_{t} \right\|^{2}\right] \\
%           & =\mathbb{E}\left[\sum_{t=1}^{T} \frac{\alpha^{2}}{c^{2} t^{2 b}}\left\|d_{t}\right\|^{2}\right]\\
            & \leqslant \frac{\alpha^{2} \bar{H}^{2}}{c^{2}} \sum_{t=1}^{T} \frac{1}{t^{2 b}} 
             \leqslant \frac{\alpha^{2} \bar{H}^{2}}{c^{2}} \sum_{t=1}^{T} \frac{1}{t} 
             \leqslant \frac{\alpha^{2} \bar{H}^{2}}{c^{2}}(1+\ln T)
        \end{aligned}
    \end{equation}
    and
    \begin{equation}\label{eqD2}
        \begin{aligned}
            \sum_{t=1}^{T} \frac{\alpha_{t}\left|\gamma_{t}\right|}{t^{b}} 
            \leqslant \alpha \bar{\gamma} \sum_{t=1}^{T} \frac{1}{t^{a+b}} 
            \leqslant \alpha \bar{\gamma} \sum_{t=1}^{T} \frac{1}{t} 
            \leqslant \alpha \bar{\gamma}(1+\ln T).
        \end{aligned}
    \end{equation}

    Since $ \beta_{1 t} $ is a constant,namely $ \beta_{1 t} = \beta_{1 1} $,
    so $\mu_{t}=\frac{\alpha_{t} (1-\beta_{1t})}{\xi_{t}}=\frac{\alpha_{t} (1-\beta_{11})}{1-\beta_{11}^{t}-\beta_{11}(1-\beta_{11}^{t-1})}=\alpha_{t}$, 
    we have
    \begin{equation}\label{eqD3}
        \begin{aligned}
            & \mathbb{E}\left[\sum_{t=2}^{T} \sum_{k=1}^{d}\left|\mu_{t} \hat{v}_{t, k}^{-\frac{1}{2}}-\mu_{t-1} \hat{v}_{t-1,k}^{-\frac{1}{2}}\right|\right]
            =\mathbb{E}\left[\sum_{t=2}^{T} \sum_{k=1}^{d}\left(\alpha_{t-1} \hat{v}_{t-1, k}^{-\frac{1}{2}}-\alpha_{t} \hat{v}_{t, k}^{-\frac{1}{2}}\right)\right] \\
            & =\mathbb{E}\left[\sum_{k=1}^{d}\left(\alpha_{1} \hat{v}_{1, k}^{-\frac{1}{2}}-\alpha_{T} \hat{v}_{T, k}^{-\frac{1}{2}}\right)\right]
            \leqslant \mathbb{E}\left[\sum_{k=1}^{d} \alpha_{1} \hat{v}_{1, k}^{-\frac{1}{2}}\right] 
            \leqslant \frac{\alpha d}{c}.
        \end{aligned}
    \end{equation}

    Therefore the term $R_{2}$ (Eq.\eqref{eqC1}) can be bounded as follows:
    \begin{equation}\label{eqD4}
        \begin{aligned}
            R_{2} 
             \leqslant \left(\frac{\alpha^{2} \bar{H}^{2} C_{2}}{c^{2}} + C_{4} \alpha \bar{\gamma} \right) (1+\ln T)
             + \frac{\alpha d C_{1}}{c} + C_{3}.
        \end{aligned} 
    \end{equation}

    Besides, because of 
    $\hat{v}_{t}=\frac{1}{1-\beta^{t}_{2}} \left[\beta_{2} v_{t-1}+\left(1-\beta_{2}\right) d^{2}_{t}\right]$
    and the lemma \ref{lem10}, then
    \begin{equation}\label{eqD5}
        \begin{aligned}
            \frac{\alpha_{t}}{\sqrt{\hat{v}_{t, k}}} 
            \geqslant \frac{\left(1-\beta_{2}\right) \alpha_{t}}{\bar{H}}
            = \frac{1-\beta_{2}}{\bar{H}} \frac{\alpha}{t^{b}} 
            \geqslant \frac{\alpha\left(1-\beta_{2}\right)}{\bar{H}} \frac{1}{T^{b}}.
        \end{aligned} 
    \end{equation}

    From Eq.\eqref{eqD5}, the following holds apparently:
    \begin{equation}\label{eqD6}
        \begin{aligned}
            & \mathbb{E}\left[\sum_{t=1}^{T} \alpha_{t}\left\langle\nabla f\left(x_{t}\right), \hat{V}_{t}^{-\frac{1}{2}} \nabla f\left(x_{t}\right)\right\rangle\right] \\
            & \geqslant \frac{\alpha\left(1-\beta_{2}\right)}{\bar{H}} \frac{1}{T^{b}} \mathbb{E}\left[\sum_{t=1}^{T}\left\|\nabla f\left(x_{t}\right)\right\|^{2}\right] \\
            & \geqslant \frac{\alpha \left(1-\beta_{2}\right)}{\bar{H}} T^{1-b} \cdot \min _{t \in \mathcal{T}} \mathbb{E}\left[\left\|\nabla f\left(x_{t}\right)\right\|^{2}\right].
        \end{aligned} 
    \end{equation}

    The theorem \ref{th3.1}, the Eq.\eqref{eqD4} and Eq.\eqref{eqD6} lead to
    \begin{equation}\label{eqD7}
        \begin{aligned}
            \frac{\alpha\left(1-\beta_{2}\right)}{\bar{H}} T^{1-b} \cdot \min _{t \in \mathcal{T}} \mathbb{E}\left[\left\|\nabla f\left(x_{t}\right)\right\|^{2}\right]
            \leqslant \left(\frac{\alpha^{2} \bar{H}^{2} C_{2}}{c^{2}} + C_{4} \alpha \bar{\gamma} \right) (1+\ln T)
            + \frac{\alpha d C_{1}}{c} + C_{3}.
        \end{aligned} 
    \end{equation}
    
    Let 
    \begin{equation}\label{eqD8}
        \begin{aligned}
            Q_{1} = \frac{\bar{H}}{\alpha\left(1-\beta_{2}\right)}, \quad
            Q_{2} = \frac{\alpha^{2} \bar{H}^{2} C_{2}}{c^{2}} + C_{4} \alpha \bar{\gamma}, \quad
            Q_{3} = Q_{2} + \frac{\alpha d C_{1}}{c} + C_{3}. 
        \end{aligned} 
    \end{equation}
    
    Combining Eq.\eqref{eqD7} and Eq.\eqref{eqD8} obtains 
    \begin{equation}\label{eqD9}
        \begin{aligned}
            \min _{t \in \mathcal{T}} \mathbb{E}\left[\left\|\nabla f\left(x_{t}\right)\right\|^{2}\right]
            \leqslant \frac{Q_{1}}{T^{1-b}} \left(Q_{2} \ln T + Q_{3} \right).
        \end{aligned} 
    \end{equation}

    The proof is over.
\end{proof}

\end{sloppypar}
\end{document}